\documentclass{article}

\usepackage{arxiv}

\usepackage{amsmath, amsthm, amssymb, graphicx, hyperref, etoolbox, longtable, colonequals}

\makeatletter
\let\@nodottedtocline\@dottedtocline
\patchcmd{\@nodottedtocline}{\hbox{.}}{\hbox{}}{}{}
\patchcmd{\@nodottedtocline}{\normalcolor #5}{\normalcolor}{}{}
\newcommand*\l@sectionsubtitle{\@nodottedtocline{1}{0em}{1.5em}}
\makeatother

\usepackage{imakeidx}
\usepackage[all]{xy}

\makeindex[title=Index,columns=3]
\indexsetup{noclearpage}

\usepackage{amsmath}

\DeclareFontFamily{U}{dmjhira}{}
\DeclareFontShape{U}{dmjhira}{m}{n}{ <-> dmjhira }{}

\DeclareRobustCommand{\yo}{\text{\usefont{U}{dmjhira}{m}{n}\symbol{"48}}}

\usepackage{tikz}
\usetikzlibrary{calc,arrows.meta}

\theoremstyle{definition}

\newcommand\cofib\rightarrowtail

\newcommand\mdel[1]{}

\renewcommand\geq\geqslant
\renewcommand\leq\leqslant

\usepackage[utf8]{inputenc} 
\usepackage[T1]{fontenc}    
\usepackage{hyperref}       
\usepackage{url}            
\usepackage{booktabs}       
\usepackage{amsfonts}       
\usepackage{nicefrac}       
\usepackage{microtype}      
\usepackage{lipsum}		
\usepackage{graphicx}
\usepackage{natbib}
\usepackage{doi}
\usepackage{tikz-cd}
\usepackage{mathtools}

\usepackage{amsmath}
\usepackage{amssymb}
\usepackage{amsthm}

\newtheorem{theorem}{Theorem}
\newtheorem{definition}{Definition}
\newtheorem{lemma}{Lemma}
\newtheorem{example}{Example}

\usepackage[boxruled]{algorithm2e}
\usepackage{algorithmic}

\usepackage{multirow}
\usepackage{booktabs}
\usepackage{balance}
\usepackage{subfig}

\newcommand{\CI}{\mathrel{\perp\mspace{-10mu}\perp}}

\title{GAIA: Categorical Foundations of Generative AI\thanks{This is a preliminary draft of a forthcoming book. This draft may contain errors or omissions, and will be periodically updated. } }

\author{ Sridhar Mahadevan \\
	Adobe Research and University of Massachusetts, Amherst\\
	\texttt{smahadev@adobe.com, mahadeva@umass.edu}
}

\hypersetup{
pdftitle={A template for the arxiv style},
pdfsubject={q-bio.NC, q-bio.QM},
pdfauthor={David S.~Hippocampus, Elias D.~Striatum},
pdfkeywords={First keyword, Second keyword, More},
}

\begin{document}
\maketitle

\begin{abstract}
In this paper, we explore the categorical foundations of generative AI. Specifically, we investigate a  Generative AI Architecture (GAIA) that lies beyond backpropagation, the longstanding algorithmic workhorse of deep learning. Backpropagation is at its core a compositional framework for (un)supervised learning: it can be conceptualized as a sequence of modules, where each module updates its parameters based on information it receives from downstream modules, and in turn, transmits information back to upstream modules to guide their updates. GAIA is based on a fundamentally different {\em hierarchical model}. Modules in GAIA are organized into a simplicial complex. Each $n$-simplicial complex acts like a manager of a business unit: it receives updates from its superiors and transmits information back to its $n+1$ subsimplicial complexes that are its subordinates. To ensure this simplicial generative AI organization behaves coherently, GAIA builds on the  mathematics of the higher-order category theory of simplicial sets and objects. Computations in GAIA, from query answering to foundation model building, are posed in terms of lifting diagrams over simplicial objects.  The problem of machine learning in GAIA is modeled as ``horn" extensions of simplicial sets: each sub-simplicial complex tries to update its parameters in such a way that a lifting diagram is solved. Traditional approaches used in generative AI using backpropagation can be used to solve ``inner" horn extension problems, but addressing ``outer horn" extensions requires a more elaborate framework.

At the top level, GAIA uses the simplicial category of ordinal numbers with objects defined as $[n], n \geq 0$ and arrows defined as weakly order-preserving mappings $f: [n] \rightarrow [m]$, where $f(i) \leq f(j), i \leq j$. This top-level structure can be viewed as a combinatorial ``factory" for constructing, manipulating, and destructing complex objects that can be built out of modular components defined over categories. The second layer of GAIA defines the building blocks of generative AI models as universal coalgebras over categories that can be defined using current generative AI approaches, including Transformers that define a category of permutation-equivariant functions on vector spaces, structured state-space models that define a category over linear dynamical systems, or image diffusion models that define a probabilistic coalgebra over ordinary differential equations. The third layer in GAIA is a category of elements over a (relational) database that defines the data over which foundation models are built. GAIA formulates the machine learning problem of building foundation models as extending functors over categories, rather than interpolating functions on sets or spaces, which yields  canonical solutions called left and right Kan extensions.  GAIA uses the metric Yoneda Lemma to construct universal representers of objects in non-symmetric generalized metric spaces. GAIA uses a categorical integral calculus of (co)ends to define two families of  generative AI systems. GAIA models based on coends correspond to topological generative AI systems, whereas GAIA systems based on ends correspond to probabilistic generative AI systems. 
\end{abstract}


\bigskip 

\newpage 

\tableofcontents

\newpage

\section{Overview of the Paper} 

\begin{figure}[h]
\centering
\includegraphics[scale=.25]{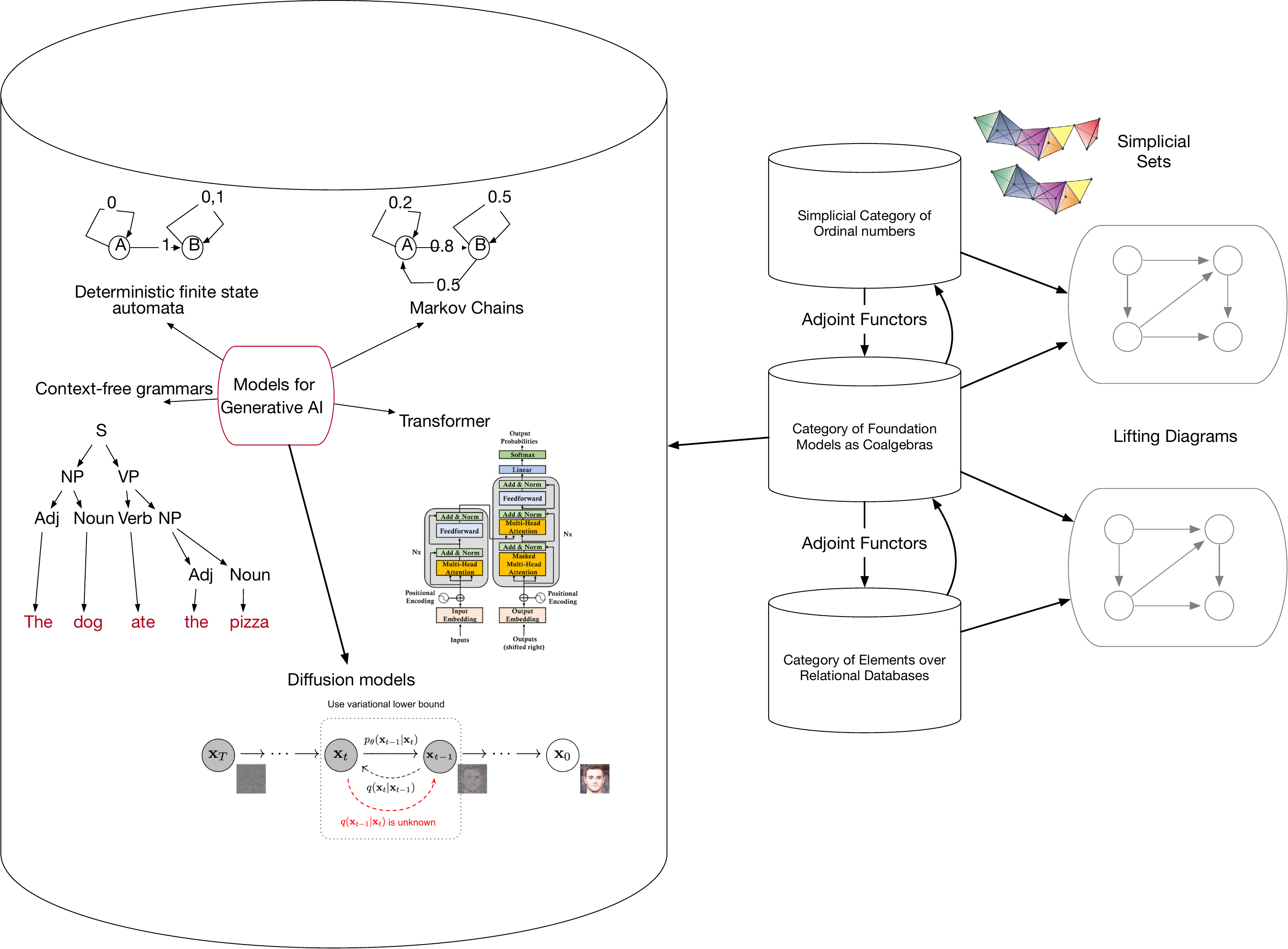}
\caption{We propose a hierarchical Generative AI Architecture (GAIA) using higher-order category theory. }  
\label{gaia} 
\end{figure}

Generative AI has become a dominant paradigm for building intelligent systems in the last few years, ranging from large language models developed with the widely used Transformer model \cite{DBLP:conf/nips/VaswaniSPUJGKP17}, or more recently with the structured state space sequence models \cite{DBLP:conf/iclr/GuGR22,yin2023onestep}, and with the growing use of image diffusion algorithms \cite{DBLP:conf/nips/SongE19,yin2023onestep}. We can broadly define the problem of generative AI as the construction, maintenance, and deployment of foundation models \cite{fm}, a storehouse of human knowledge that provides the basic infrastructure for AI across some set of applications. A fundamental question, therefore, to investigate is to study the mathematical basis for foundation models. We propose a mathematical framework for a Generative AI Architecture (GAIA) (see Figure~\ref{gaia}) based on the hypothesis that {\em category theory} \cite{maclane:71,riehl2017category,Lurie:higher-topos-theory} provides a universal mathematical language for foundation models. In particular, GAIA is based on {\em simplicial learning}, which is intended to generalize {\em compositional} learning frameworks based on well-established machine learning algorithms, such as backpropagation \cite{deeplearningreview-2009}. Category theory has been called a ``Rosetta Stone" \cite{Baez_2010}, as it provides a universal language for defining interactions among objects in all of mathematics, physics, computer science, and mathematical logic. Category theory has recently seen increasing use in machine learning, including dimensionality reduction \cite{umap} and clustering \cite{Carlsson2010}. One unique aspect of defining machine learning as extending functors in a category, in contrast to the well-established previous approach of extending functions over sets or spaces \cite{DBLP:journals/jmlr/WagstaffFEOP22},  is that there are two canonical solutions that emerge -- the left and right Kan extensions \cite{maclane:71} -- whereas there is no corresponding canonical solution to the problem of learning functions over sets.

\begin{figure}
    \centering
    \includegraphics[scale=0.35]{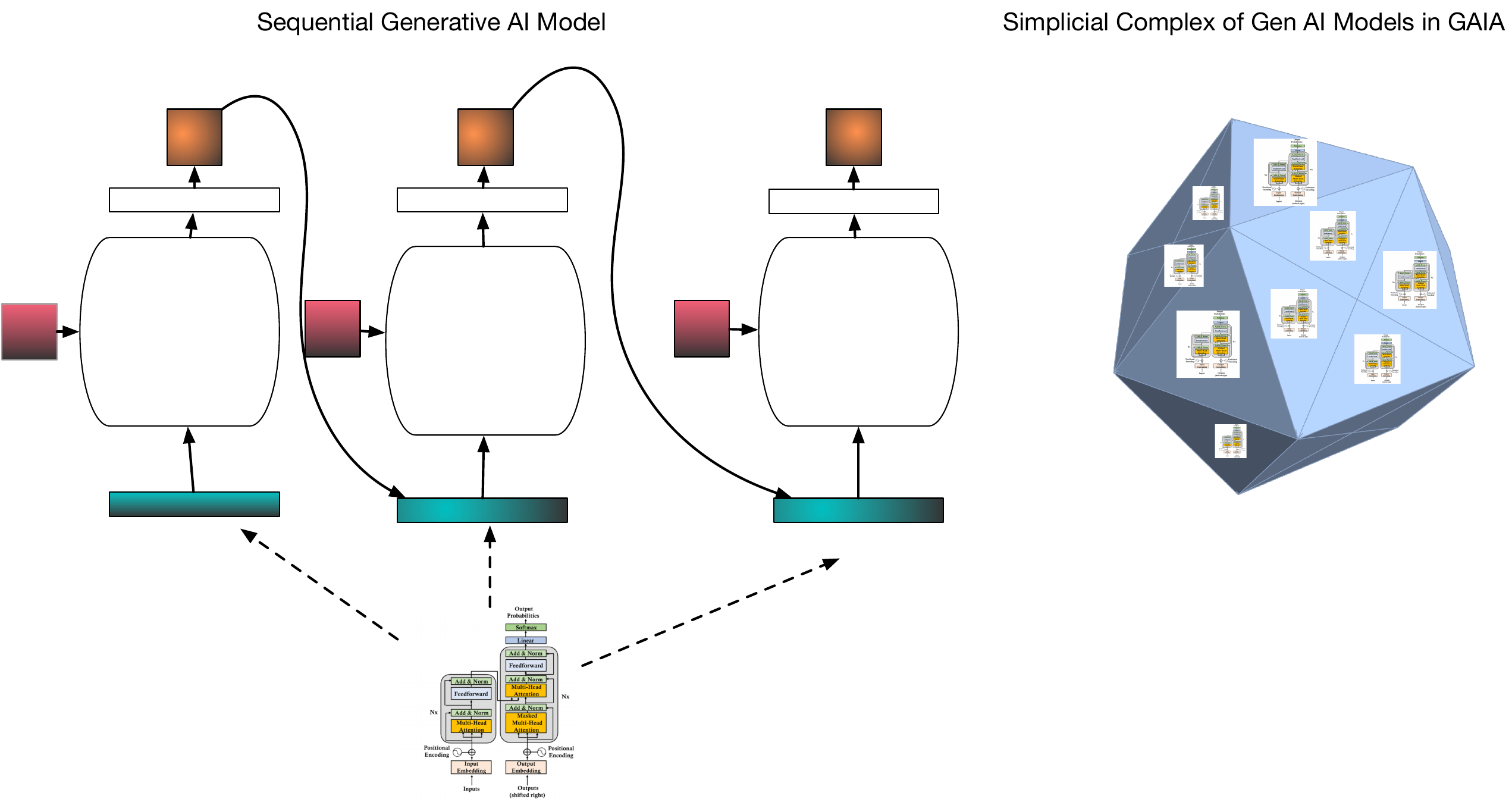}
    \caption{Traditional Generative AI models, such as Transformers, are based on a compositional sequential model. GAIA is based on a simplicial model, where each ``face" of the $n$-simplicial complex defines a generative model.}
    \label{transformersimplex}
\end{figure}

\begin{figure}
    \centering
    \includegraphics[scale=0.25]{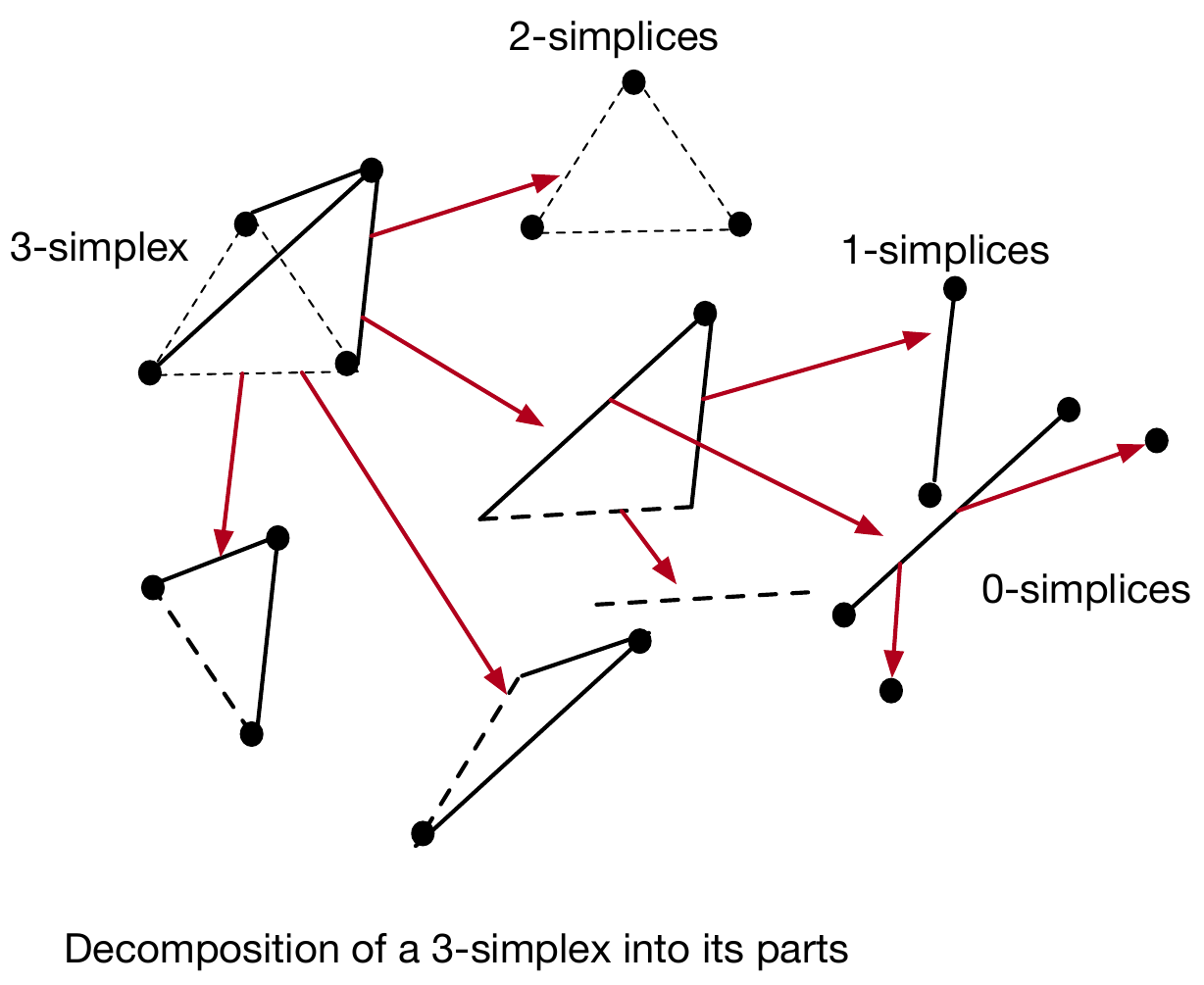}
    \caption{GAIA is based on a {\em hierarchical} framework, where each $n$-simplicial complex acts as a business unit in a company: each $n$-simplex updates its parameters based on data it receives from its superiors, and it transmits guidelines for its $n+1$ sub-simplicial complexes to help them with their updates. The mathematics for this hierarchical framework is based on higher-order category theory of simplicial sets and objects.}
    \label{3simplex}
\end{figure}

Current generative AI systems are built on the longstanding algorithmic workhorse of backpropagation \cite{deeplearningreview-2009}. Backpropagation is fundamentally a compositional sequential framework, where each module updates its parameters based on information it gets from its downstream modules, and in turn, transmits information back to upstream modules. \cite{DBLP:conf/lics/FongST19} propose a categorical framework for backpropagation, which we review in Section~\ref{backprop}.   In this paper, we will generalize this category-theoretic formalization of neural networks  in several ways. As Figure~\ref{transformersimplex} illustrates, unlike traditional generative AI models, such as those developed with sequence models like Transformers or structured state space sequence models, is not sequential, but rather {\em simplicial}. Each ``face" of the $n$-simplex defines a ``local" Transformer model (or indeed, any type of machine learning model), which are then ``stitched" together into the whole structure using the mathematics of higher-order category theory of simplicial objects and sets \cite{may1992simplicial}.

Figure~\ref{backprop-functor} illustrates a crucial conceptual perspective that forms the basis for the design of GAIA. A generative model is, first and foremost, an {\em algebraic structure} of some kind. It may be a sequence, a directed graph, or as in our case, a simplicial complex. This specification is akin to specifying the ``skeleton" of the generative AI model. To give the skeleton some ``flesh and blood", it is necessary to map it into a parameter space (e.g., Euclidean space), through a suitable functor. The actual process of building a foundation model occurs through implementing a learning method, such as backpropagation,  which \cite{DBLP:conf/lics/FongST19} model as a  functor that maps from the space of parameters into the category of learners. A crucial difference in our approach is that we model backpropagation not just as a functor, but rather as an endofunctor on the category of parameters, as it must eventually result in a new set of parameters (see Figure~\ref{backpropendofunctor}. This important difference in our approach makes it possible to apply the rich theory of universal coalgebras over endofunctors \cite{jacobs:book,rutten2000universal} to analyze generative AI methods. We describe in detail in Section~\ref{backprop} one specific instantiation of this general perspective proposed by \cite{DBLP:conf/lics/FongST19}. In their case, the algebraic structure is a symmetric monoidal category that defines the skeleton. Their parameter space is Euclidean space, and their category of learners is defined by compositional learning using backpropagation. In GAIA, we make a more sophisticated framework, where the algebraic structure is a simplicial set or category, the parameter space may be a {\em sheaf} in a {\em topos} \cite{maclane:sheaves}, and the category of learners is defined as horn extensions in a simplicial set. 

\begin{figure}[h] 
\includegraphics[scale=0.5]{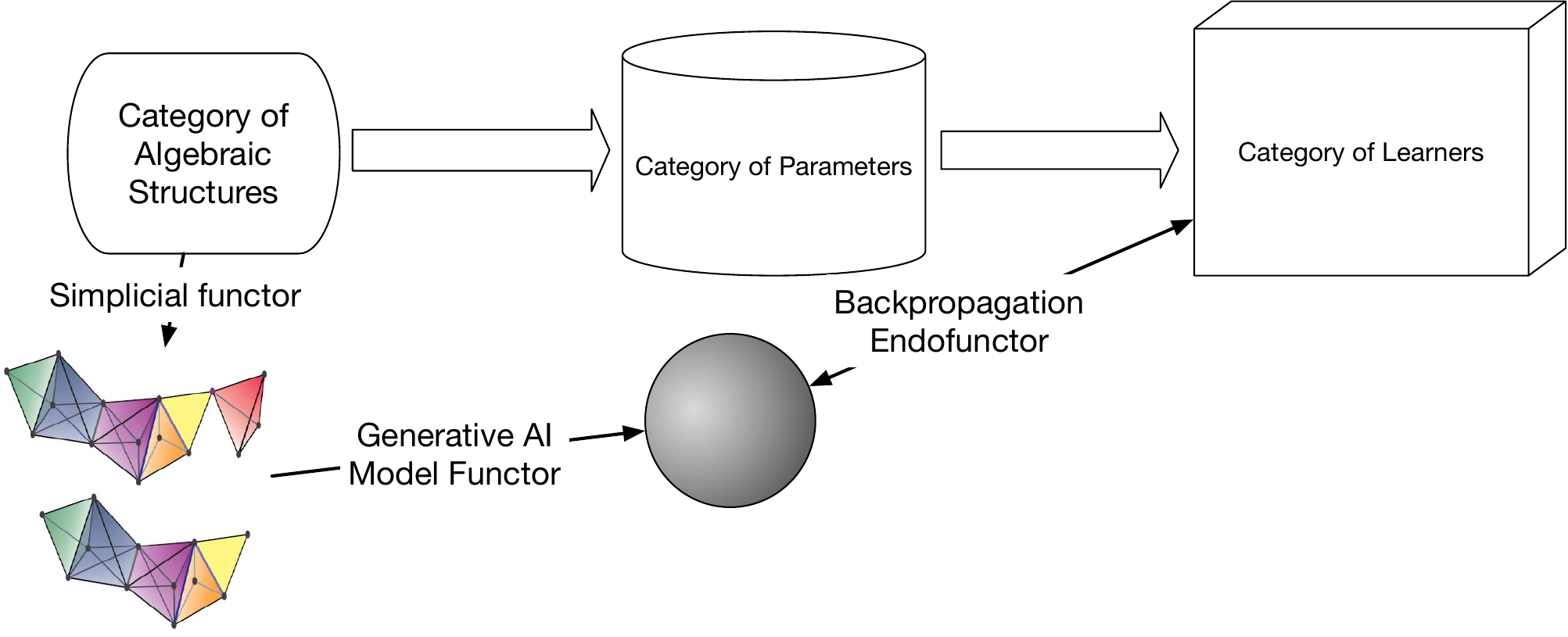}
\caption{Crucial to the GAIA framework is understanding the separation between the algebraic structure of a generative AI model, and the parameter space over which the model is defined, and how specific machine learning algorithms such as backpropagation can be viewed as functors. \cite{DBLP:conf/lics/FongST19} defined backpropagation as a functor as shown from the category {\tt Param} of parameters to the category {\tt Learn} of machine learners. Crucially, GAIA models backpropagation as an {\em endofunctor} from the category {\tt Param} back to itself, because every morphism in {\tt Learn} must result in an update of the parameters of the network, thus resulting in a new object in {\tt Param}. Thus, in this paper, we ``close the loop", opening the rich theory of universal coalgebras defined by endofunctors \cite{jacobs:book,rutten2000universal} to analyze generative AI methods, such as backpropagation.   \label{backprop-functor}}
 \end{figure} 

\begin{figure}
    \centering
       \includegraphics[scale=0.3]{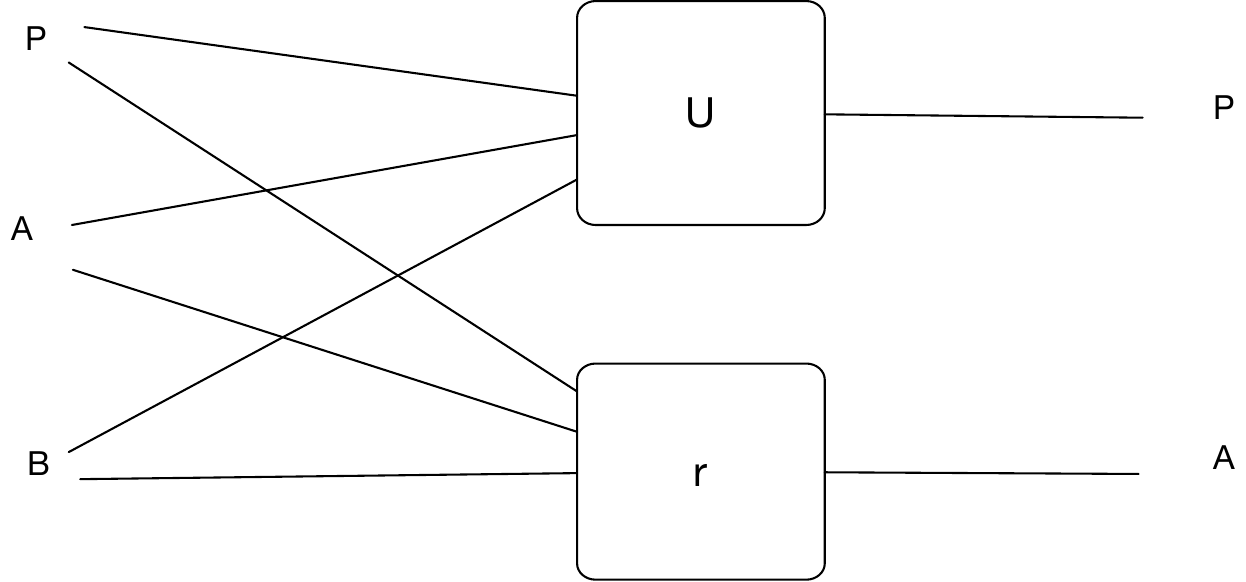}
    \includegraphics[scale=0.3]{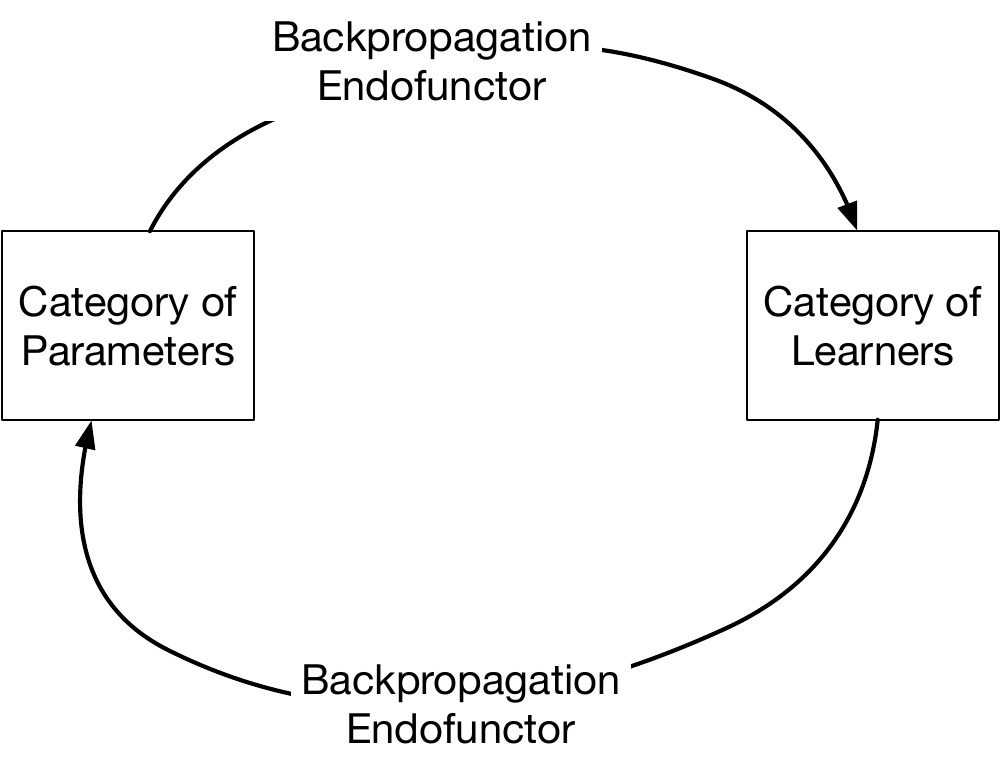}
    \caption{Left: In the categorical framework for deep learning proposed by \cite{DBLP:conf/lics/FongST19}, a learner is a morphism in the category {\tt Learn} that acts sequentially on its input $A$ to produce an output $B$, updating its parameters $P$, and sending back a request $A$ to an upstream module that represents ``backpropagation".  In GAIA, we view backpropagation as a {\em coalgebra} \cite{rutten2000universal}, defined by an endofunctor on the category of parameters, so that each step of backpropagation is modeled as a dynamical system that maps some parameter object into a new parameter object. }
    \label{backpropendofunctor}
\end{figure}

GAIA is based on a hierarchical organization, much like the business units in a company. An $n$-simplex defines a collection of $n+1$ sub-simplicial sets (or object), and each $n$-simplex computes some function based on a set of parameters, which it updates based on information it receives from its superiors. It then transmits guidelines to its $n+1$ subordinate sub-simplicial sets on how to update their parameters.  We use the simplicial category $\Delta$ at the top layer of GAIA to define not just sequences of morphisms, each representing a layer of a generative AI network, but simplicial complexes of them. One way to understand the connection between categories and simplicial sets is through the {\em nerve} functor that maps sequences of morphisms -- for example, each representing a Transformer block or a diffusion image generation step -- into a simplicial set. The $n$-simplices are defined by sequences of composable morphisms of length $n$. It can be shown that the nerve functor is a full and faithful functor that fully captures the category structure as a simplicial set \cite{Lurie:higher-topos-theory}. However, the left adjoint of the nerve functor is a ``lossy" inverse, in that it only preserves structure for $n$-simplices, where $n \leq 2$.  GAIA defines generative AI over $n$-simplicial complexes that allow more complex interactions among them than that which can be modeled by compositional learning frameworks, such as backpropagation.  Simplicial sets are defined as a graded set $S_n, n \geq 0$, where $S_0$ represents ``objects", $S_1$ represent morphisms (as in \cite{DBLP:conf/lics/FongST19}), $S_2$ define triangles of composable morphisms that have to be filled in different ways, and constitute ``inner" and ``outer" horn extension problems \cite{Lurie:higher-topos-theory}. In summary, GAIA generalizes the category-theoretic framework for deep learning in \cite{DBLP:conf/lics/FongST19} to a higher-order category of simplicial sets and objects.  Figure~\ref{simplicialgenAI} illustrates the simplicial generative AI vision underlying GAIA. 

\begin{figure}
    \centering
    \includegraphics[scale=0.35]{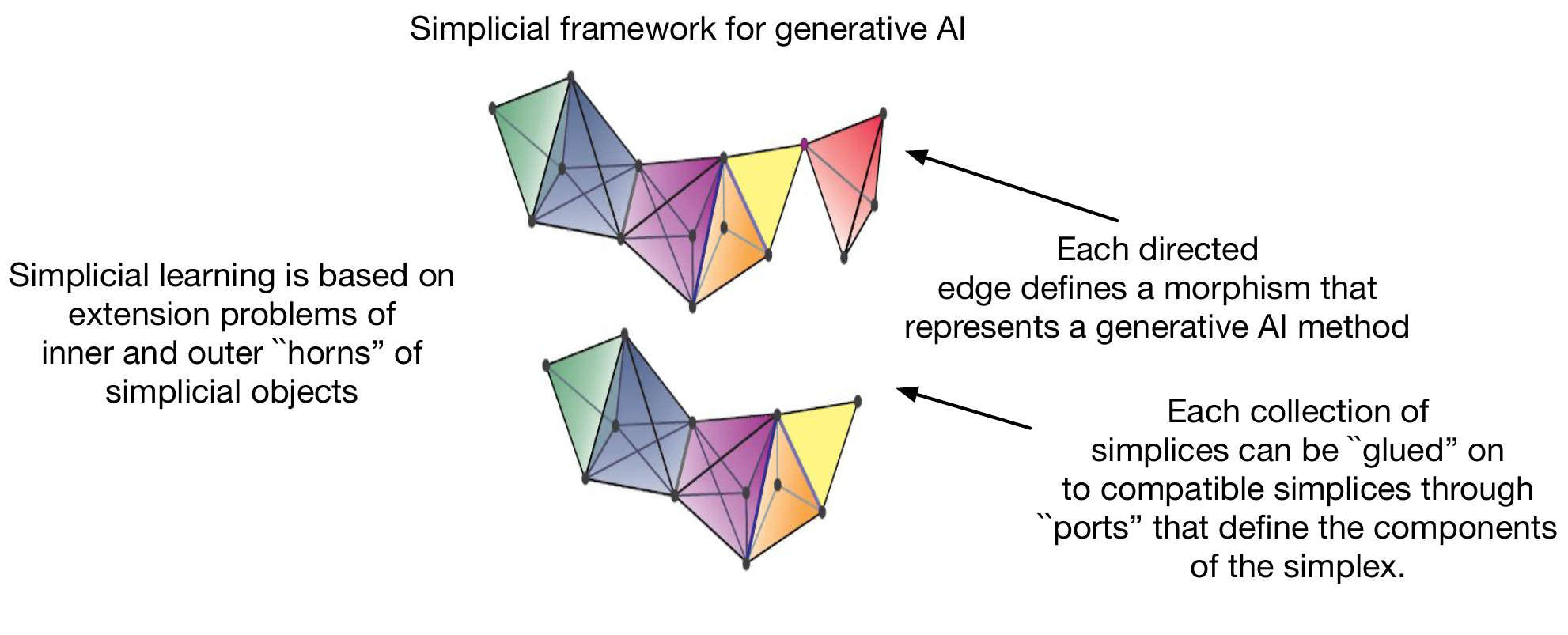}
    \caption{GAIA is based on a simplicial framework, where each generative AI method is modeled as a morphism that maps between two objects. In the simplest case of compositional learning, a $1$-simplex is defined as an ``edge", where its beginning and ending ``vertices" represent data that flows into and out of a generative AI model, such as a Transformer, or a structured state space sequence model, or a diffusion process. Backpropagation can be used to solve compositional learning problems over such sequences of ``edge" building blocks. GAIA generalizes this paradigm to define ``higher-order" simplicial objects where the interfaces between simplices can be more elaborate. Each $n$-simplex is comprised of a family of $n-1$ subsimplicial objects, each of which can be ``glued" together to form the $n$-simplex.}
    \label{simplicialgenAI}
\end{figure}

Backpropagation can solve a wide range of generative AI problems defined as supervised learning problems, where the task is to infer an unknown function $f: A \rightarrow C$ from samples $(a, f(a))$, where $f$ is constructed as a sequential composition of building block unknown functions that represent intermediate targets such as $f \simeq g \circ h$, and $h: A \rightarrow B$, and $g: B \rightarrow C$. Such problems are modeled in GAIA as ``inner horn" extensions of simplicial objects. GAIA is able to formulate ``outer horn" extension problems, such as inferring unknown functions $f: B \rightarrow C$ from samples of  unknown functions $g: A \rightarrow B$ and $h: A \rightarrow C$, or infer unknown functions $f: A \rightarrow B$ given samples of  unknown functions $g: A \rightarrow C$ and $h: B \rightarrow C$, which lie outside the scope of sequential compositional methods like backpropagation. One example of a setting where both inner and outer horn extensions are solvable are in Kan complexes. 

We can define the class of ``horn extensions" of simplicial complexes, where each morphism might represent a generative AI morphism (such as in the category {\tt Learn} considered by \cite{DBLP:conf/lics/FongST19}), which is essentially all the ways of composing $1$-dimensional simplices  to form a $2$-dimensional simplicial object. Each simplicial subset of an $n$-simplex induces a  a {\em horn} $\Lambda^n_k$, where  $ 0 \leq k \leq n$. Intuitively, a horn $\Lambda^n_k$ is a subset of a simplicial object that results from removing the interior of the $n$-simplex and the face opposite the $k$th vertex. Consider the three horns defined below. The dashed arrow  $\dashrightarrow$ indicates edges of the $2$-simplex $\Delta^2$ not contained in the horns. 
\begin{center}
 \begin{tikzcd}[column sep=small]
& \{0\}  \arrow[dl] \arrow[dr] & \\
  \{1 \} \arrow[rr, dashed] &                         & \{ 2 \} 
\end{tikzcd} \hskip 0.5 in 
 \begin{tikzcd}[column sep=small]
& \{0\}  \arrow[dl] \arrow[dr, dashed] & \\
  \{1 \} \arrow{rr} &                         & \{ 2 \} 
\end{tikzcd} \hskip 0.5in 
 \begin{tikzcd}[column sep=small]
& \{0\}  \arrow[dl, dashed] \arrow[dr] & \\
  \{1 \} \arrow{rr} &                         & \{ 2 \} 
\end{tikzcd}
\end{center}

The inner horn $\Lambda^2_1$ is the middle diagram above, and admits an easy solution to the ``horn filling'' problem of composing the simplicial subsets. In defining a compositional category for neural networks and supervised learning, \cite{DBLP:conf/lics/FongST19} only consider ``inner horn" extension problems defined as how to compose two morphism in the category {\tt Learn}. In other words, if $f: A \rightarrow B$ and $g: B \rightarrow C$ are two functions to be learned from a database of samples, their framework works out the updates to the composition $g \circ f: A \rightarrow C$. 

The two outer horns on either end pose a more difficult challenge. For example, filling the outer horn $\Lambda^2_0$ when the morphism between $\{0 \}$ and $\{1 \}$ is $f$ and that between $\{0 \}$ and $\{2 \}$ is the identity ${\bf 1}$ is tantamount to finding the left inverse of $f$ up to homotopy. Dually, in this case, filling the outer horn $\Lambda^2_2$ is tantamount to finding the right inverse of $f$ up to homotopy. A considerable elaboration of the theoretical machinery in category theory is required to describe the various solutions proposed, which led to different ways of defining higher-order category theory \cite{weakkan,quasicats,Lurie:higher-topos-theory}. 

As examples, consider a large language model (LLM) that is trained to output programs based on textual inputs. For example, given data corresponding to possible textual descriptions of programs and actual code, a GitHub Copilot can generate programs from textual inputs. An outer horn problem for this case would be to infer an unknown function between two generated sample programs from their textual prompts as inputs. Similarly, for a generative AI program that produces images by diffusion from textual inputs, the outer horn problem might correspond to learning an unknown function between two generated images. 

If we assume that the simplicial complex is a Kan complex \cite{kan}, all horn extensions can be solved, which intuitively can be understood as implying that the outer horn extension problems can be turned into inner horn extensions. So, for example, we can solve the outer horn problem defined by the first diagram on the left above by assuming that the morphism $f: [0] \rightarrow [1]$ has an inverse $f^{-1}: [1] \rightarrow [0]$, and hence turn the outer horn into an inner horn problem. Similarly, for the outer horn problem on the right hand side of the above diagram, we can assume that morphism $f: [0] \rightarrow [2]$ has an inverse $f^{-1}: [2] \rightarrow [0]$ that converts it back into an inner horn extension problem.

Note that neural networks that are trained through backpropagation are inherently {\em directional}: there is a well-defined notion of an input and an output over which forward and backwards propagation occurs. In essence, what outer horn extension problems imply is that if there exists a solution to a problem of inferring an unknown function $f: A \rightarrow B$ from samples $(a, f(a)) \in A \times B$, does that imply a solution to the problem of inferring an inverse function $f^{-1}: B \rightarrow A$? Lastly, it must be noted that horn extensions can be more complex than the simple $2$-simplex case described above. In general, in a lifting diagram, defined below, we are asking if a solution exists to an arbitrary lifting problem in a certain category? 

We define the update process through lifting diagrams from algebraic topology \cite{lifting} as a unifying framework, from answering queries to building foundation models. A lifting diagram defines constraints between different paths that lead from one category to another.  They have been used to formulate queries in relational databases \cite{SPIVAK_2013}. In our previous work, we used lifting diagrams to define queries for causal inference  \cite{DBLP:journals/entropy/Mahadevan23}. Lifting problems define ways of decomposing structures into simpler pieces, and putting them back together again, and thus play a central role in GAIA (see Figure~\ref{topologylift}. 
\begin{figure}
    \centering
    \includegraphics[scale=0.3]{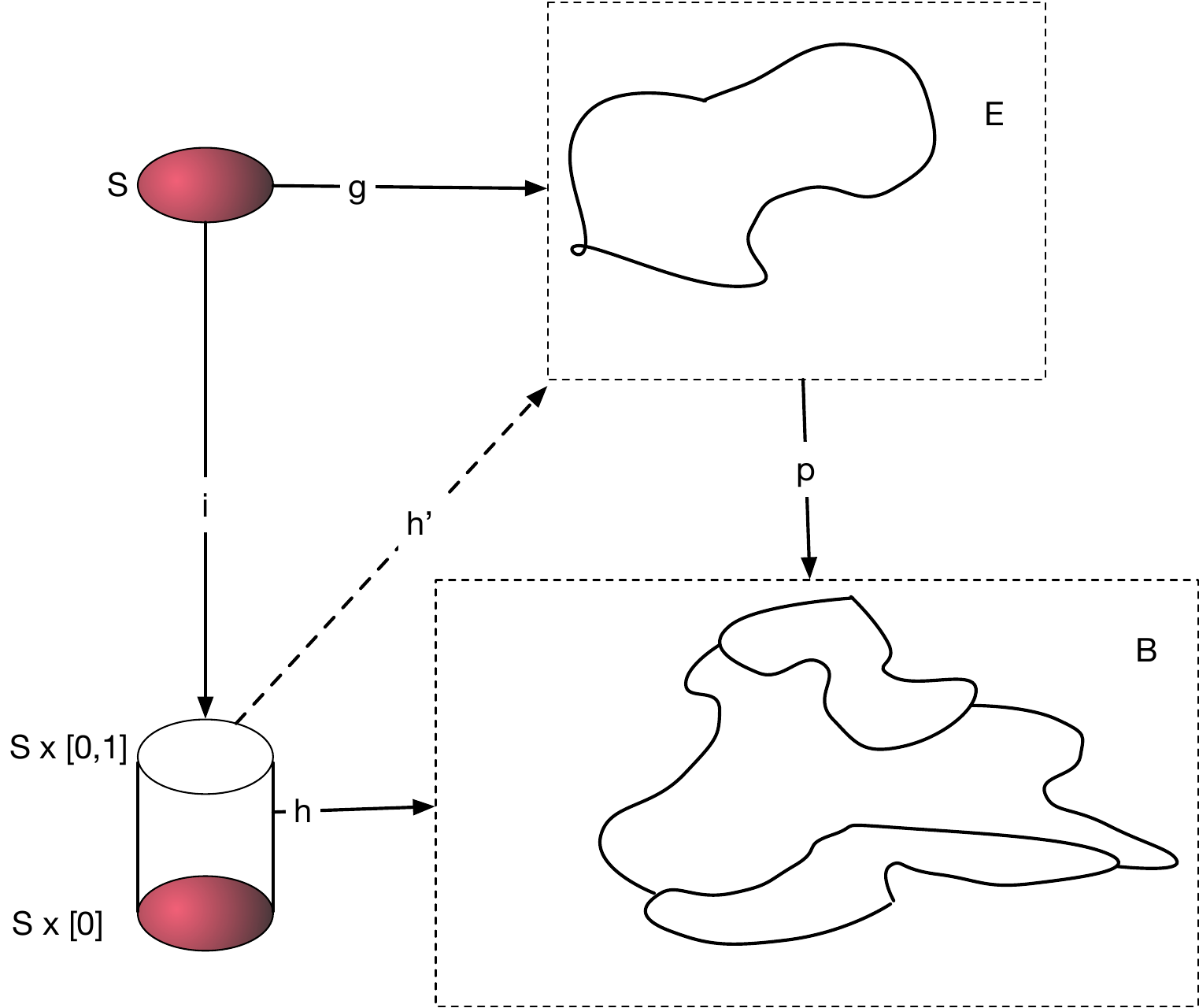}
    \caption{Lifting diagrams were originally studied in algebraic topology \cite{lifting}, and provide a concise way to define diverse computational problems in GAIA. A map $p: E \rightarrow B$ is called a {\em fibration} if and only if for any maps $h$ and $g$ that make this diagram ``commute", there exists a diagonal map $h'$ that makes the whole diagram commute. Fibrations have been used to formalize SQL queries in relational databases \cite{SPIVAK_2013} and causal inference \cite{DBLP:journals/entropy/Mahadevan23}, and are central to homotopy theory in higher-order category theory. Many universal approximability results in deep learning \cite{yarotsky, DBLP:journals/jmlr/WagstaffFEOP22,DBLP:conf/iclr/YunBRRK20} can be phrased in terms of lifting diagrams.}
    \label{topologylift}
\end{figure}
 
 \begin{definition}
 Let ${\cal C}$ be a category. A {\bf {lifting problem}} in ${\cal C}$ is a commutative diagram $\sigma$ in ${\cal C}$. 
 \begin{center}
 \begin{tikzcd}
  A \arrow{d}{f} \arrow{r}{\mu}
    & X \arrow[]{d}{p} \\
  B  \arrow[]{r}[]{\nu}
&Y \end{tikzcd}
 \end{center} 
 \end{definition}

 To understand the meaning of such a diagram for generative AI, let us consider the setting where every edge in the above commutative diagram represents an instance of a Transformer module that maps a object $x \in \mathbb{R}^{d \times n}$ into another using a permutation-equivariant mapping. \cite{DBLP:conf/iclr/YunBRRK20} show that Transformers compute permutation-equivariant functions and are nonetheless universal approximators in the space of all continuous functions on $\mathbb{R}^{d \times n}$ due to their reliance on absolute positional encoding \cite{DBLP:conf/nips/VaswaniSPUJGKP17} to overcome the limitations imposed by permutation equivariance.  Permutation-equivariant functions are defined as $f: \mathbb{R}^{d \times n} \rightarrow \mathbb{R}^{d \times n}$ such that $f(XP) = f(X) P$ for any $X \in \mathbb{R}^{d \times n}$ where $P$ is a permutation matrix. It is straightforward to define a category of Transformer models, where the objects are vectors $X \in \mathbb{R}^{d \times n}$ and the arrows are permutation-invariant mappings. Similarly, diffusion models used in image generation \cite{DBLP:conf/nips/SongE19} can be viewed as a category of stochastic dynamical systems that can be viewed as probabilistic coalgebras \cite{SOKOLOVA20115095}. \cite{rutten2000universal} and \cite{jacobs:book} show that a wide class of dynamical systems used in computer science can be expressed as categories of universal coalgebras. With this context, asking for solutions to lifting problems is posing a question on the representational adequacy of a framework for generative AI. 
 
 \begin{definition}
 Let ${\cal C}$ be a category. A {\bf {solution to a lifting problem}} in ${\cal C}$ is a morphism $h: B \rightarrow X$ in ${\cal C}$ satisfying $p \circ h = \nu$ and $h \circ f = \mu$ as indicated in the diagram below. 
 \begin{center}
 \begin{tikzcd}
  A \arrow{d}{f} \arrow{r}{\mu}
    & X \arrow[]{d}{p} \\
  B \arrow[ur,dashed, "h"] \arrow[]{r}[]{\nu}
&Y \end{tikzcd}
 \end{center} 
 \end{definition}

Although lifting diagrams have been proposed for deep learning architectures recently \cite{papillon2023architectures}, it is important to stress that the notion of simplicial complex used in this paper, as well as other notions used in computer graphics  is fundamentally different from our paper. In our case, simplicial complexes are directional, since each edge is defined by a directional arrow, not an undirected arrow. The edges correspond to morphisms, and indeed, a basic question that we will ask is under what circumstances can directional morphisms be inverted. This question is fundamental to solving extension problems, and in really nice settings like Kan complexes, all morphisms can be inverted since all inner and outer horn extension problems have solutions. 

\begin{figure}
    \centering
    \includegraphics[scale=0.4]{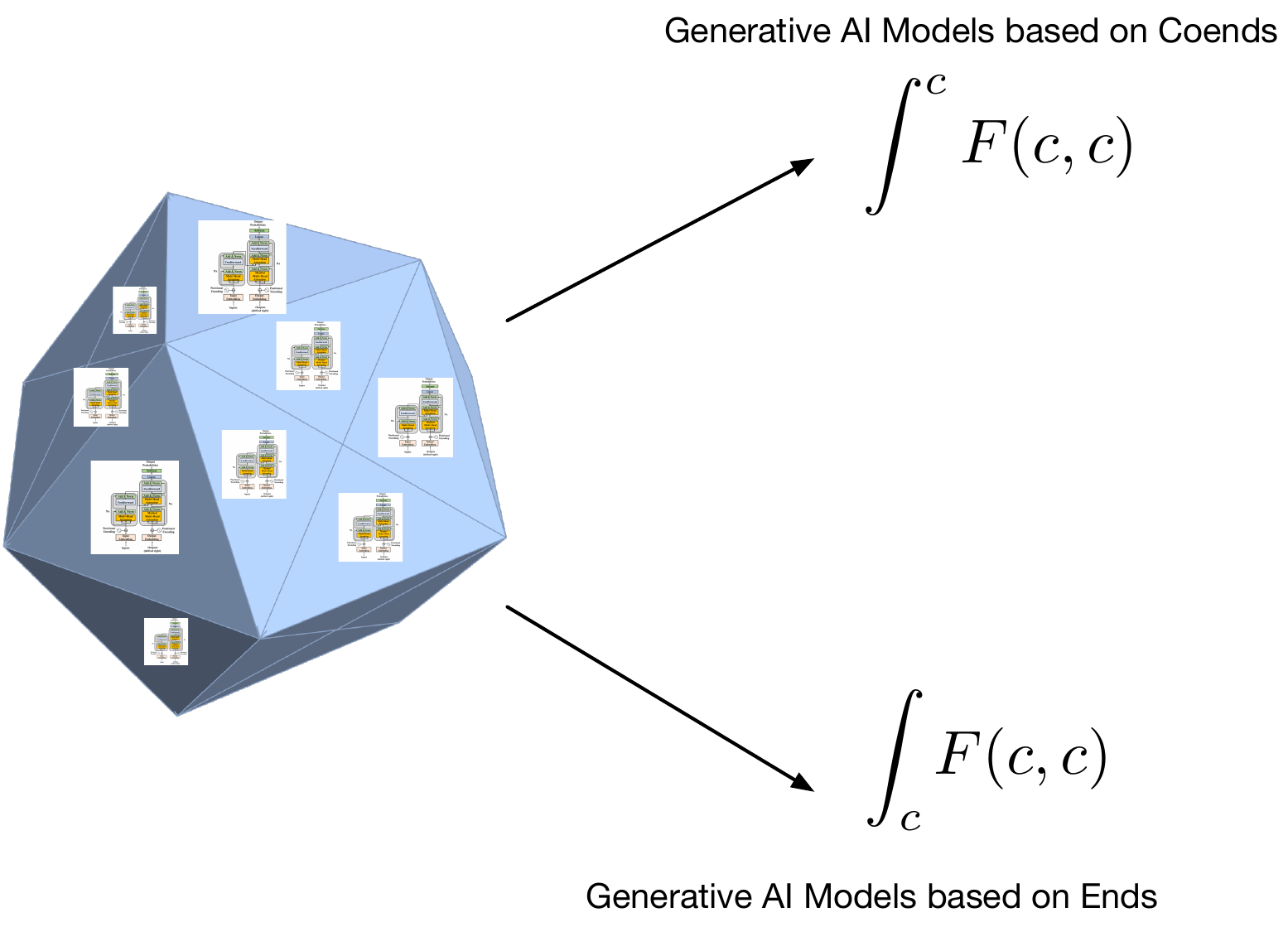}
    \caption{We propose two families of GAIA models in this paper (see Section~\ref{coend} for details),  based on coends and ends \cite{loregian_2021}. In this diagram, the bifunctor $F \in \mbox{Cat}({\cal C}^{op} \times {\cal C}, {\cal D})$ acts both contravariantly and covariantly on objects in the category ${\cal C}$. Coend and end objects correspond to objects in the category ${\cal D}$. Coend GAIA models are based on topological realizations of the simplicial model, whereas end GAIA models are based on probabilistic generative models.} 
    \label{fig:enter-label}
\end{figure}

In the literature on higher-order category theory \cite{quasicats,Lurie:higher-topos-theory,kerodon} and homotopy theory \cite{gabriel1967calculus,richter2020categories,Quillen:1967}, lifting diagrams were used to define structures such as Kan complexes that possess nice extension properties. For example, for any given topological space $X$, there is a functor that defines a simplicial set $\mbox{Sing}_\bullet(X)$ defined by all continuous functions from the topological simplex $\Delta_n$ into the space $X$. The simplicial set $\mbox{Sing}_\bullet(X)$ is a Kan complex because all ``horn extensions" of simplicial subsets can be solved. In practical terms, this property explains the success of dimensionality reduction methods like UMAP \cite{umap}, which constructs functors from simplicial sets that represent data into topological spaces. The topological realization of simplicial sets computed by UMAP is an example of a coend, a unifying ``integral calculus" proposed originally by Yoneda \cite{yoneda-end}. \cite{loregian_2021} gives a detailed description of the integral calculus of (co)ends, which provides a unifying way to design an entire spectrum of generative AI systems, where systems based on ends leads to probabilistic generative models like Transformers, whereas models based on coends lead to topologically based generative AI systems that have not been explored in the literature. 

In this paper, we use the framework of lifting diagrams to both formulate queries as shown in Figure~\ref{gaia}, but also to define the problem of building foundation models from data. In particular, we build on the framework of simplicial sets and objects, where we pose lifting diagram queries as solving ``horn extension" problems \cite{Lurie:higher-topos-theory}. Formally, we pose the problem of learning in a generative AI system in terms of properties such as Kan complexes \cite{may1992simplicial}, which are ideal categories to solve lifting problems since there is a unique solution to all extension problems (inner and outer horn extensions). As a concrete example, every topological space can be mapped into a simplicial set using the singular functor, which was used in the UMAP \cite{umap} dimensionality reduction method, which can be shown to form a Kan complex. 

Beginning at the top layer, GAIA uses the simplicial category of ordinal numbers \cite{may1992simplicial} as a way to build, manipulate, and destroy compositional structures. The category of ordinal numbers $\Delta$ includes a collection of objects indexed by the ordinals $[n], n \geq 0$, where $[n] = (0, 1, \ldots, n)$ under the natural ordering $<$. The morphisms of $\Delta$ are order-preserving functions $f: [n] \rightarrow [m]$ where $f(i) \leq f(j), i \leq j, i, j \in [n]$. The category $\Delta$ has provided the basis for a combinatorial approach to topology and also serves as the basis for higher-order category theory \cite{quasicats,Lurie:higher-topos-theory,gabriel1967calculus}. This category ``comes to life" when it is functorially mapped to some other category, such as {\bf Sets}, when the resulting structure is called a simplicial set. The contravariant functor $X: \Delta^{op} \rightarrow {\bf Sets}$ is defined by viewing $X([0])$ as a set of ``objects", $X[1])$ as a set of ``arrows" representing pairwise interactions among the objects, and in general, $X([n])$ -- which is often simply written as $X_n$ and consists of a set of $n$-simplices -- defines interactions of order $n$ among the objects. Simplicial sets generalize directed graphs, partial orders, sequences, and in fact, regular categories \cite{maclane:71} as well. There are constructor and destructor morphisms that map from $X([n])$ to $X([n+1])$ and $X([n])$ to $X(n-1])$, which are usually denoted as degeneracy and face operators. 

The second layer of GAIA defines a category of generative AI models, which can be composed of any of the standard technologies used to build generative models, including finite and probabilistic automata, context-free grammars, structured state-space sequence models \cite{DBLP:conf/iclr/GuGR22} or Transformer models \cite{DBLP:conf/nips/VaswaniSPUJGKP17}, or cellular automata \cite{DBLP:journals/jca/Vollmar06,wolfram:book}. We assume that each of these models defines a category whose objects and arrows can be composed and otherwise manipulated by the simplicial category $\Delta$.

The third layer of GAIA defines the category of (relational) databases out of which the category of foundation models is build (e.g., such as by using one of the standard generative AI methods, such as self-attention \cite{DBLP:conf/nips/VaswaniSPUJGKP17} or structured state-space sequence models \cite{DBLP:conf/iclr/GuGR22}). \cite{Spivak_2012} has shown that categories provide a foundation for defining relational databases, and that many common operations in databases can be defined in terms of lifting diagrams in topology \cite{lifting}. In particular, a fundamental premise of GAIA is that machine learning is defined as the extension of {\em functors} on categories, not functions on sets \cite{DBLP:journals/jmlr/WagstaffFEOP22}. The fundamental reason to view machine learning as extending functors is that there are two canonical solutions to the problem of extending functors, defined as left and right Kan extensions \cite{kan}. In contrast, there is no obvious or natural solution to the problem of extending functions on sets, which has prompted an enormous literature in the field of machine learning over many decades, and also in fields like information theory \cite{chaitin,cover}.  Lifting diagrams provide an elegant and general framework to pose the problem of generalization for generative AI, based not just on individual units of experience, but by providing a theoretically sound way to do generalization over arbitrary relational structures. Much of the work in machine learning has focused on the ability to generalize propositional representations, which also includes most of the work in statistics. To generalize over first-order relational structures requires bringing in some powerful tools from algebraic topology and higher-order category theory, in particular the ability to do {\em horn filling} of simplicial horns \cite{Lurie:higher-topos-theory}. 

\subsection{Roadmap to the Paper}

Given the length of this paper, a roadmap to its organization will be helpful to the reader. Keep in mind that this paper is a condensed version of a forthcoming book, which is designed to provide a detailed tutorial level introduction to category theory, in addition to illustrating its application to generative AI. With that mind, Section~\ref{backprop} begins us off with a detailed look at a category theory of deep learning, building on the work of \cite{DBLP:conf/lics/FongST19}. The crucial idea of separating the algebraic structure of a generative AI model from its parameterization, which in turn is independent of the structure of a learning framework is crucial to our framework as well, although GAIA differs in many ways from the approach proposed in \cite{DBLP:conf/lics/FongST19}. 

Section~\ref{endofunctor} gives our alternative view of backpropagation as an endofunctor, in particular a universal coalgebra of the form $X \rightarrow F(X)$, where the endofunctor $F$ maps objects $X$ in a category ${\cal C}$ back to the same category. Universal coalgebras \cite{jacobs:book,rutten2000universal} provide a rich language for specifying dynamical systems, and they have also been extended to describe probabilistic generative models, such as Markov chains \cite{SOKOLOVA20115095}, Markov decision processes \cite{feys:hal-02044650} and a wealth of programming-related abstractions \cite{jacobs:book}. 

Section~\ref{layer1} defines the simplicial layer of GAIA, which acts like a ``combinatorial factory" that can assemble together pieces of generative AI models. The heart of the GAIA framework is that the simplicial category allows a hierarchical framework for generative AI, which we believe goes beyond the purely sequential framework thus far studied in the literature. We illustrate how hierarchical learning works in GAIA in terms of lifting problems in simplicial sets. Section~\ref{layer2} defines particular categories for generative AI, including the popular Transformer architecture as permutation equivariant functions over Euclidean spaces.  Section~\ref{layer3} defines universal properties and the Yoneda Lemma, which are used to define universal parameterizations of generative AI models. In particular, we show that non-symmetric generalized metric spaces can be studied with the metric Yoneda Lemma, which has applications in constructing non-symmetric attention models for natural language. 

Section~\ref{coend} defines an abstract integral calculus for generative AI, based on Yoneda's pioneering work \cite{yoneda-end}. \cite{loregian_2021} gives a detailed textbook level account of (co)end calculus. We define two classes of generative AI models, those based on coends and ends. We show that coend GAIA models lead to topological realizations, whereas GAIA models based on ends lead to probabilistic generative AI models.  We also introduce sheaves and topoi as alternative parameterizations of generative AI models, which arise from the Yoneda Lemma, and can give additional structure over simply using Euclidean spaces. Section~\ref{homotopy} finally defines abstract notions of equivalence in category theory for comparing two generative AI models. When can we say, for example, that a summarized document produced by a generative AI copilot is actually faithful to the original document on which it was based? Homotopy theory provides some answers to these questions, which have only been studied empirically in the literature. We introduce the notion of a {\em classifying space} for generative AI models, and define their homotopy colimits. 

The paper covers a great deal of abstract mathematics, but we have attempted to provide a range of concrete examples of its application to the problem of generative AI. Many more examples can be given, but would significantly increase the length of an already really long paper!  Ultimately, as we conclude at the end, the real proof of the utility of GAIA will come from its actual implementation as a working system, but we view that as a multi-year research problem. There are many open problems that are remaining to be worked out, and we discuss a few of them in the paper at various places. 

\section{Backpropagation as a Functor: Compositional Learning}  

\label{backprop}

Our principal goal in this section to review the categorical framework for deep learning proposed in \cite{DBLP:conf/lics/FongST19}, which models backpropagation as a functor. In the next section, we will argue that backpropagation should  be viewed instead as an {\em endufunctor} on the category {\em Param}, which defines the space over which generative AI model are defined. 

We first  give a high-level introduction to generative AI, building on the framework of category theory (see Figure~\ref{functors}).  Category theory is intrinsically a framework for compositional structures, which generative AI exemplifies as well. Excellent textbook length treatments are readily available and should be consulted for further background \cite{maclane:71,maclane:sheaves,riehl2017category,richter2020categories}. We summarize salient concepts from category theory as and when needed. We define several types of categories in this section, beginning with a category for supervised learning, and then other categories that represent machine learning algorithms, including the traditional backpropagation algorithm, and zeroth-order optimization, as well as categories for specific deep learning architectures, such as Transformers, structured state space sequence models, and diffusion models. We then introduce some key ideas from category theory, including the fundamental Yoneda Lemma that shows all objects in a category can be characterized in terms of their interactions. 

\begin{figure}[h]
\centering
\includegraphics[scale=.4]{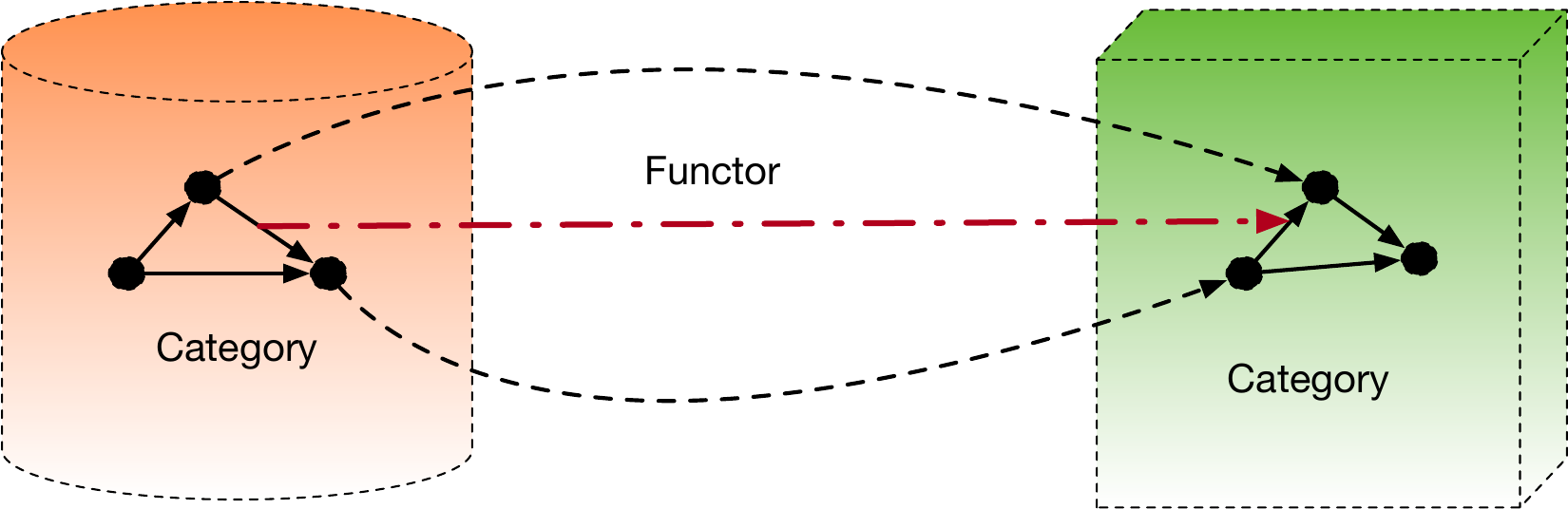}
\caption{Categories are defined by collection of arbitrary objects that interact through morphisms (also called arrows). Functors map objects from one category into another, but also map the arrows of the domain category into corresponding arrows in the co-domain category. We define generative AI systems and learning algorithms as categories in GAIA.}
\label{functors} 
\end{figure}

\subsection{Category of Supervised Learning}

\cite{DBLP:conf/lics/FongST19} give an elegant characterization of the well-known backpropagation algorithm that serves as the ``workhorse" of deep learning as a functor  over symmetric monoidal categories. In such categories, objects can be ``multiplied": for example, sets form a symmetric monoidal category as the Cartesian product of two sets defines a multiplication operator. A detailed set of coherence axioms are defined for monoidal categories (see \cite{maclane:71} for details), which we will not go through, but they ensure that multiplication is associative, as well as that there are identity operators such that $I \otimes A \simeq A$ for all objects $A$, where $I$ is the identity object. 

\begin{figure}
    \centering
    \includegraphics[scale=0.4]{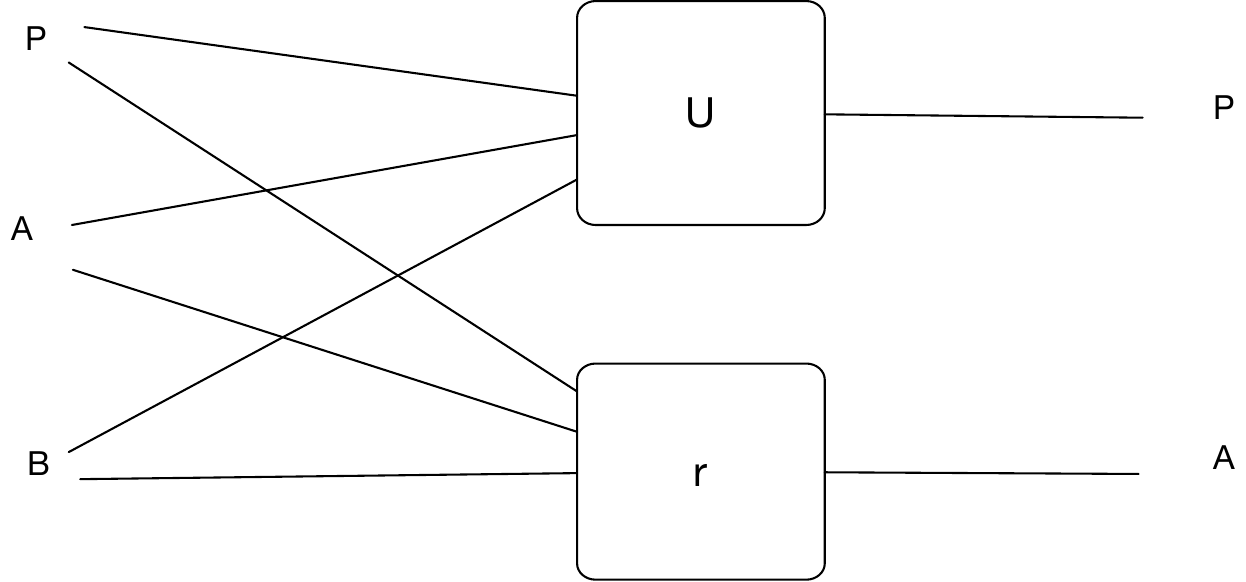}
    \caption{A learner in the symmetric monoidal category {\tt Learn} is defined as a morphism. Later in Section~\ref{endofunctor}, we will see how to define learners as coalgebras instead.}
    \label{learncat}
\end{figure}

\begin{definition}\cite{DBLP:conf/lics/FongST19}
    The symmetric monoidal category {\bf Learn} is defined as a collection of objects that define sets, and a collection of an equivalence class of learners. Each learner is defined by the following $4$-tuple (see Figure~\ref{learncat}). 
    
\begin{itemize}
    \item A parameter space $P$

    \item An implementation function $I: P \times A \rightarrow B$

    \item An update function $U: P \times A \times B \rightarrow P$

    \item A request function $r: P \times A \times B \rightarrow A$
\end{itemize}

Note that it is the request function that allows learners to be composed, as each request function transmits information back upstream to earlier learners what output they could have produced that would be more ``desirable". This algebraic characterization of the backpropagation algorithm clarifies its essentially compositional nature 

Two learners $(P, I, U, R)$ and $(P', I', U', r')$ are equivalent if there is a bijection $f: P \rightarrow P'$ such that the following identities hold for each $p \in P, a \in A$ and $b \in B$. 

\begin{itemize}
    \item $I'(f(p), a) = I(p, a)$. 

    \item $U'(f(p), a, b) = f(U(p, a, b))$. 

    \item $r'(f(p), a, b) = r(p, a, b)$
\end{itemize}
    
\end{definition}

Typically, in generative AI trained with neural networks, the parameter space $P = \mathbb{R}^N$ where the neural network has $N$ parameters. The implementation function $I$ represents the ``feedforward" component, and the request function represents the ``backpropagation" component. The update function represents the change in parameters as a result of processing a training example $(a, f(a)) \in A \times B$. The main contribution of \cite{DBLP:conf/lics/FongST19} is in showing that supervised learning can be defined as a compositional category under the sequential composition of morphisms defining individual building blocks of learners. 

\begin{figure}[t]
    \centering
    \includegraphics[scale=0.4]{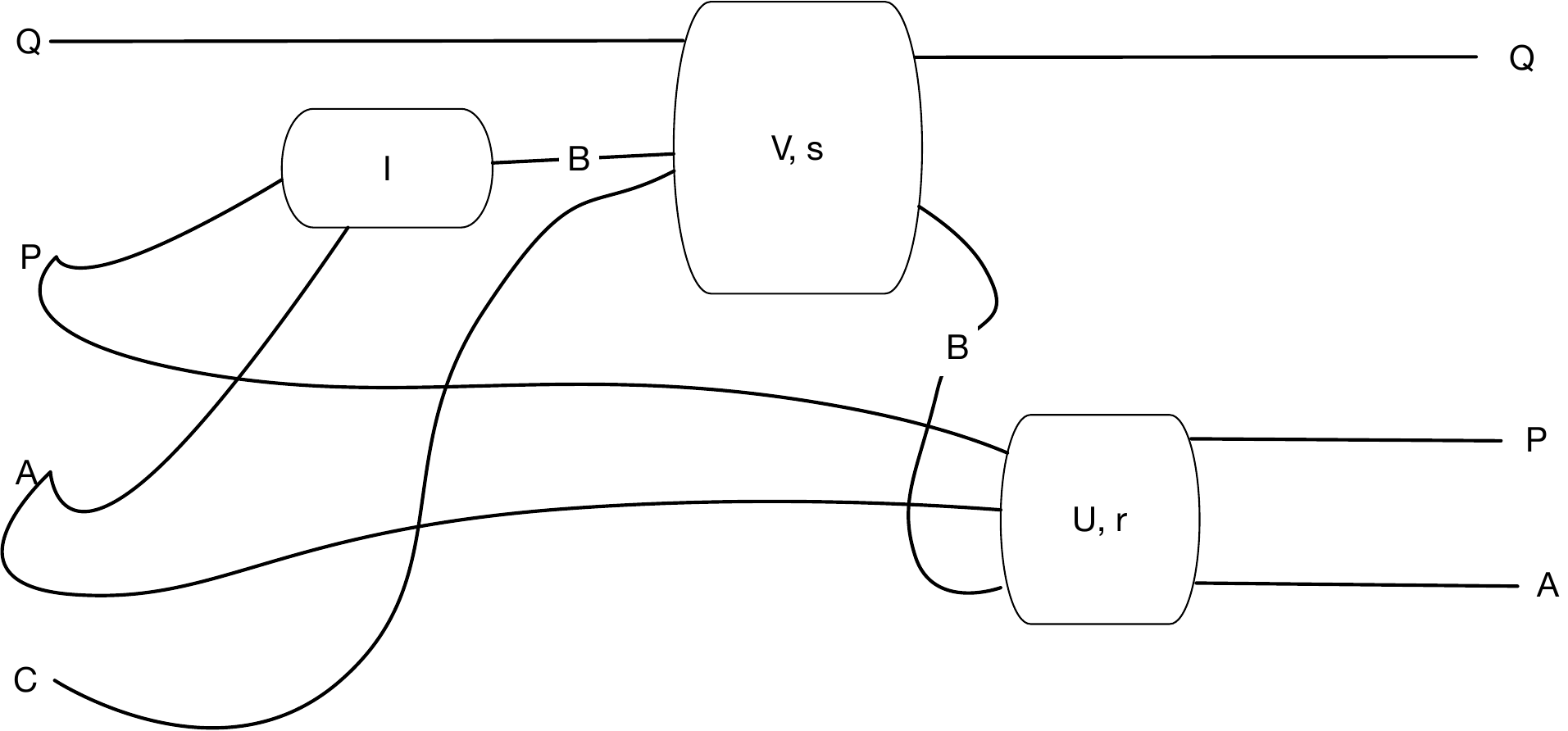}
    \includegraphics[scale=0.4]{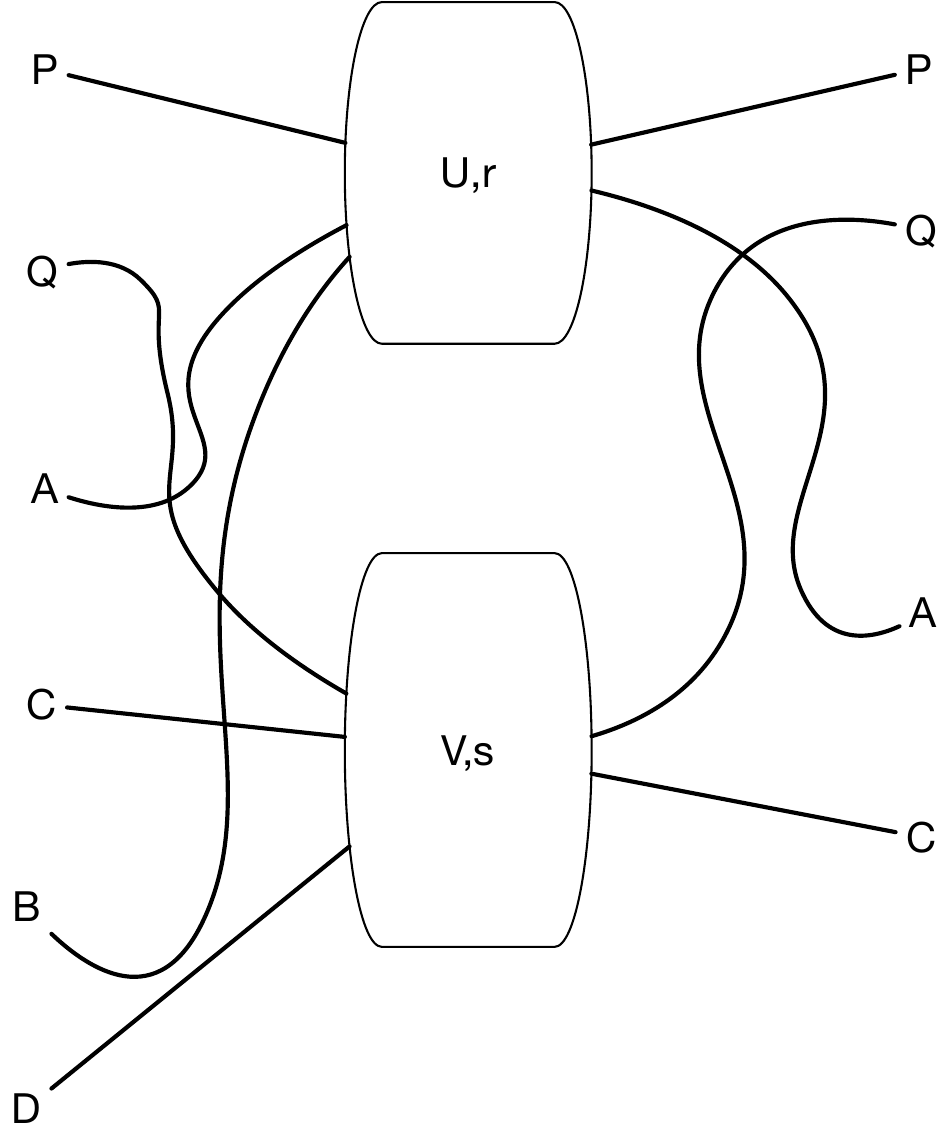}
    \caption{Sequential and parallel composition of two learners in the symmetric monoidal category {\tt Learn}.}
    \label{seqlearn}
\end{figure}

\cite{DBLP:conf/lics/FongST19} show that each learner can be combined in sequentially and in parallel (see Figure~\ref{seqlearn}), both formally using the operations of composition $\circ$ and tensor product $\otimes$ in the symmetric monoidal category {\tt Learn}, and equivalently in terms of string diagrams. For clarity, let us write out the compositional rule for a pair of learners

\[ A \xrightarrow[]{(P, I, U, r)} B \xrightarrow[]{(Q, J, V, s)} C \]

The composite learner $A \rightarrow C$ is defined as $(P \times Q, I \cdot J, U \cdot V, r \cdot s)$, where the composite implementation function is

\[ (I \cdot J)(p, q, a) \coloneqq J(q, I(p, a)) \]

and the composite update function is 

\[ U \cdot V(p, q, a, c) \coloneqq \left( U(p, a, s(q, I(p, a), c) \right), V(q, I(p, a), c) \]

and the composite request function is 

\[ (r \cdot s)(p, q, a, c) \coloneqq r(p, a, s(q, I(p, a), c)). \]

\subsection{Backpropagation as a Functor}

We can define the backpropagation procedure as a functor that maps from the category {\tt Para} to the category {\tt Learn}. Functors can be viewed as a generalization of the notion of morphisms across algebraic structures, such as groups, vector spaces, and graphs. Functors do more than functions: they not only map objects to objects, but like graph homomorphisms, they need to also map each morphism in the domain category to a corresponding morphism in the co-domain category. Functors come in two varieties, as defined below. The Yoneda Lemma, in its most basic form, asserts that any set-valued functor $F: {\cal C} \rightarrow {\bf Sets}$ can be universally represented by a {\em representable functor} ${\cal C}(-, x): {\cal C}^{op} \rightarrow {\bf Sets}$. 

 \begin{definition} 
A {\bf {covariant functor}} $F: {\cal C} \rightarrow {\cal D}$ from category ${\cal C}$ to category ${\cal D}$, and defined as \mbox{the following: }

\begin{itemize} 
    \item An object $F X$ (also written as $F(X)$) of the category ${\cal D}$ for each object $X$ in category ${\cal C}$.
    \item An  arrow  $F(f): F X \rightarrow F Y$ in category ${\cal D}$ for every arrow  $f: X \rightarrow Y$ in category ${\cal C}$. 
   \item The preservation of identity and composition: $F \ id_X = id_{F X}$ and $(F f) (F g) = F(g \circ f)$ for any composable arrows $f: X \rightarrow Y, g: Y \rightarrow Z$. 
\end{itemize}
\end{definition} 

\begin{definition} 
A {\bf {contravariant functor}} $F: {\cal C} \rightarrow {\cal D}$ from category ${\cal C}$ to category ${\cal D}$ is defined exactly like the covariant functor, except all the arrows are reversed. 
\end{definition} 

The {\em functoriality} axioms dictate how functors have to be behave: 

\begin{itemize} 

\item For any composable pair $f, g$ in category $C$, $Fg \cdot Ff = F(g \cdot f) $.

\item For each object $c$ in $C$, $F (1_c) = 1_{Fc}$.

\end{itemize} 

Note that the category {\tt Learn} is ambivalent as to what particular learning method is used. To define a particular learning method, such as backpropagation, we can define a category whose objects define the parameters of the particular learning method, and then another category for the learning method itself.  We can define a functor from the category {\tt NNet} to the category {\tt Learn} that factors through the category {\tt Param}. Later in the next section, we show how to generlize this construction to simplicial sets. 

\[\begin{tikzcd}
	NNet &&&& Learn \\
	&& Param
	\arrow["F", from=1-1, to=2-3]
	\arrow["{L_{\epsilon, e}}", from=2-3, to=1-5]
	\arrow[from=1-1, to=1-5]
\end{tikzcd}\]

\begin{definition} \cite{DBLP:conf/lics/FongST19}
    The category {\tt Param} defines a strict symmetric monoidal category whose objects are Euclidean spaces, and whose morphisms $f: \mathbb{R}^n \rightarrow \mathbb{R}^m$ are equivalence classes of differential parameterized functions. In particular, $(P, I)$ defines a Euclidean space $P$ and $I: P \times A \rightarrow B$ defines a differentiable parameterized function $A \rightarrow B$. Two such pairs $(P, I), (P', I')$ are considered equivalent if there is a differentiable bijection $f: P \rightarrow P'$ such that for all $p \in P$, and $a \in A$, we have that $I'(f'(p),a) = I(p,a)$. The composition of $(P, I): \mathbb{R}^n \rightarrow \mathbb{R}^m$ and $(Q, J): \mathbb{R}^n \rightarrow \mathbb{R}^m$ is given as 

    \[ (P \times Q, I \cdot J) \ \ \ \mbox{where} \ \ \ (I \cdot J)(p, q, a) = J(q, I(p, a)) \]

    The monoidal product of objects $\mathbb{R}^n$ and $\mathbb{R}^m$ is the object $\mathbb{R}^{n+m}$, whereas the monoidal product of morphisms $(P, I): \mathbb{R}^m \rightarrow \mathbb{R}^m$ and $(Q, J): \mathbb{R}^l \rightarrow \mathbb{R}^k$ is given as $(P \times Q, I \parallel J)$, where 

    \[ (I \parallel J) (p, q, a, c)  = \left( I(p, a), J(q, c) \right) \]

    Symmetric monoidal categories can also be braided. In this case, the braiding $\mathbb{R}^m \parallel \mathbb{R}^m \rightarrow \mathbb{R}^m \parallel \mathbb{R}^n $ is given as $(\mathbb{R}^0, \sigma)$ where $\sigma(a, b) = (b, a)$. 
\end{definition}

The backpropagation algorithm can itself be defined as a functor over symmetric monoidal categories

\[ L_{\epsilon, e}: {\tt Param} \rightarrow {\tt Learn}\]

where $\epsilon > 0$ is a real number defining the learning rate for backpropagation, and $e(x,y): \mathbb{R} \times \mathbb{R} \rightarrow \mathbb{R}$ is a differentiable error function such that $\frac{\partial e}{\partial x}(x_0, -)$ is invertible for each $x_0 \in \mathbb{R}$. This functor essentially defines an update procedure for each parameter in a compositional learner. In other words,  the functor $L_{\epsilon, e}$ defined by backpropagation sends each parameterized function $I: P \times A \rightarrow B$ to the learner $(P, I, U_I,r_I)$

\[ U_I(p, a, b) \coloneqq p - \epsilon \nabla_p E_I(p, a, b) \]

\[ r_I(p, a, b) \coloneqq f_a(\nabla_a E_I(p, a, b)) \]

where $E_I(p, a, b) \coloneqq \sum_j e(I_j(p, a), b_j)$ and $f_a$ is a component-wise application of the inverse to $\frac{\partial e}{\partial x}(a_i, -)$ for each $i$. 

Note that we can easily define functors that define other ways of doing parameterized updates, such as a stochastic approximation method \cite{rm} that updates each parameter using only the (noisy) value of the function at the current value of the parameter, and uses a gradual decay of the learning parameters to ``simulate" the process of taking gradients. These sort of stochastic approximation updates are now called ``zeroth-order" optimization in the deep learning literature.

\section{Backpropagation as an Endofunctor: Generative AI using Universal Coalgebras}

\label{endofunctor} 

Our categorical framework for generative AI differs in crucial ways from the analysis in \cite{DBLP:conf/lics/FongST19} that defined backpropagation as a functor, but not an endofunctor. In their framework, which we reviewed in the previous section, backpropagation was modeled as a functor from the category {\tt Param} to the cateory {\tt Learn}, but that masks the simple fact that the goal of learning is to produce a new set of parameters (i.e., construct a new object in {\tt Param}). Once you complete that loop, backpropagation becomes an endofunctor. This property allows bringing in the rich framework of universal coalgebras \cite{jacobs:book,rutten2000universal} to analyze a whole family of endofunctors for generative AI. 

As the ultimate goal of backpropagation at each step is to produce a new parameter, i.e. a new object in {\tt Param}, we argue that our endofunctor characterization provides a rich source of insight into the analysis of generative AI methods. Accordingly, we review below the theory of universal coalgebras, and then show more formally how to model backpropagation and other similiar generative AI methods as coalgebras. 

\subsection{Non-Well-Founded Sets and Universal Coalgebras} 

\begin{figure}[h]
\centering
\includegraphics[scale=.35]{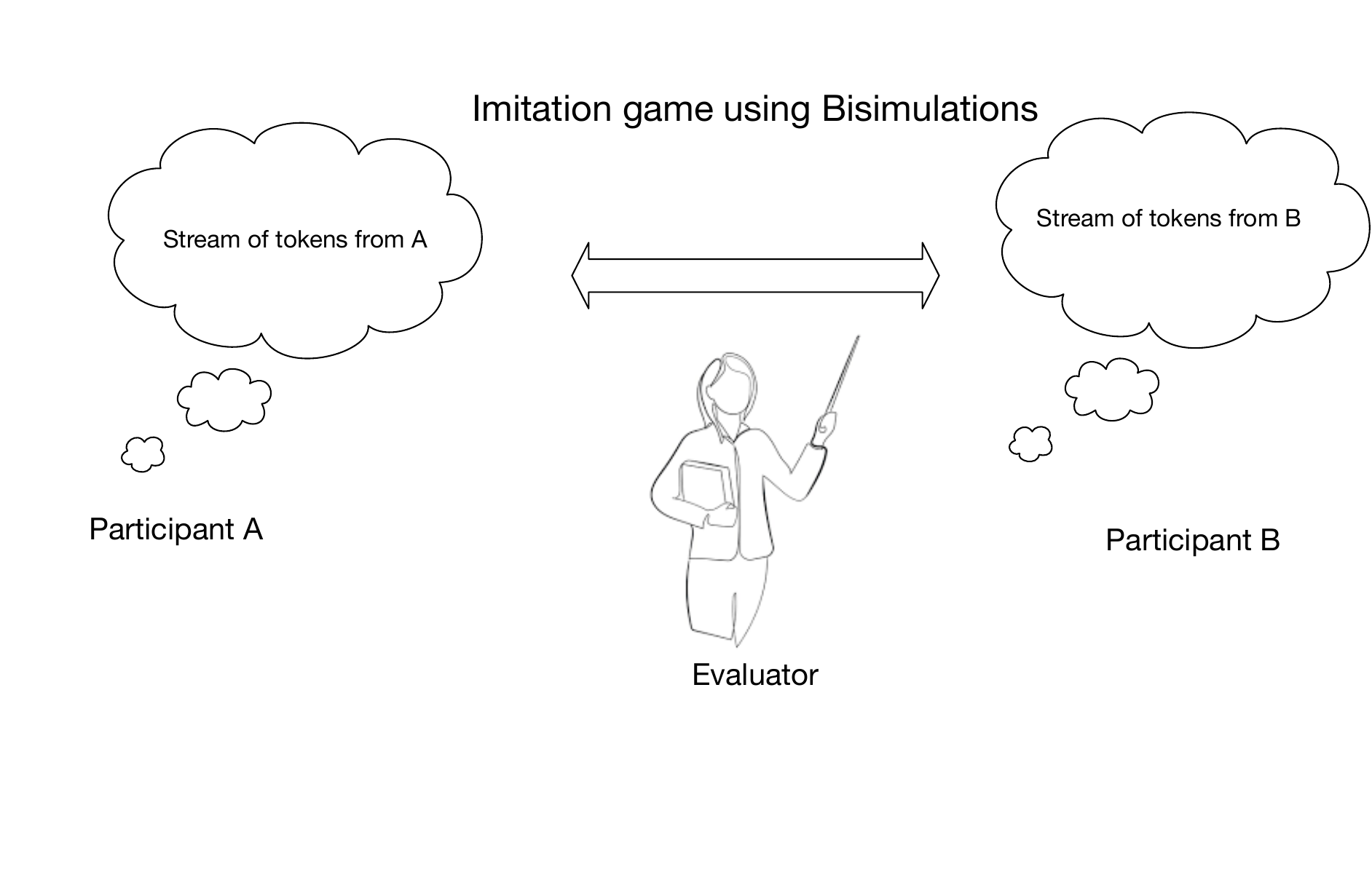}
\caption{Generative models define an infinite stream of tokens. Solving an `imitation game" \cite{turing} that involves comparing two infinite data streams involves the process of deciding if two  non-well-founded sets are categorically {\em bisimulations} of each other\cite{Aczel1988-ACZNS,rutten2000universal}.} 
\label{imitgamebisim} 
\end{figure}

To begin with, we present an elegant formalism for defining generative AI models as universal coalgebras \cite{rutten2000universal}, and non-well-founded sets \cite{Aczel1988-ACZNS}. Figure~\ref{imitgamebisim} illustrates the main idea. We define the two participants in an imitation game as universal coalgebras (or non-well-founded sets) and ask if there is a bisimulation relationship between them. This characterization covers a wide range of probabilistic models, including Markov chains and Markov decision processes \cite{DBLP:books/lib/SuttonB98}, and automata-theoretic models, as well as generative AI models \cite{DBLP:conf/iclr/GuJTRR23}.

Generative AI has become  popular recently due to the successes of neural and structured-state space sequence models  \cite{DBLP:conf/iclr/GuGR22,DBLP:conf/nips/VaswaniSPUJGKP17} and text-to-image diffusion models \cite{DBLP:conf/nips/SongE19}. The underlying paradigm of building generative models has a long history in computer science and AI, and it is useful to begin with the simplest models that have been studied for several decades, such as deterministic finite state machines, Markov chains, and context-free grammars.  We use category theory to build generative AI models and analyze them, which is one of the unique and novel aspects of this paper. To explain briefly, we represent a generative model in terms of {\em universal coalgebras} \cite{rutten2000universal} generated by an {\em endofunctor} $F$ acting on a category $C$. Coalgebras provide an elegant way to model dynamical systems, and capture the notion of state \cite{jacobs:book} in ways that provide new insight into the design of AI and ML systems. Perhaps the simplest and in some ways, the most general, type of generative AI model that is representable as a coalgebra is the {\em powerset functor} 

\[ F: S \Rightarrow {\cal P}(S)\]

where $S$ is any set (finite or not), and ${\cal P}(S)$ is the set of all subsets of $S$, that is: 

\[ {\cal P}(S) = \{ A | A \subseteq S \} \]

Notice in the specification of the powerset functor coalgebra, the same term $S$ appears on both sides of the equation. That is a hallmark of coalgebras, and it is what distinguishes coalgebras from algebras. Coalgebras generate search spaces, whereas algebras compact search spaces and summarize them.  This admittedly simple structure nonetheless is extremely versatile and enables modeling a remarkably rich and diverse set of generative AI models, including the ones listed in Figure~\ref{genaimodel}. To explain briefly, we can model a context-free grammar as a mapping from a set $S$ that includes all the vertices in the context-free grammar graph shown in Figure~\ref{genaimodel} to the powerset of the set $S$. More specifically, if $S = N \cup T$ is defined as the non-terminals $N$ as well as the terminal symbols (the actual words) $T$, any context-free grammar rule can be represented in terms of a power set functor. We will explain how this approach can be refined later in this Section, and in much more detail in later sections of the paper. To motivate further why category theory provides an elegant way to model generative AI systems, we look at some actual examples of generative AI systems to see why they can be modeled as functors. 

\begin{figure}[h]
\centering
\includegraphics[scale=.3]{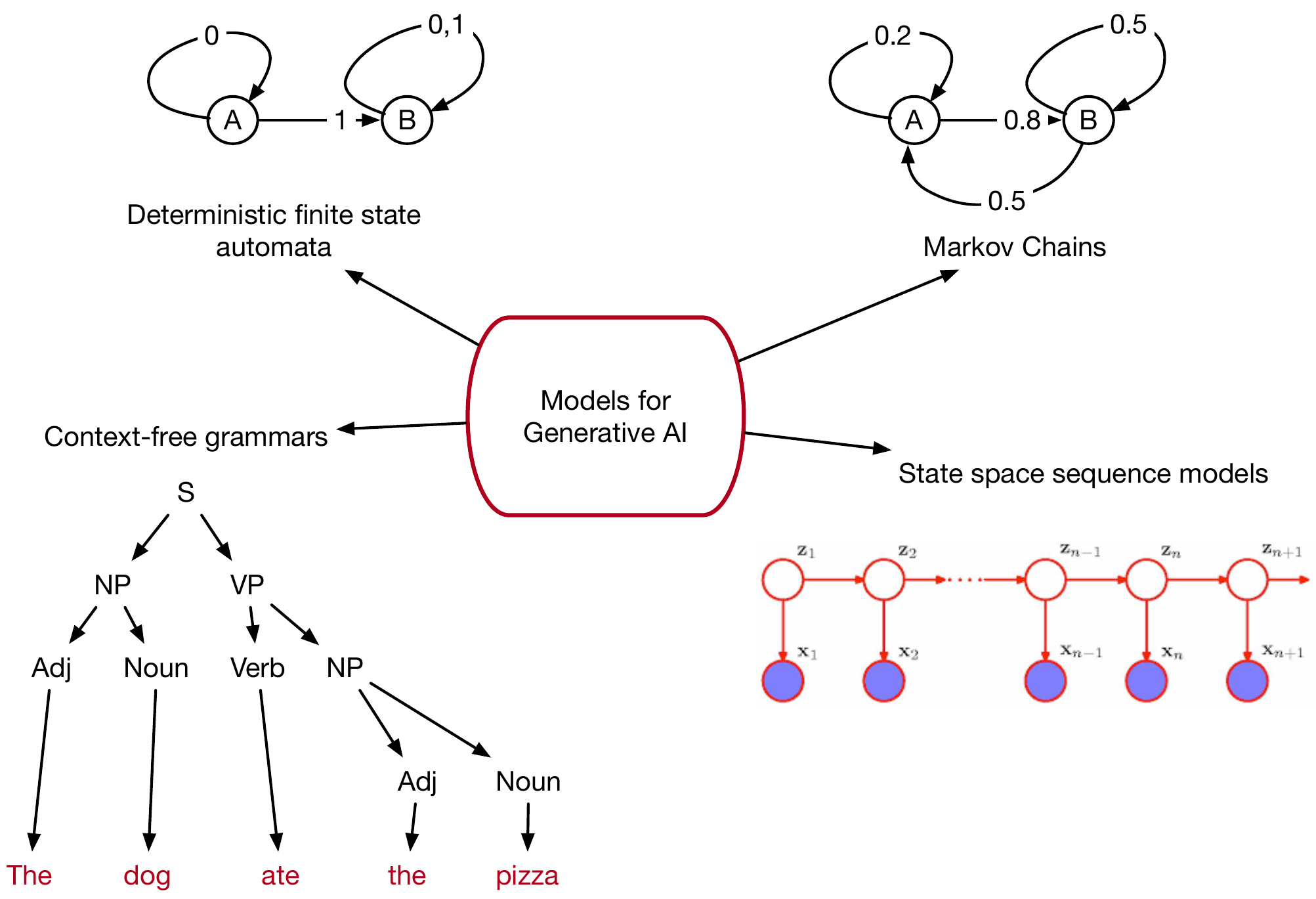}
\caption{In this paper, generative AI models, from the earliest models studied in computer science such as deterministic finite state automata and context-free grammars, to models in statistics and information theory like Markov chains, and lastly, sequence models can be represented as universal coalgebras. }
\label{genaimodel} 
\end{figure}

To compare say two large language models, we need to compare two potentially infinite data streams of tokens (e.g., words, or in general, other forms of communication represented digitally by bits). Many problems in AI and ML involve reasoning about circular phenomena. These include reasoning about {\em common knowledge} \cite{barwise,fagin} such as social conventions, dealing with infinite data structures such as lists or trees in computer science, and causal inference in systems with feedback where part of the input comes from the output of the system. In all these situations, there is an intrinsic problem of having to deal with infinite sets that are recursive and violate a standard axiom called well-foundedness in set theory. First, we explain some of the motivations for including non-well-founded sets in AI and ML, and then proceed to define the standard ZFC axiomatization of set theory and how to modify it to allow circular sets. We build on the pioneering work of Peter Aczel on the anti-foundation-axiom in modeling non-well-founded sets \cite{Aczel1988-ACZNS}, which has elaborated previously in other books as well \cite{barwise,jacobs:book}, although we believe this paper is perhaps one of the first to focus on the application of non-well-founded sets and universal coalgebras to problems in AI and ML at a broad level. 

\begin{figure}[t]
\centering
\includegraphics[scale=.5]{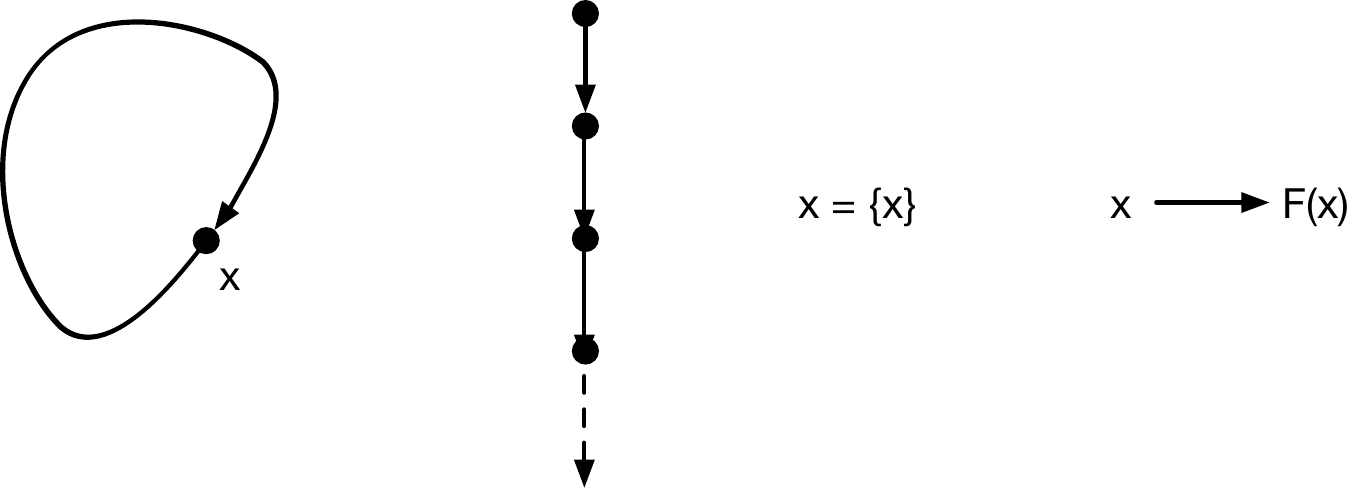}
\caption{Three representations of infinite data streams: non-well-founded set $x = \{ x \}$: accessible pointed graphs (AGPs), non-well-founded sets specified by systems of equations, and universal coalgebras. We can view these as generative models of the recursive set $\{ \{ \{ \ldots \} \} \} $.}
\label{threereps} 
\end{figure}

Figure~\ref{threereps} illustrates three ways to represent an infinite object, as a directed graph, a (non-well-founded) set or as a system of equations.  We begin with perhaps the simplest approach introduced by Peter Aczel called accessible pointed graphs (APGs) (see Figure~\ref{threereps}), but also include the category-theoretic approach of using {\em universal coalgebras} \cite{rutten2000universal}, as well as systems of equations \cite{barwise}. 

We now turn to describing coalgebras, a much less familiar construct that will play a central role in the proposed ML framework of coinductive inference. Coalgebras capture hidden state, and enable modeling infinite data streams. Recall that in the previous Section, we explored non-well-founded sets, such as the set $\Omega = \{ \Omega \}$, which gives rise to a circularly defined object. As another example, consider the infinite data stream comprised of a sequence of objects, indexed by the natural numbers: 

\[ X = (X_0, X_1, \ldots, X_n, \ldots ) \]

We can define this infinite data stream as a coalgebra, comprised of an accessor function {\bf head} that returns the head of the list, and a destructor function that gives the {\bf tail}  of the list, as we will show in detail below. 

To take another example, consider a deterministic finite state machine model defined as the tuple $M = (X, A, \delta)$, where $X$ is the set of possible states that the machine might be in, $A$ is a set of input symbols that cause the machine to transition from one state to another, and $\delta: X \times A \rightarrow X$ specifies the transition function. To give a coalgebraic definition of a finite state machine, we note that we can define a functor $F: X \rightarrow {\cal P}(A \times X)$ that maps any given state $x \in X$ to the subset of possible future states $y$ that the machine might transition to for any given input symbol $a \in A$. 

We can now formally define $F$-coalgebras analogous to the definition of $F$-algebras given above. 

\begin{definition}
Let $F: {\cal C} \rightarrow {\cal C}$ be an endofunctor on the category ${\cal C}$. An {\bf $F$-coalgebra} is defined as a pair $(A, \alpha)$ comprised of an object $A$ and an arrow $\alpha: A \rightarrow F(A)$. 
\end{definition}

The fundamental difference between an algebra and a coalgebra is that the structure map is reversed! This might seem to be a minor distinction, but it makes a tremendous difference in the power of coalgebras to model state and capture dynamical systems. Let us use this definition to capture infinite data streams, as follows. 

\[ {\bf Str}: {\bf Set} \rightarrow {\bf Set}, \ \ \ \ \ \ {\bf Str}(X) = \mathbb{N} \times X\]

Here, {\bf Str} is defined as a functor on the category {\bf Set}, which generates a sequence of elements. Let $N^\omega$ denote the set of all infinite data streams comprised of natural numbers:

\[ N^\omega = \{ \sigma | \sigma: \mathbb{N} \rightarrow \mathbb{N} \} \]

To define the accessor function {\bf head} and destructor function {\bf tail} alluded to above, we proceed as follows: 

\begin{eqnarray}
{\bf head}&:& \mathbb{N}^\omega \rightarrow \mathbb{N} \ \ \  \ \ \ \ {\bf tail}: \mathbb{N}^\omega \rightarrow\mathbb{N}^\omega \\
{\bf head}(\sigma) &=& \sigma(0) \ \ \ \ \ \ \ {\bf tail}(\sigma) = (\sigma(1), \sigma(2), \ldots )
\end{eqnarray}

Another standard example that is often used to illustrate coalgebras, and provides a foundation for many AI and  ML applications, is that of a {\em labelled transition system}. 

\begin{definition}
    A {\bf labelled transition system} (LTS) $(S, \rightarrow_S, A)$ is defined by a set $S$ of states, a transition relation $\rightarrow_S \subseteq S \times A \times S$, and a set $A$ of labels (or equivalently, ``inputs" or ``actions"). We can define the transition from state $s$ to $s'$ under input $a$ by the transition diagram $s \xrightarrow[]{a} s'$, which is equivalent to writing $\langle s, a,  s' \rangle \in \rightarrow_S$. The ${\cal F}$-coalgebra for an LTS is defined by the functor 

    \[ {\cal F}(X) = {\cal P}(A \times X) = \{V | V \subseteq A \times X\} \]
\end{definition}

Just as before, we can also define a category of $F$-coalgebras over any category ${\cal C}$, where each object is a coalgebra, and the morphism between two coalgebras is defined as follows, where $f: A \rightarrow B$ is any morphism in the category ${\cal C}$. 

\begin{definition}
Let $F: {\cal C} \rightarrow {\cal C}$ be an endofunctor. A {\em homomorphism} of $F$-coalgebras $(A, \alpha)$ and $(B, \beta)$ is an arrow $f: A \rightarrow B$ in the category ${\cal C}$ such that the following diagram commutes:

\begin{center}
\begin{tikzcd}
  A \arrow[r, "f"] \arrow[d, "\alpha"]
    & B \arrow[d, "\beta" ] \\
  F(A) \arrow[r,  "F(f)"]
& F(B)
\end{tikzcd}
\end{center}
\end{definition}

For example, consider two labelled transition systems $(S, A, \rightarrow_S)$ and $(T, A, \rightarrow_T)$ over the same input set $A$, which are defined by the coalgebras $(S, \alpha_S)$ and $(T, \alpha_T)$, respectively. An $F$-homomorphism $f: (S, \alpha_S) \rightarrow (T, \alpha_T)$ is a function $f: S \rightarrow T$ such that $F(f) \circ \alpha_S  = \alpha_T \circ f$. Intuitively, the meaning of a homomorphism between two labeled transition systems means that: 

\begin{itemize}
    \item For all $s' \in S$, for any transition $s \xrightarrow[]{a}_S s'$ in the first system $(S, \alpha_S)$, there must be a corresponding transition in the second system $f(s) \xrightarrow[]{a}_T f(s;)$ in the second system. 

    \item Conversely, for all $t \in T$, for any transition $t \xrightarrow[]{a}_T t'$ in the second system, there exists two states $s, s' \in S$ such that $f(s) = t, f(t) = t'$ such that $s \xrightarrow[]{a}_S s'$ in the first system. 
\end{itemize}

If we have an $F$-homomorphism $f: S \rightarrow T$ with an inverse $f^{-1}: T \rightarrow S$ that is also a $F$-homomorphism, then the two systems $S \simeq T$ are isomorphic. If the mapping $f$ is {\em injective}, we have a  {\em monomorphism}. Finally, if the mapping $f$ is a surjection, we have an {\em epimorphism}. 

The analog of congruence in universal algebras is {\em bisimulation} in universal coalgebras. Intuitively, bisimulation allows us to construct a more ``abstract" representation of a dynamical system that is still faithful to the original system. We will explore many applications of the concept of bisimulation to AI and ML systems in this paper. We introduce the concept in its general setting first, and then in the next section, we will delve into concrete examples of bisimulations. 

\begin{definition}
Let $(S, \alpha_S)$ and $(T, \alpha_T)$ be two systems specified as coalgebras acting on the same category ${\cal C}$. Formally, a $F$-{\bf bisimulation} for coalgebras defined on a set-valued functor $F: {\bf Set} \rightarrow {\bf Set}$ is a relation $R \subset S \times T$ of the Cartesian product of $S$ and $T$ is a mapping $\alpha_R: R \rightarrow F(R)$ such that the projections of $R$ to $S$ and $T$ form valid $F$-homomorphisms.

\begin{center}
\begin{tikzcd}
  R \arrow[r, "\pi_1"] \arrow[d, "\alpha_R"]
    & S \arrow[d, "\alpha_S" ] \\
  F(R) \arrow[r,  "F(\pi_1)"]
& F(S)
\end{tikzcd}
\end{center}

\begin{center}
\begin{tikzcd}
  R \arrow[r, "\pi_2"] \arrow[d, "\alpha_R"]
    & T \arrow[d, "\alpha_T" ] \\
  F(R) \arrow[r,  "F(\pi_2)"]
& F(T)
\end{tikzcd}
\end{center}

Here, $\pi_1$ and $\pi_2$ are projections of the relation $R$ onto $S$ and $T$, respectively. Note the relationships in the two commutative diagrams should hold simultaneously, so that we get 

\begin{eqnarray*}
    F(\pi_1) \circ \alpha_R &=& \alpha_S \circ \pi_1 \\
    F(\pi_2) \circ \alpha_R &=& \alpha_T \circ \pi_2 
\end{eqnarray*}

Intuitively, these properties imply that we can ``run" the joint system $R$ for one step, and then project onto the component systems, which gives us the same effect as if we first project the joint system onto each component system, and then run the component systems. More concretely, for two labeled transition systems that were considered above as an example of an $F$-homomorphism, an $F$-bisimulation between $(S, \alpha_S)$ and $(T, \alpha_T)$ means that  there exists a relation $R \subset S \times T$ that satisfies for all $\langle s, t \rangle \in R$

\begin{itemize}
    \item For all $s' \in S$, for any transition $s \xrightarrow[]{a}_S s'$ in the first system $(S, \alpha_S)$, there must be a corresponding transition in the second system $f(s) \xrightarrow[]{a}_T f(s;)$ in the second system, so that $\langle s', t' \rangle \in R$

    \item Conversely, for all $t \in T$, for any transition $t \xrightarrow[]{a}_T t'$ in the second system, there exists two states $s, s' \in S$ such that $f(s) = t, f(t) = t'$ such that $s \xrightarrow[]{a}_S s'$ in the first system, and $\langle s', t' \rangle \in R$.
\end{itemize}

\end{definition} 

A simple example of a bisimulation of two coalgebras is shown in Figure~\ref{bisim}. 

\begin{figure}[t]
\centering
\includegraphics[scale=.45]{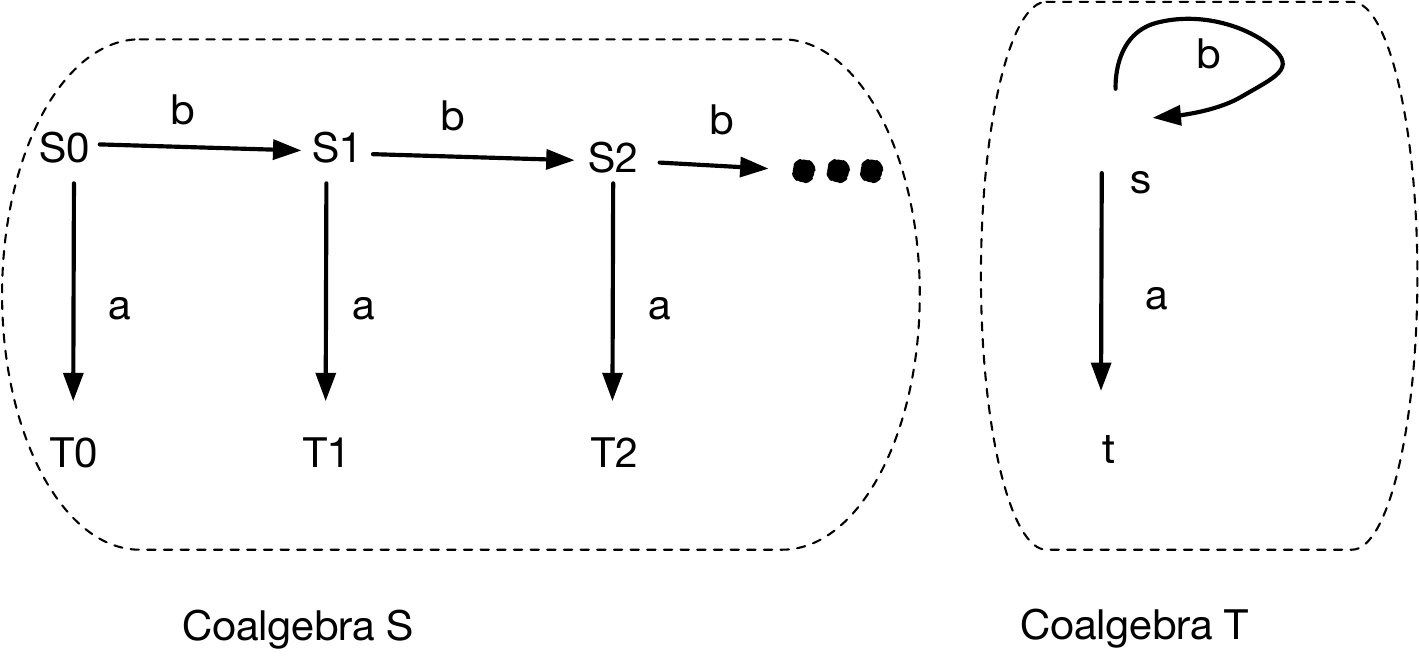}
\caption{A bisimulation among two coalgebras.}
\label{bisim} 
\end{figure}

There are a number of basic properties about bisimulations, which we will not prove, but are useful to summarize here: 

\begin{itemize}
    \item If $(R, \alpha_R)$ is a bisimulation between systems $S$ and $T$, the inverse $R^{-1}$ of $R$ is a bisimulation between systems $T$ and $S$. 

    \item Two homomorphisms $f: T \rightarrow S$ and $g:T \rightarrow U$ with a common domain $T$ define a {\em span}. The {\em image} of the span $\langle f, g \rangle(T) = \{ \langle f(t), g(t) \rangle | t \in T \}$ of $f$ and $g$ is also a bisimulation between $S$ and $U$. 

    \item The composition $R \circ Q$ of two bisimulations $R \subseteq S \times T$ and $Q \subseteq T \times U$ is a bisimulation between $S$ and $U$. 

    \item The union $\cup_k R_k$ of a family of bisimulations between $S$ and $T$ is also a bisimulation. 

    \item The set of all bisimulations between systems $S$ and $T$ is a complete lattice, with least upper bounds and greatest lower bounds given by: 

    \[ \bigvee_k R_k = \bigcup_k R_k \]

    \[ \bigwedge_K R_k = \bigcup \{ R | R \ \mbox{is a bisimulation between} S \ \mbox{and} \ T \ \mbox{and} \ R \subseteq \cap_k R_k \} \]

    \item The kernel $K(f) = \{ \langle s, s' \rangle | f(s) = f(s') \}$ of a homomorphism $f: S \rightarrow T$ is a bisimulation equivalence.

    \end{itemize}

  \subsection{Backpropagation as a Coalgebra}

 Finally, we return to the original goal of this section, which is to argue that any generative AI machine learning method can be usefully modeled not just as a functor, but rather as an endofunctor that maps an object in a category {\tt Param} of parameters into a new object as a result of doing a machine learning step, such as a gradient update. We can now formally define backpropagation as a coalgebra over the caetgory {\tt Param} as follows.

 Recall that the category {\tt Param} defines a strict symmetric monoidal category whose objects are Euclidean spaces, and whose morphisms $f: \mathbb{R}^n \rightarrow \mathbb{R}^m$ are equivalence classes of differential parameterized functions. To see why the backpropagation algorithm can be defined as an endofunctor over the symmetric monoidal category {\tt Param}, recall from the previous section that backpropagation was viewed as a functor from the cateory {\tt Param} to the cateogry {\tt Learn}. 

\[ L_{\epsilon, e}: {\tt Param} \rightarrow {\tt Learn}\]

where $\epsilon > 0$ is a real number defining the learning rate for backpropagation, and $e(x,y): \mathbb{R} \times \mathbb{R} \rightarrow \mathbb{R}$ is a differentiable error function such that $\frac{\partial e}{\partial x}(x_0, -)$ is invertible for each $x_0 \in \mathbb{R}$. This functor essentially defines an update procedure for each parameter in a compositional learner. In other words,  the functor $L_{\epsilon, e}$ defined by backpropagation sends each parameterized function $I: P \times A \rightarrow B$ to the learner $(P, I, U_I,r_I)$

\[ U_I(p, a, b) \coloneqq p - \epsilon \nabla_p E_I(p, a, b) \]

\[ r_I(p, a, b) \coloneqq f_a(\nabla_a E_I(p, a, b)) \]

where $E_I(p, a, b) \coloneqq \sum_j e(I_j(p, a), b_j)$ and $f_a$ is a component-wise application of the inverse to $\frac{\partial e}{\partial x}(a_i, -)$ for each $i$. 

But a simpler and we argue more elegant characterization of backpropagation is to view it as a coalgebra or dynamical system defined by an endofunctor on {\tt Param}. Here, we view the inputs $A$ and outputs $B$ as the input ``symbols" and output produced by a dynamical system. The actual process of updating the parameters need not be defined as ``gradient descent", but it can involve any other functor (as we saw earlier, it could involve a stochastic approximation method \cite{borkar}). Our revised definition of backpropagation as an endofunctor follows. Note that this definition is generic, and applies to virtually any approach to building foundation models that updates each object to a new object in the category {\tt Param} as a result of processing a data instance. 

  \begin{definition}
{\bf Backpropagation} defines an $F_B$-coalgebra over the symmetric monoidal category {\tt Param}, specified by an endofunctor $X \rightarrow F_B(X)$ defined as 

\[ F_B(X) = A \times B \times X\]
  \end{definition}

Note that in this definition, the endofunctor $F_B$ takes an object $X$ of {\tt Param}, which is a set of network weights of a generative AI model, and produces a new set of weights, where $A$ is the ``input" symbol of the dynamical system and $B$ is the output symbol. 

  \subsection{Zeroth-Order Deep Learning using Stochastic Approximation}

To illustrate how the broader coalgebraic definition of backpropagation is more useful than the previous definition in \cite{DBLP:conf/lics/FongST19}, we describe a class of generative AI methods based on adapting stochastic approximation \cite{rm} to deep learning, which are popularly referred to zeroth-order optimization \cite{liu2009large} (see Figure~\ref{zoodl}). A vast range of stochastic approximation methods have been explored in the literature (e.g., see \cite{borkar,kushner2003stochastic}). For example, in {\em random directions} stochastic approximation, each parameter is adjusted in a random direction by sampling from distribution, such as a multivariate normal, or a uniform distribution. Any of these zeroth-order stochastic approximation algorithms can itself be defined as a functor over symmetric monoidal categories

\[ L^{0}_{\epsilon}: {\tt Param} \rightarrow {\tt Learn}\]

where $\epsilon > 0$ is a real number defining a learning rate parameter that is gradually decayed. Notice now that the error of the approximation with respect to the target plays no role in the update process itself.  backpropagation, and $e(x,y): \mathbb{R} \times \mathbb{R} \rightarrow \mathbb{R}$ is a differentiable error function such that $\frac{\partial e}{\partial x}(x_0, -)$ is invertible for each $x_0 \in \mathbb{R}$. The functor $L^0_{\epsilon}$ defined by zeroth-order optimization sends each parameterized function $I: P \times A \rightarrow B$ to the learner $(P, I, U^0_I,r^0_I)$

\[ U^0_I(p, a, b) \coloneqq p - \epsilon I(p, a, b) \]

Here, the $1$-point gradient estimate is approximated by the (noisy) sampled value, averaged over multiple steps using a decaying learning rate as required by the convergence theorems of stochastic approximation \cite{rm,kushner2003stochastic}.  The advantages of zeroth-order stochastic approximation methods are that it avoids computing gradients over a very large number of parameters (which for state of the art generative AI models can be in the billions or trillions of parameters), and it potentially helps avoid local minima by stochastically moving around in a very high-dimensional space. The disadvantage is that it can be significantly slower than gradient-based methods for well-behaved (convex) functions.  

\begin{figure}
    \centering
    \includegraphics[scale=0.6]{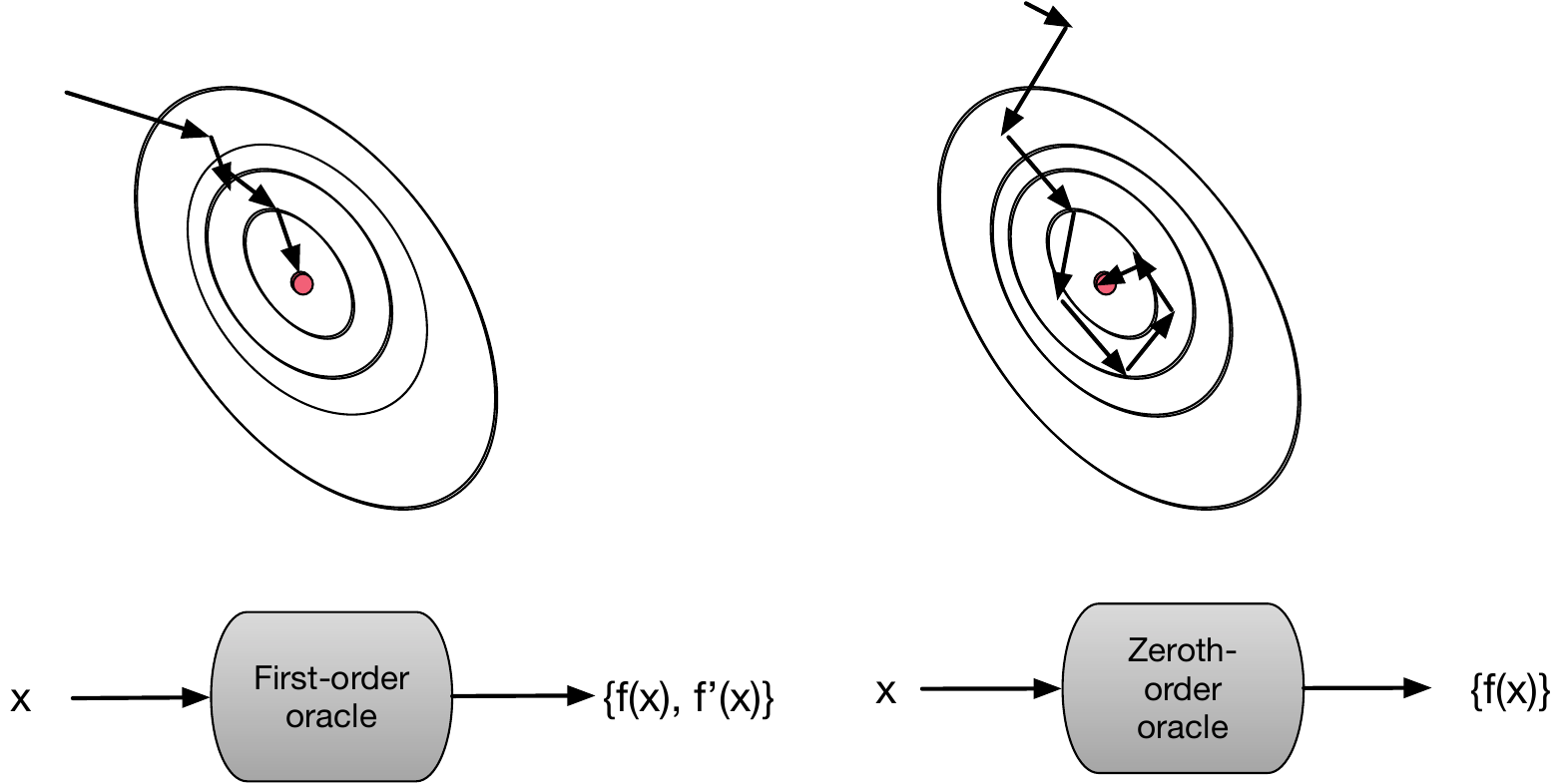}
    \caption{Zeroth-order optimization methods for generative AI are based on stochastic approximation, and average noisy values of the function to approximate gradient steps. Such methods define probabilistic coalgebras \cite{SOKOLOVA20115095}.}
    \label{zoodl}
\end{figure}

We can easily extend our previous definition of backpropagation as a coalgebra to capture zeroth-order optimization methods which act like stochastic dynamical systems,   where there is a distribution of possible ``next" states that is produced as a result of doing stochastic approximation updates. 

  \begin{definition}
{\bf Stochastic Backpropagation} defines an $F_{\mbox{SGD}}$-coalgebra over the symmetric monoidal category {\tt Param}, specified by an endofunctor $X \rightarrow F_{\mbox{SGD}}B(X)$ defined as 
\[ F_{\mbox{SGD}}(X) = A \times B \times {\cal D}(X)\]
  \end{definition}

  where $F_{\mbox{SGD}}$ defines the variant of backpropagation defined by stochastic gradient descent, and ${\cal D}$ is the distribution functor over $X$ that defines a distribution over possible objects $X$ in {\tt Param}. There is a vast literature on stochastic coalgebras that can be defined in terms of such distribution functors. \cite{SOKOLOVA20115095} contains an excellent review of this literature.

\subsection{Lambek's Theorem and Final Coalgebras: Analyzing the Convergence of Generative AI Algorithms}

Another advantage modeling backpropagation as a coalgebra defined by an endofunctor is that it provides an elegant way to analyze the problem of convergence of the algorithm to some (local) minimum solution. We explain the general principle of using final coalgebras as a generalization of (greatest) fixed points in this section. Later in Section~\ref{layer3}, when we introduce the metric Yoneda Lemma, we will show how to use the metric coinduction property to analyze convergence of backpropagation. 

Let us illustrate the concept of final coalgebras defined by a functor that represents a monotone function over the category defined by a preorder $(S, \leq)$, where $S$ is a set and $\leq$ is a relation that is reflexive and transitive. That is, $a \leq a, \forall a \in S$, and if $a \leq b$, and $b \leq c$, for $a, b, c \in S$, then $a \leq c$. Note that we can consider $(S, \leq)$ as a category, where the objects are defined as the elements of $S$ and if $a \leq b$, then there is a unique arrow $a \rightarrow b$. 

Let us define a functor $F$ on a preordered set $(S, \leq)$ as any monotone mapping $F: S \rightarrow S$,  so that if $a \leq b$, then $F(a) \leq F(b)$. Now, we can define an $F$-algebra as any {\em pre-fixed point} $x \in S$ such that $F(x) \leq x$. Similarly, we can define any {\em post-fixed point} to be any $x \in S$ such that $x \leq F(x)$. Finally, we can define the {\em final $F$-coalgebra} to be the {\em greatest post-fixed point} $x \leq F(x)$, and analogously, the {\em initial $F$-algebra} to be the least pre-fixed point of $F$. 

In this section, we give a detailed overview of the concept of {\em final coalgebras} in the category of coalgebras parameterized by some endofunctor $F$. This fundamental notion plays a central role in the application of universal coalgebras to model a diverse range of AI and ML systems. Final coalgebras generalize the concept of (greatest) fixed points in many areas of application in AI, including causal inference, game theory and network economics, optimization, and reinforcement learning among others. The final coalgebra, simply put, is just the final object in the category of coalgebras. From the universal property of final objects, it follows that for any other object in the category, there must be a unique morphism to the final object. This simple property has significant consequences in applications of the coalgebraic formalism to AI and ML, as we will see throughout this paper. 

An $F$-system $(P, \pi)$ is termed {\bf final} if for another $F$-system $(S, \alpha_S)$, there exists a unique homomorphism $f_S: (S, \alpha_S) \rightarrow (P, \pi)$. That is, $(P, \pi)$ is the terminal object in the category of coalgebras $Set_F$ defined by some set-valued endofunctor $F$. Since the terminal object in a category is unique up to isomorphism, any two final systems must be isomorphic. 

\begin{definition}
    An $F$-coalgebra $(A, \alpha)$ is a {\em fixed point} for $F$, written as $A \simeq F(A)$ if $\alpha$ is an isomorphism between $A$ and $F(A)$. That is, not only does there exist an arrow $A \rightarrow F(A)$ by virtue of the coalgebra $\alpha$, but there also exists its inverse $\alpha^{-1}: F(A) \rightarrow A$ such that 

    \[ \alpha \circ \alpha^{-1} = \mbox{{\bf id}}_{F(A)} \ \ \mbox{and} \ \  \alpha^{-1} \circ \alpha = \mbox{{\bf id}}_A \]
\end{definition}

The following lemma was shown by Lambek, and implies that the transition structure of a final coalgebra is an isomorphism. 

\begin{theorem}
    {\bf Lambek:} A final $F$-coalgebra is a fixed point of the endofunctor $F$. 
\end{theorem}

{\bf Proof:} The proof is worth including in this paper, as it provides a classic example of the power of diagram chasing. Let $(A, \alpha)$ be a final $F$-coalgebra. Since $(F(A), F(\alpha)$ is also an $F$-coalgebra, there exists a unique morphism $f: F(A) \rightarrow A$ such that the following diagram commutes: 

\begin{tikzcd}
  F(A) \arrow[r, "f"] \arrow[d, "F(\alpha)"]
    & A \arrow[d, "\alpha" ] \\
  F(F(A)) \arrow[r,  "F(f)"]
& F(A)
\end{tikzcd}

However, by the property of finality, the only arrow from $(A, \alpha)$ into itself is the identity. We know the following diagram also commutes, by virtue of the definition of coalgebra homomorphism:

\begin{tikzcd}
  A \arrow[r, "\alpha"] \arrow[d, "\alpha"]
    & F(A) \arrow[d, "\alpha" ] \\
  F(A) \arrow[r,  "F(\alpha)"]
& F(F(A))
\end{tikzcd}

Combining the above two diagrams, it clearly follows that $f \circ \alpha$ is the identity on object $A$, and it also follows that $F(\alpha) \circ F(f)$ is the identity on $F(A)$. Therefore, it follows that: 

\[ \alpha \circ f  = F(f) \circ F(\alpha) = F(f \circ \alpha) = F(\mbox{{\bf id}}_A) = \mbox{{\bf id}}_{F(A)} \qed \]

By reversing all the arrows in the above two commutative diagrams, we get the easy duality that the initial object in an $F$-algebra is also a fixed point. 

\begin{theorem}
    {\bf Dual to Lambek}: The initial $F$-algebra $(A, \alpha)$, where $\alpha: F(A) \rightarrow A$, in the category of $F$-algebras is a fixed point of $F$. 
\end{theorem}

The proof of the duality follows in the same way, based on the universal property that there is a unique morphism from the initial object in a category to any other object. 

Lambek's lemma  has many implications, one of which is that final coalgebras generalize the concept of a (greatest) fixed point, which can be applied to analyze the convergence of generative AI methods, such as backpropagation.   Generally speaking, in optimization, we are looking to find a solution $x \in X$ in some space $X$ that minimizes a smooth real-valued function $f: X \rightarrow \mathbb{R}$. Since $f$ is smooth, a natural algorithm is to find its minimum by computing the gradient $\nabla f$ over the space $X$. The function achieves its minimum if $\nabla f = 0$, which can be written down as a fixed point equation. 

\subsection{Metric Coinduction for Generative AI} 

To make the somewhat abstract discussion of final coalgebras above a bit more concrete, we now briefly describe the concept of {\em metric coinduction} \cite{kozen}, which is a special case of the general principle of coinduction \cite{Aczel1988-ACZNS,rutten2000universal}.  The basic idea is simple to describe, and is based on viewing algorithms as forming contraction mappings in  a metric space. The novelty here for many readers is understanding how this well-studied notion of contractive mappings is related to coinduction and coalgebras. 
One way to analyze the convergence of iterative methods like backproapgation is to see if they can be shown to form contraction mappings in a metric space. 

Consider a complete metric space $(V, d)$ where $d: V \times V \rightarrow (0, 1)$ is a symmetric distance function that satisfies the triangle inequality, and all Cauchy sequences in $V$ converge (We will see later that the property of completeness itself follows from the Yoneda Lemma!). A function $H: V \rightarrow V$ is {\em contractive} if there exists $0 \leq c < 1$ such that for all $u, v \in V$, 

\[ d(H(u), H(v)) \leq c \cdot d(u, v) \]

In effect, applying the algorithm represented by the continuous mapping $H$ causes the distances between $u$ and $v$ to shrink, and repeated application eventually guarantees convergence to a fixed point. The novelty here is to interpret the fixed point as a final coalgebra. We will later see that the concept of a (greatest) fixed point is generalized by the concept of final coalgebras.

\begin{definition}
{\bf Metric Coinduction Principle} \cite{kozen}: If $\phi$ is a closed nonempty subset of a complete metric space $V$, and if $H$ is an eventually contractive map on $V$ that preserves $\phi$, then the unique fixed point $u^*$ of $H$ is in $\phi$. In other words, we can write: 

\begin{eqnarray*}
    \exists u \phi(u), \ \ \ \forall \phi(u) &\Rightarrow& \phi(H(u)) \\
    &\phi(u^*)& 
\end{eqnarray*}
\end{definition}

It should not be surprising to those familiar with contraction mapping style arguments that a large number of applications in AI and ML, including game theory, reinforcement learning and stochastic approximation involve proofs of convergence that exploit properties of contraction mappings. What might be less familiar is how to think of this in terms of the concept of {\em coinductive inference}. To explain this perspective briefly, let us introduce the relevant category theoretic terminology. 

Let us define a category $C$ whose objects are nonempty closed subsets of $V$, and whose arrows are reverse set inclusions. That is, there is a unique arrow $\phi_1 \rightarrow \phi_2$ if $\phi_1 \supseteq \phi_2$. Then, we can define an {\em endofunctor} $\bar{H}$ as the closure mapping ${\bar H}(\phi) = \mbox{cl}(H(\phi))$, where $\mbox{cl}$ denotes the closure in the metric topology of $V$. Note that $\bar{H}$ is an endofunctor on $C$ because $\bar{H}(\phi_1) \supseteq  \bar{H}(\phi_2)$ whenever $\phi_1 \supseteq \phi_2$. 

\begin{definition}
    An {\bf $\bar{H}$-coalgebra} is defined as the pair $(\phi, \phi \supseteq \bar{H}(\phi))$ (or equivalently, we can write $\phi \subseteq H(\phi)$. The {\bf final coalgebra} is the isomorphism $u^* \simeq \bar{H}(u^*)$ where $u^*$ is the unique fixed point of the mapping $H$. The metric coinduction rule can be restated more formally in this case as: 

    \[ \phi \supseteq H(\phi) \ \ \Rightarrow \ \ \phi \supseteq H(u^*) \]
\end{definition}

This result has deep significance, as we will see later, and provides an elegant way to prove contraction style arguments in very general settings. In particular, it can be applied to analyze the convergence of machine learning methods for GAIA models to see if they can be shown to convergence in a (generalized) metric space. We will postpone the details of this analysis to a subsequent paper 

\section{Layer 1 of GAIA: Simplicial Sets for Generative AI} 

\label{layer1} 

\begin{figure}
    \centering
    \includegraphics[scale=0.25]{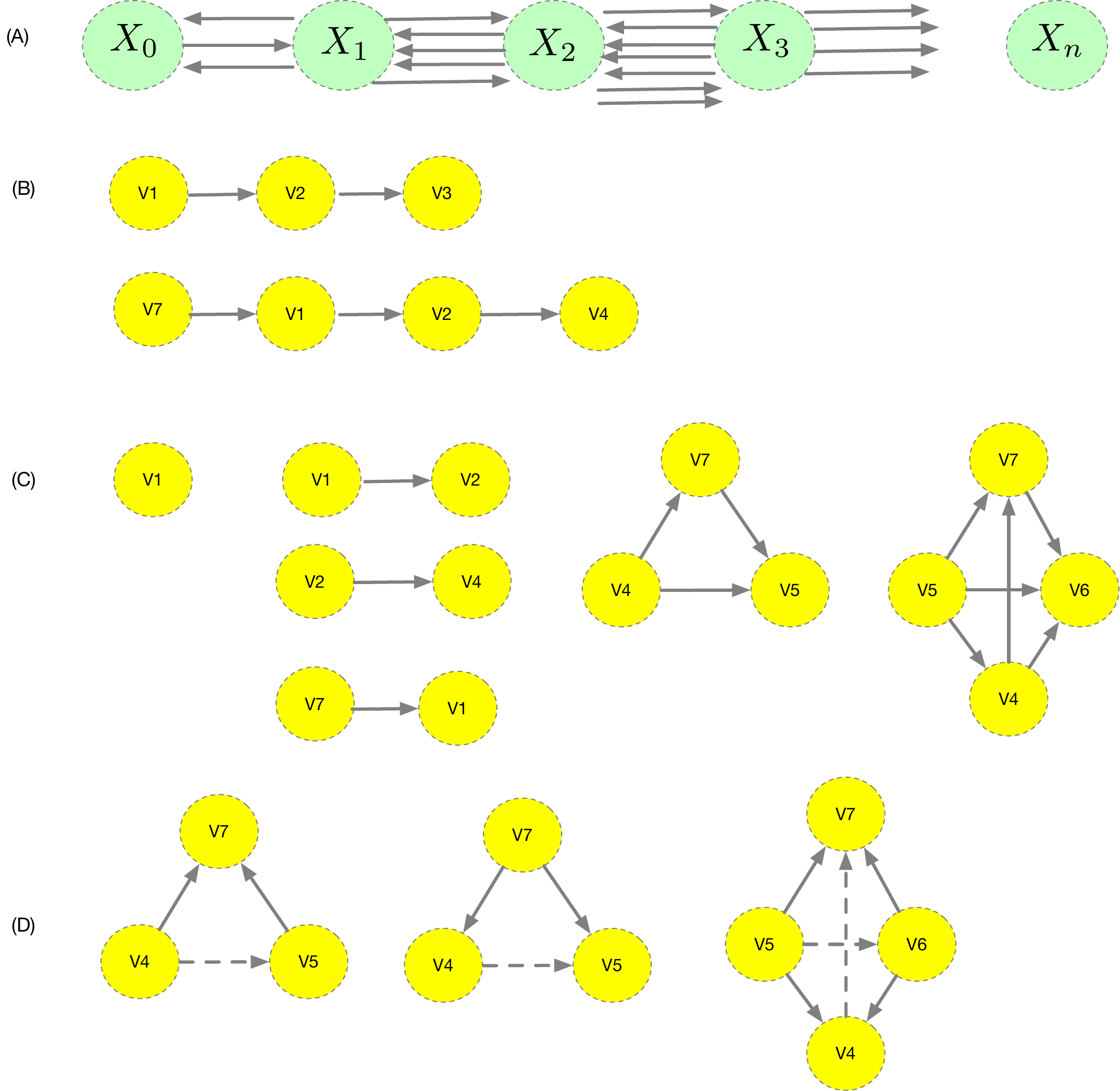}
    \caption{(A) GAIA is based on a {\em hierarchical} simplicial sets and objects. The base simplicial set $X_0$ is a set of entities that can be mapped to computational entities in generative AI, such as a tokens in a large language model, or images in a diffusion based system. The set $X_1$ defines a collection of morphisms between pairs of objects in $X_0$, where each morphism could define a deep learning module as explained in Section~\ref{backprop}. (B) The first two sets $X_0$ and $X_1$ essentially define what is possible with today's compositionally based generative AI system using backpropagation, where learning is conceived of as an entirely sequential process. (C) $X_2$ and higher-level simplicial sets constitute the novel core of GAIA: here, groups of sub-simplicial objects act like business units with a common objective. Each $n$-simplex has $n+1$ sub-simplicial complexes, and information is transmitted hierarchically in GAIA from superior simplicial sets to subordinate simplicial sets using lifting diagrams. (D) Solving ``outer horn" extension problems is more challenging for methods like deep learning with backpropagation, than solving ``inner horn" extensions. }
    \label{deltacat}
\end{figure}

Unlike earlier generative AI architectures, GAIA uses the paradigm of simplicial sets and objects as the basic building blocks for generative AI (see Figure~\ref{3simplex}). We can still cast much of the earlier work described above in terms of GAIA, but this flexibility gives us the power to also formulate generative AI approaches that lie beyond the scope of compositional learning methods, such as backpropagation. As we illustrated in Figure~\ref{simplicialgenAI}, GAIA puts together building blocks of generative AI methods as $n$-simplices of a simplicial set. In the simplest setting, these combine compositionally in the way that \cite{DBLP:conf/lics/FongST19} defined for the category {\tt Learn} that we discussed in detail in the previous section. It is possible to take the category {\tt Learn} and embed it in the category of simplicial sets using the {\em nerve functor} \cite{kerodon}, which is a full and faithful embedding of the category as a simplicial set. Each $n$-simplex is then defined by $n$-length sequences of a generative AI system, like a Transformer building block that computes a permutation-equivariant map. But the left adjoint of the nerve functor that maps back from the simplicial set category is a ``lossy" functor that does not generate a full and faithful embedding, which shows why simplicial learning is more powerful in principle than compositional learning. 

The first layer of GAIA is based on the simplicial category $\Delta$, which serves as a ``combinatorial factory" for piecing together building blocks of generative AI systems into larger units, and for decomposing complex systems into their component subsystems. The category $\Delta$ is defined over ordinal numbers $[n], n \geq 0$, but really ``comes to life" when it is plugged into some concrete category, such as the ones described in the previous section like {\tt Learn} or {\tt Para}. For example, if the parameters of a learning method are defined over a category of {\bf Sets}, then a contravariant functor from $\Delta$ into sets is called a simplicial set \cite{may1992simplicial}. We can also define functors from $\Delta$ into  some category of generative AI models, like Transformers \cite{DBLP:conf/nips/VaswaniSPUJGKP17} or structured state space sequence (S4) models \cite{DBLP:conf/iclr/GuGR22} or diffusion models \cite{DBLP:conf/nips/SongE19}. 

\subsection{Simplicial Sets and Objects} 

As shown in Figure~\ref{deltacat}, a simplicial set can be viewed as a collection of sets, or a {\em graded} set, $S_n, n \geq 0$, where $S_0$ defines the primitive objects (which can be elements of the category {\tt Param} defined in the previous section), $S_1$ represents a collection of ``edge" objects (which can be viewed as Learners as defined in the previous section), $S_2$ represents simplices of three objects interacting, and in general, $S_n$ defines a collection of objects that represents interactions of order $n$. These higher-level simplicial sets act like ``business units" in a company: they have a hierarchical structure, receive inputs and outputs from higher-level superiors and lower-level subordinates, and adjust their internal parameters. These $n$-simplicial sets are related to each other by {\em degeneracy} operators that map $S_n$ into $S_{n+1}$ or {\em face} operators that map $S_n$ into $S_{n-1}$. Figure~\ref{simplicialgenAI} shows an example of a $3$-simplex. Note how in Figure~\ref{3simplex}, the simplicial set $X_3$ sends ``back" information to $X_2$ through four face operators. These exactly correspond to the four subsimplices of each object in $X_3$, as illustrated in Figure~\ref{3simplex}, because each $3$-simplex has four faces. The crux of the GAIA framework is to treat each such simplex as a building block of a generative AI system. 

Simplicial sets are higher-dimensional generalizations of directed graphs, partially ordered sets, as well as regular categories themselves.  Importantly, simplicial sets and simplicial objects form a foundation for higher-order category theory. Simplicial objects have long been a foundation for algebraic topology, and  more recently in  higher-order category theory. The category $\Delta$ has non-empty ordinals $[n] = \{0, 1, \ldots, n]$ as objects, and order-preserving maps $[m] \rightarrow [n]$ as arrows. An important property in $\Delta$ is that any many-to-many mapping is decomposable as a composition of an injective and a surjective mapping,  each of which is decomposable into a sequence of elementary injections $\delta_i: [n] \rightarrow [n+1]$, called {\em {coface}} mappings, which omits $i \in [n]$, and a sequence of elementary surjections $\sigma_i: [n] \rightarrow [n-1]$, called {\em {co-degeneracy}} mappings, which repeats $i \in [n]$. The fundamental simplex $\Delta([n])$ is the presheaf of all morphisms into $[n]$, that is, the representable functor $\Delta(-, [n])$.  The Yoneda Lemma  assures us that an $n$-simplex $x \in X_n$ can be identified with the corresponding map $\Delta[n] \rightarrow X$. Every morphism $f: [n] \rightarrow [m]$ in $\Delta$ is functorially mapped to the map $\Delta[m] \rightarrow \Delta[n]$ in ${\cal S}$. 

Any morphism in the category $\Delta$ can be defined as a sequence of {\em co-degeneracy} and {\em co-face} operators, where the co-face operator $\delta_i: [n-1] \rightarrow [n], 0 \leq i \leq n$ is defined as: 
\[ 
\delta_i (j)  =
\left\{
	\begin{array}{ll}
		j,  & \mbox{for } \ 0 \leq j \leq i-1 \\
		j+1 & \mbox{for } \  i \leq j \leq n-1 
	\end{array}
\right. \] 

Analogously, the co-degeneracy operator $\sigma_j: [n+1] \rightarrow [n]$ is defined as 
\[ 
\sigma_j (k)  =
\left\{
	\begin{array}{ll}
		j,  & \mbox{for } \ 0 \leq k \leq j \\
		k-1 & \mbox{for } \  j < k \leq n+1 
	\end{array}
\right. \] 

Note that under the contravariant mappings, co-face mappings turn into face mappings, and co-degeneracy mappings turn into degeneracy mappings. That is, for any simplicial object (or set) $X_n$, we have $X(\delta_i) \coloneqq d_i: X_n \rightarrow X_{n-1}$, and likewise, $X(\sigma_j) \coloneqq s_j: X_{n-1} \rightarrow X_n$. 

The compositions of these arrows define certain well-known properties \cite{may1992simplicial,richter2020categories}: 
\begin{eqnarray*}
    \delta_j \circ \delta_i &=& \delta_i \circ \delta_{j-1}, \ \ i < j \\
    \sigma_j \circ \sigma_i &=& \sigma_i \circ \sigma_{j+1}, \ \ i \leq j \\ 
    \sigma_j \circ \delta_i (j)  &=&
\left\{
	\begin{array}{ll}
		\sigma_i \circ \sigma_{j+1},  & \mbox{for } \ i < j \\
		1_{[n]} & \mbox{for } \  i = j, j+1 \\ 
		\sigma_{i-1} \circ \sigma_j, \mbox{for} \ i > j + 1
	\end{array}
\right.
\end{eqnarray*}

\begin{example}
The ``vertices'' of a simplicial object ${\cal C}_n$ are the objects in  ${\cal C}$, and the ``edges'' of ${\cal C}$ are its arrows $f: X \rightarrow Y$, where $X$ and $Y$ are objects in ${\cal C}$. Given any such arrow, the degeneracy operators $d_0 f = Y$ and $d_1 f = X$ recover the source and target of each arrow. Also, given an object $X$ of category ${\cal C}$, we can regard the face operator $s_0 X$ as its identity morphism ${\bf 1}_X: X \rightarrow X$. 
\end{example}

\begin{example} 
Given a category ${\cal C}$, we can identify an $n$-simplex $\sigma$ of a simplicial set ${\cal C}_n$ with \mbox{the sequence: }
\[ \sigma = C_o \xrightarrow[]{f_1} C_1 \xrightarrow[]{f_2} \ldots \xrightarrow[]{f_n} C_n \] 
the face operator $d_0$ applied to $\sigma$ yields the sequence 
\[ d_0 \sigma = C_1 \xrightarrow[]{f_2} C_2 \xrightarrow[]{f_3} \ldots \xrightarrow[]{f_n} C_n \] 
where the object $C_0$ is ``deleted'' along with the morphism $f_0$ leaving it.  

\end{example} 

\begin{example} 
Given a category ${\cal C}$, and an $n$-simplex $\sigma$ of the simplicial set ${\cal C}_n$, the face operator $d_n$ applied to $\sigma$ yields the sequence 
\[ d_n \sigma = C_0 \xrightarrow[]{f_1} C_1 \xrightarrow[]{f_2} \ldots \xrightarrow[]{f_{n-1}} C_{n-1} \] 
where the object $C_n$ is ``deleted'' along with the morphism $f_n$ entering it.  Note this face operator can be viewed as analogous to interventions on leaf nodes in a causal DAG model. 

\end{example} 

\begin{example} 
Given a category  ${\cal C}$, and an $n$-simplex $\sigma$ of the simplicial set ${\cal C}_n$
the face operator $d_i, 0 < i < n$ applied to $\sigma$ yields the sequence 
\[ d_i \sigma = C_0 \xrightarrow[]{f_1} C_1 \xrightarrow[]{f_2} \ldots C_{i-1} \xrightarrow[]{f_{i+1} \circ f_i} C_{i+1} \ldots \xrightarrow[]{f_{n}} C_{n} \] 
where the object $C_i$ is ``deleted'' and the morphisms $f_i$ is composed with morphism $f_{i+1}$.  Note that this process can be abstractly viewed as intervening on object $C_i$ by choosing a specific value for it (which essentially ``freezes'' the morphism $f_i$ entering object $C_i$ to a constant value). 

\end{example} 

\begin{example} 
Given a category ${\cal C}$, and an $n$-simplex $\sigma$ of the simplicial set ${\cal C}_n$, 
the degeneracy operator $s_i, 0 \leq i \leq n$ applied to $\sigma$ yields the sequence 
\[ s_i \sigma = C_0 \xrightarrow[]{f_1} C_1 \xrightarrow[]{f_2} \ldots C_{i} \xrightarrow[]{{\bf 1}_{C_i}} C_{i} \xrightarrow[]{f_{i+1}} C_{i+1}\ldots \xrightarrow[]{f_{n}} C_{n} \] 
where the object $C_i$ is ``repeated'' by inserting its identity morphism ${\bf 1}_{C_i}$. 

\end{example} 

\begin{definition} 
Given a category ${\cal C}$, and an $n$-simplex $\sigma$ of the simplicial set ${\cal C}_n$, 
$\sigma$ is a {\bf {degenerate}} simplex if some $f_i$ in $\sigma$  is an identity morphism, in which case $C_i$ and $C_{i+1}$ are equal. 
\end{definition} 

\subsection{Hierarchical Learning in GAIA by solving Lifting Problems}
\label{lift} 

As we mentioned earlier, a crucial difference between GAIA and earlier generative AI architectures is that it is based on a hierarchical model of simplicial learning, rather than the standard compositional learning framework described in Section~\ref{backprop}. To understand how such a structure will work, we need to define some key ideas from higher-order category theory below, but before we do that, we want to build up some intuition as to how this process will work at a more informal level. Figure~\ref{simplicial-learning} illustrates the idea using the same figure we showed earlier as Figure~\ref{3simplex}, but here, we will use it to illustrate the simplicial learning concept. 

\begin{figure}
    \centering
    \includegraphics[scale=0.4]{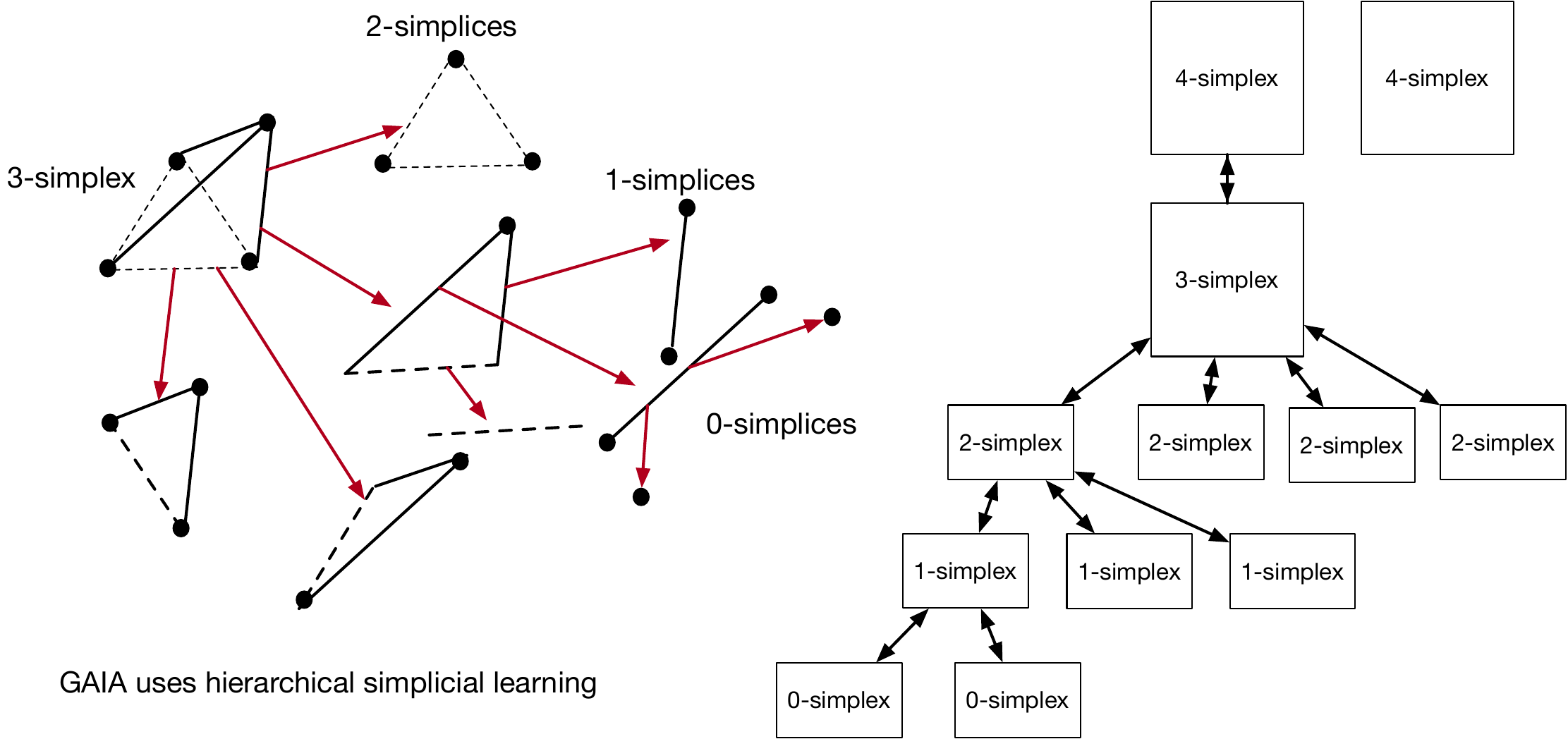}
    \caption{Example of a ``small business unit" in GAIA defined as a $3$-simplex that maintains its set of internal parameters, and updates them based on information it receives from its superiors and subordinates.}
    \label{simplicial-learning}
\end{figure}

To understand how simplicial learning works, let us consider as an example the $3$-simplex shown in Figure~\ref{simplicial-learning}.  We generalize the earlier compositional model defined in Section~\ref{backprop-functor}, where in the category {\tt Learn} (see Figure~\ref{learncat}),  each learner was a morphism in the symmetric monoidal category. In that model, each learner morphism transmits information downstream to its successors and upstream to its predecessors, but there is no hierarchical structure. Here, in Figure~\ref{simplicial-learning}, the simplicial structure defines a hierarchy of learners, so that each learner is not just a morphism anymore, but a $n$-simplex that maintains its internal set of parameters that it then updates based on the information from its superiors and subordinates. To define this more carefully, we can construct a functor that maps the algebraic structure of a simplicial set $\Delta$ into a suitable parameter space (e.g., a symmetric monoidal category like vector spaces), whereby each $n$-simplex now becomes defined as a contravariant functor $\Delta^{op} \rightarrow {\bf Vect}$. 

We can in fact use exactly the same updates defined earlier in Section~\ref{backprop-functor}, following \cite{DBLP:conf/lics/FongST19}, with the crucial difference being that the updates must be consistent across the hierarchical structure of the simplicial complex. So, each $n$-simplex is updated based on data from its subordinate $n-1$ sub-simplicial complexes and its superior $n+1$-simplicial complexes, but these need to be made consistent with each other. To solve this problem requires some additional machinery from higher-order category theory, which we now introduce below. 

Lifting problems provide elegant ways to define solutions to computational problems in category theory regarding the existence of mappings. We will use these lifting diagrams later in this paper. For example, the notion of injective and surjective functions, the notion of separation in topology, and many other basic constructs can be formulated as solutions to lifting problems. Lifting problems define ways of decomposing structures into simpler pieces, and putting them back together again. 
 
 \begin{definition}
 Let ${\cal C}$ be a category. A {\bf {lifting problem}} in ${\cal C}$ is a commutative diagram $\sigma$ in ${\cal C}$. 
 \begin{center}
 \begin{tikzcd}
  A \arrow{d}{f} \arrow{r}{\mu}
    & X \arrow[]{d}{p} \\
  B  \arrow[]{r}[]{\nu}
&Y \end{tikzcd}
 \end{center} 
 \end{definition}
 
 \begin{definition}
 Let ${\cal C}$ be a category. A {\bf {solution to a lifting problem}} in ${\cal C}$ is a morphism $h: B \rightarrow X$ in ${\cal C}$ satisfying $p \circ h = \nu$ and $h \circ f = \mu$ as indicated in the diagram below. 
 \begin{center}
 \begin{tikzcd}
  A \arrow{d}{f} \arrow{r}{\mu}
    & X \arrow[]{d}{p} \\
  B \arrow[ur,dashed, "h"] \arrow[]{r}[]{\nu}
&Y \end{tikzcd}
 \end{center} 
 \end{definition}

 \begin{definition}
 Let ${\cal C}$ be a category. If we are given two morphisms $f: A \rightarrow B$ and $p: X \rightarrow Y$ in ${\cal C}$, we say that $f$ has the {\bf {left lifting property}} with respect to $p$, or that p has the {\bf {right lifting property}} with respect to f if for every pair of morphisms $\mu: A \rightarrow X$ and $\nu: B \rightarrow Y$ satisfying the equations $p \circ \mu = \nu \circ f$, the associated lifting problem indicated in the diagram below. 
 \begin{center}
 \begin{tikzcd}
  A \arrow{d}{f} \arrow{r}{\mu}
    & X \arrow[]{d}{p} \\
  B \arrow[ur,dashed, "h"] \arrow[]{r}[]{\nu}
&Y \end{tikzcd}
 \end{center} 
admits a solution given by the map $h: B \rightarrow X$ satisfying $p \circ h = \nu$ and $h \circ f = \mu$. 
 \end{definition}

 \cite{lifting} shows that a remarkable number of properties in mathematics can be defined as lifting problems. \cite{SPIVAK_2013} showed that query answering in languages like SQL in relational databases can be formalized as lifting problems. \cite{DBLP:journals/entropy/Mahadevan23} showed that causal inference could be posed in terms of lifting problems. At its heart, a lifting problem defines a constrained search over a space of parameters, and it is that property that makes it so useful in generative AI because in effect, methods like backpropagation can be viewed as solving lifting problems. As a simple example to build intuition, here is a way any surjective (onto) function as a solution to a lifting problem. 
 
 \begin{example}
 Given the paradigmatic non-surjective morphism $f: \emptyset \rightarrow \{ \bullet \}$, any morphism p that has the right lifting property with respect to f is a {\bf {surjective mapping}}. 
. 
 \begin{center}
 \begin{tikzcd}
  \emptyset \arrow{d}{f} \arrow{r}{\mu}
    & X \arrow[]{d}{p} \\
  \{ \bullet \} \arrow[ur,dashed, "h"] \arrow[]{r}[]{\nu}
&Y \end{tikzcd}
 \end{center} 

 \end{example}

 Similarly, here is another lifting problem whose solution defines an $1-1$ injective function. 

  \begin{example}
 Given the paradigmatic non-injective morphism $f: \{ \bullet, \bullet \} \rightarrow \{ \bullet \}$, any morphism p that has the right lifting property with respect to f is an {\bf {injective mapping}}. 
.
 \begin{center}
 \begin{tikzcd}
  \{\bullet, \bullet \} \arrow{d}{f} \arrow{r}{\mu}
    & X \arrow[]{d}{p} \\
  \{ \bullet \} \arrow[ur,dashed, "h"]  \arrow[]{r}[]{\nu}
&Y \end{tikzcd}
 \end{center} 

 \end{example}

\subsection{Simplicial Subsets and Horns in GAIA}

To explain how lifting problems can be used for generative AI in GAIA, we need to define lifting problems over $n$-simplicial complexes. The basic idea, as illustrated in Figure~\ref{simplicial-learning}, is that we construct a solution to a lifting problem by asking if a particular sub-simplicial complex can be ``extended" into the whole complex. This extension process is essentially what methods like backpropagation are doing, and universal approximation results for Transformers \cite{DBLP:conf/iclr/YunBRRK20} are in effect saying that a solution to a lifting problem exists for a particular class of simplicial complexes defined as $n$-length sequences of Transformer models. 

We first describe more complex ways of extracting parts of categorical structures using simplicial subsets and horns. These structures will play a key role in defining suitable lifting problems. 
 
 \begin{definition}
 The {\bf {standard simplex}} $\Delta^n$ is the simplicial set defined by the construction 
 \[ ([m] \in \Delta) \mapsto {\bf Hom}_\Delta([m], [n]) \] 
 
 By convention, $\Delta^{-1} \coloneqq \emptyset$. The standard $0$-simplex $\Delta^0$ maps each $[n] \in \Delta^{op}$ to the single element set $\{ \bullet \}$. 
 \end{definition}
 
 \begin{definition}
 Let $S_\bullet$ denote a simplicial set. If for every integer $n \geq 0$, we are given a subset $T_n \subseteq S_n$, such that the face and degeneracy maps 
 \[ d_i: S_n \rightarrow S_{n-1} \ \ \ \ s_i: S_n \rightarrow S_{n+1} \] 
 applied to $T_n$ result in 
 \[ d_i: T_n \rightarrow T_{n-1} \ \ \ \ s_i: T_n \rightarrow T_{n+1} \] 
 then the collection $\{ T_n \}_{n \geq 0}$ defines a {\bf {simplicial subset}} $T_\bullet \subseteq S_\bullet$
 \end{definition}
 
 \begin{definition}
 The {\bf {boundary}} is a simplicial set $(\partial \Delta^n): \Delta^{op} \rightarrow$ {\bf {Set}} defined as
 \[ (\partial \Delta^n)([m]) = \{ \alpha \in {\bf Hom}_\Delta([m], [n]): \alpha \ \mbox{is not surjective} \} \]
 \end{definition}
 
 Note that the boundary $\partial \Delta^n$ is a simplicial subset of the standard $n$-simplex $\Delta^n$. 
 
 \begin{definition}
 The {\bf {Horn}} $\Lambda^n_i: \Delta^{op} \rightarrow$ {\bf {Set}} is defined as
 \[ (\Lambda^n_i)([m]) = \{ \alpha \in {\bf Hom}_\Delta([m],[n]): [n] \not \subseteq \alpha([m]) \cup \{i \} \} \] 
 \end{definition}
 
 Intuitively, the Horn $\Lambda^n_i$ can be viewed as the simplicial subset that results from removing the interior of the $n$-simplex $\Delta^n$ together with the face opposite its $i$th vertex.   
 
Consider the problem of composing $1$-dimensional simplices  to form a $2$-dimensional simplicial object. Each simplicial subset of an $n$-simplex induces a  a {\em horn} $\Lambda^n_k$, where  $ 0 \leq k \leq n$. Intuitively, a horn is a subset of a simplicial object that results from removing the interior of the $n$-simplex and the face opposite the $i$th vertex. Consider the three horns defined below. The dashed arrow  $\dashrightarrow$ indicates edges of the $2$-simplex $\Delta^2$ not contained in the horns. 
\begin{center}
 \begin{tikzcd}[column sep=small]
& \{0\}  \arrow[dl] \arrow[dr] & \\
  \{1 \} \arrow[rr, dashed] &                         & \{ 2 \} 
\end{tikzcd} \hskip 0.5 in 
 \begin{tikzcd}[column sep=small]
& \{0\}  \arrow[dl] \arrow[dr, dashed] & \\
  \{1 \} \arrow{rr} &                         & \{ 2 \} 
\end{tikzcd} \hskip 0.5in 
 \begin{tikzcd}[column sep=small]
& \{0\}  \arrow[dl, dashed] \arrow[dr] & \\
  \{1 \} \arrow{rr} &                         & \{ 2 \} 
\end{tikzcd}
\end{center}

The inner horn $\Lambda^2_1$ is the middle diagram above, and admits an easy solution to the ``horn filling'' problem of composing the simplicial subsets. The two outer horns on either end pose a more difficult challenge. For example, filling the outer horn $\Lambda^2_0$ when the morphism between $\{0 \}$ and $\{1 \}$ is $f$ and that between $\{0 \}$ and $\{2 \}$ is the identity ${\bf 1}$ is tantamount to finding the left inverse of $f$ up to homotopy. Dually, in this case, filling the outer horn $\Lambda^2_2$ is tantamount to finding the right inverse of $f$ up to homotopy. A considerable elaboration of the theoretical machinery in category theory is required to describe the various solutions proposed, which led to different ways of defining higher-order category theory \cite{weakkan,quasicats,Lurie:higher-topos-theory}. 

\subsection{Higher-Order Categories} 

We now formally introduce higher-order categories, building on the framework proposed in a number of formalisms. We briefly summarize various approaches to the horn filling problem in higher-order category theory.
 
 \begin{definition}
 Let $f: X \rightarrow S$ be a morphism of simplicial sets. We say $f$ is a {\bf {Kan fibration}} if, for each $n > 0$, and each $0 \leq i \leq n$, every lifting problem. 
 \begin{center}
 \begin{tikzcd}
  \Lambda^n_i \arrow{d}{} \arrow{r}{\sigma_0}
    & X \arrow[]{d}{f} \\
  \Delta^n \arrow[ur,dashed, "\sigma"] \arrow[]{r}[]{\bar{\sigma}}
&S \end{tikzcd}
 \end{center}  
 admits a solution. More precisely, for every map of simplicial sets $\sigma_0: \Lambda^n_i \rightarrow X$ and every $n$-simplex $\bar{\sigma}: \Delta^n \rightarrow S$ extending $f \circ \sigma_0$, we can extend $\sigma_0$ to an $n$-simplex $\sigma: \Delta^n \rightarrow X$ satisfying $f \circ \sigma = \bar{\sigma}$. 
 \end{definition}
 
 \begin{example}
Given a simplicial set $X$, then a projection map $X \rightarrow \Delta^0$ that is a Kan fibration is called a {\bf {Kan complex}}. 
\end{example} 

\begin{example}
Any isomorphism between simplicial sets is a Kan fibration. 
\end{example}

\begin{example}
The collection of Kan fibrations is closed under retracts. 
\end{example}

\begin{definition}\cite{Lurie:higher-topos-theory}
\label{ic} 
An $\infty$-category is a simplicial object $S_\bullet$ which satisfies the following condition: 

\begin{itemize} 
\item For $0 < i < n$, every map of simplicial sets $\sigma_0: \Lambda^n_i \rightarrow S_\bullet$ can be extended to a map $\sigma: \Delta^n \rightarrow S_i$. 
\end{itemize} 
\end{definition}

This definition emerges out of a common generalization of two other conditions on a simplicial set $S_i$: 

\begin{enumerate} 
\item {\bf {Property K}}: For $n > 0$ and $0 \leq i \leq n$, every map of simplicial sets $\sigma_0: \Lambda^n_i \rightarrow S_\bullet$ can be extended to a map $\sigma: \Delta^n \rightarrow S_i$. 

\item {\bf {Property C}}:  for $0 < 1 < n$, every map of simplicial sets $\sigma_0: \Lambda^n_i \rightarrow S_i$ can be extended uniquely to a map $\sigma: \Delta^n \rightarrow S_i$. 
\end{enumerate} 

Simplicial objects that satisfy property K were defined above to be Kan complexes. Simplicial objects that satisfy property C above can be identified with the nerve of a category, which yields a full and faithful embedding of a category in the category of sets.  definition~\ref{ic} generalizes both of these definitions, and was called a {\em {quasicategory}} in \cite{quasicats} and {\em {weak Kan complexes}} in \cite{weakkan} when ${\cal C}$ is a category.

\section{Layer 2 of GAIA: Generative AI using Simplicial Categories} 

\label{layer2} 

The second layer of GAIA is based on defining generative models as universal coalgebras over some base category, including the standard mathematical categories ({\bf Sets}, measurable spaces {\bf Meas}, topoogical spaces {\bf Top} or Vector spaces {\bf Vect}). Existing approaches to generative AI, such as Transformers \cite{DBLP:conf/nips/VaswaniSPUJGKP17}, structured state-space models \cite{DBLP:conf/iclr/GuGR22}, or image diffusion models \cite{DBLP:conf/nips/SongE19} can all be defined as (stochastic) coalgebras over one of the base categories. 
We first introduce some basic categorical structures, and then define the category of  universal coalgebras over these. to  define generative AI systems as coalgebras. 

\subsection{Categories as Building Blocks of GAIA} 

GAIA is built on the hypothesis that category theory provides a universal language for encoding foundation models. We define a few salient aspects of category theory in this section, including showing how it can reveal surprising similarities between algebraic structures that superficially look very different, such as metric spaces and partially ordered sets (see Table~\ref{homtable}). A key principle that is often exploited is to explicitly represent the structure in the collection of morphisms between two objects. That is, for some category {\cal C}, the {\bf Hom}$_{\cal C}(a,b)$ between objects $a$ and $b$ might itself have some additional structure beyond that of merely being a collection or a set. For example, in the category of vector spaces, the set of morphisms (linear transformations) between two vector spaces $U$ and $V$ is itself a vector space. So-called {\cal V}-enriched categories signify cases when the {\bf Hom} values are specified in some structure {\cal V}. Examples include metric spaces, where the {\bf Hom} values are non-negative real numbers representing distances, and partially ordered sets (posets) where the {\bf Hom} values are Boolean. 

\begin{table}[t] 
\caption{Categories are defined as collections of objects and arrows between them.}
\vskip 0.1in
\begin{minipage}{0.7\textwidth}
 \begin{small}\hfill 
  \begin{tabular}{|c|c|c|c|c|} \hline 
{\cal C}  & {\bf Hom}$_{\cal C}$ values   & Composition  and & Domain  & Domain for \\ 
& & and identity law & for composition & for identity laws \\ \hline 
General category & Sets & Functions & Cartesian product & One element set \\ \hline 
Metric spaces & Non-negative numbers  & $\geq$  & sum & zero \\ \hline  
Posets & Truth values & Entailment & Conjunction & true \\ \hline
{\cal V}-enriched  & objects  & morphisms  & tensor product  & unit object for   \\ 
category & in {\cal V} & in {\cal V} & in {\cal V} & tensor product in {\cal V} \\ \hline 
$F:{\cal C}^{op} \times C \rightarrow D$ & Bivalent functors & Dinatural transformations & Probabilities, distances & Unit object \\
Coends, ends & & & topological embeddings & \\ \hline 
\end{tabular}
\end{small}
\end{minipage}
\label{homtable}
\end{table}

The aim of category theory is to build a ``unified field theory" of mathematics based on a simple model of {\em objects} that {\em interact} with each other, analogous to directed graph representations. In graphs, vertices represent arbitrary entities, and the edges denote some form of (directional) interaction. In categories, there is no restriction on how many edges exist between any two objects. In a {\em locally small} category, there is assumed to be a ``set's worth" of edges, meaning that it could still be infinite! In addition, small categories are assumed to contain a set's worth of objects (again, which might not be finite).  The framework is {\em compositional}, in that categories can be formed out of objects, {\em arrows} that define the interaction between objects, or {\em functors} that define the interactions between categories. This compositionality gives us a rich {\em generative} space of models that will be invaluable in modeling UIGs. 

Category theory gives an exceptional set of  ``measuring tools" for modeling Transformers and other generative models in AI. Choosing a category means selecting a collection of objects and a collection of composable arrows by which each pair of objects interact. This choice of objects and arrows defines the measurement apparatus that is used in formulating and solving an imitation game. A key result called the Yoneda Lemma shows that {\em objects can be identified up to isomorphism solely by their interactions with other objects}. Category theory also embodies the principle of  {\em universality}: a property is universal if it defines an {\em initial} or {\em final} object in a category. Many approaches in generative AI, such as probabilistic generative models or distance metrics, can be abstractly characterized as initial or final objects in a category of {\em wedges}, where the objects are bifunctors and the arrows are dinatural transformations.  \cite{loregian_2021} has an excellent treatment of the calculus of coends, which we will discuss in detail later in the paper. At a high level, the notion of object isomorphism in category theory is defined as follows. 

\begin{definition}
Two objects $X$ and $Y$ in a category ${\cal C}$ are deemed {\bf {isomorphic}}, or $X \cong Y$ if and only if there is an invertible morphism $f: X \rightarrow Y$, namely $f$ is both {\em left invertible} using a morphism $g: Y \rightarrow X$ so that $g \circ f = $ {{\bf id}}$_X$, and $f$ is {\em right invertible} using a morphism $h$ where $f \circ h = $ {{\bf id}}$_Y$. 
\end{definition}

Category theory provides a rich language to describe how objects interact, including notions like {\em braiding} that plays a key role in quantum computing \cite{Coecke_2016}. The notion of isomorphism can be significantly weakened to include notions like homotopy.  This notion of homotopy generalizes the notion of homotopy in topology, which defines why an object like a coffee cup is topologically homotopic to a doughnut (they have the same number of ``holes'').  In the category {\bf {Sets}}, two finite sets are considered isomorphic if they have the same number of elements, as it is then trivial to define an invertible pair of morphisms between them. In the category {\bf {Vect}}$_k$ of vector spaces over some field $k$, two objects (vector spaces) are isomorphic if there is a set of invertible linear transformations between them. As we will see below, the passage from a set to the ``free'' vector space generated by elements of the set is another manifestation of the universal arrow property. In the category of topological spaces {\bf Top}, two objects are isomorphic if there is a pair of continuous functions that makes them {\em homeomorphic} \cite{may2012more}. A more refined category is {\em hTop}, the category defined by topological spaces where the arrows are now given by homotopy classes of continuous functions. 
     
 \begin{definition}
 Let $C$ and $C'$ be a pair of objects in a category ${\cal C}$. We say $C$ is {\bf {a retract}} of $C'$ if there exists maps $i: C \rightarrow C'$ and $r: C' \rightarrow C$ such that $r \circ i = \mbox{id}_{\cal C}$. 
 \end{definition}
 
 \begin{definition}
 Let ${\cal C}$ be a category. We say a morphism $f: C \rightarrow D$ is a {\bf {retract of another morphism}} $f': C \rightarrow D$ if it is a retract of $f'$ when viewed as an object of the functor category {\bf {Hom}}$([1], {\cal C})$. A collection of morphisms $T$ of ${\cal C}$ is {\bf {closed under retracts}} if for every pair of morphisms $f, f'$ of ${\cal C}$, if $f$ is a retract of $f'$, and $f'$  is in $T$, then $f$ is also in $T$. 
 \end{definition}

 The point of these examples is to illustrate that choosing a category, which means choosing a collection of objects and arrows, is like defining a measurement system for deciding if two objects are isomorphic. A richer model of interaction is provided by {\em simplicial sets} \cite{may1992simplicial}, which is a {\em graded set} $S_n, \ n \geq 0$, where $S_0$ represents a set of non-interacting objects, $S_1$ represents a set of pairwise interactions, $S_2$ represents a set of three-way interactions, and so on. We can map any category into a simplicial set by constructing sequences of length $n$ of composable morphisms. For example, we can model sequences of words in a language as composable morphisms, thereby constructing a simplicial set representation of language-based interactions in an imitation game. Then, the corresponding notion of homotopy between simplicial sets is defined as \cite{richter2020categories}: 

 \begin{definition}
  Let X and Y be simplicial sets, and suppose we are given a pair of morphisms $f_0, f_1: X \rightarrow Y$. A {\bf {homotopy}} from $f_0$ to $f_1$ is a morphism $h: \Delta^1 \times X \rightarrow Y$ satisfying $f_0 = h |_{{0} \times X}$ and $f_1 = h_{ 1 \times X}$. 
 \end{definition}

\begin{figure}[t]
\begin{center}
\begin{tabular}{|c |c | } \hline 
{\bf Set theory } & {\bf Category theory} \\ \hline 
 set & object \\ 
 subset & subobject \\
 truth values $\{0, 1 \} $ & subobject classifier $\Omega$ \\
power set $P(A) = 2^A$ & power object $P(A) = \Omega^A$ \\ \hline
bijection & isomorphims \\ 
injection & monic arrow \\
surjection & epic arrow \\ \hline
singleton set $\{ * \}$ & terminal object ${\bf 1}$ \\ 
empty set $\emptyset$ & initial object ${\bf 0}$ \\
elements of a set $X$ & morphism $f: {\bf 1} \rightarrow X$ \\
- & functors, natural  transformations \\ 
- & limits, colimits, adjunctions \\ \hline
\end{tabular}
\end{center}
\caption{Comparison of notions from set theory and category theory.} 
\label{setvscategories}
\end{figure} 

\begin{figure}[h]
\centering
\includegraphics[scale=.5]{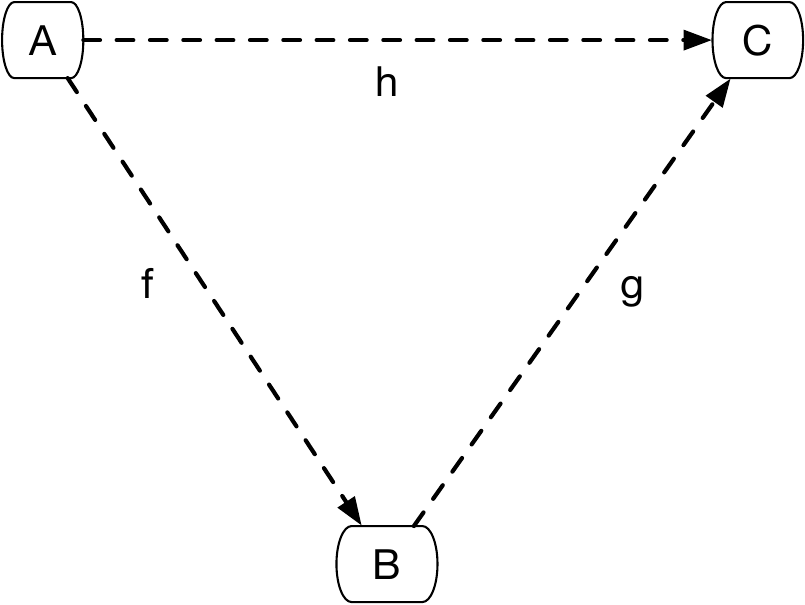}
\caption{Category theory is a compositional model of a system in terms of objects and their interactions.}
\label{morphisms} 
\end{figure}

Figure~\ref{setvscategories} compares the basic notions in set theory vs. category theory. Figure~\ref{morphisms} illustrates a simple category of 3 objects: $A$, $B$, and $C$ that interact through the morphisms $f: A \rightarrow B$, $g: B \rightarrow C$, and $h: A \rightarrow C$. Categories involve a fundamental notion of {\em composition}: the morphism $h: A \rightarrow C$ can be defined as the composition $g \circ f$ of the morphisms from $f$ and $g$. What the objects and morphisms represent is arbitrary, and like the canonical directed graph model, this abstractness gives category theory -- like graph theory -- a universal quality in terms of applicability to a wide range of problems. While categories and graphs and intimately related, in a category, there is no assumption of finiteness in terms of the cardinality of objects or morphisms. A category is defined to be {\em small} or {\em locally small} if there is a set's worth of objects and between any two objects, a set's worth of morphisms, but of course, a set need not be finite. As a simple example, the set of integers $\mathbb{Z}$ defines a category, where each integer $z$ is an object and there is a morphism $f: a \rightarrow b$ between integers $a$ and $b$ if $a \leq b$. This example serves to immediately clarify an important point: a category is only defined if both the objects and morphisms are defined. The category of integers $\mathbb{Z}$ may be defined in many ways, depending on what the morphisms represent. 

Briefly, a category is a collection of objects, and a collection of morphisms between pairs of objects, which are closed under composition, satisfy associativity, and include an identity morphism for every object. For example, sets form a category under the standard morphism of functions. Groups, modules, topological spaces and vector spaces all form categories in their own right, with suitable morphisms (e.g, for groups, we use group homomorphisms, and for vector spaces, we use linear transformations). 

A simple way to understand the definition of a category is to view it as a ``generalized" graph, where there is no limitation on the number of vertices, or the number of edges between any given pair of vertices. Each vertex defines an object in a category, and each edge is associated with a morphism. The underlying graph induces a ``free'' category where we consider all possible paths between pairs of vertices (including self-loops) as the set of morphisms between them. In the reverse direction, given a category, we can define a ``forgetful'' functor that extracts the underlying graph from the category, forgetting the composition rule. 

\begin{definition}
\label{cat-defn}
A {\bf {graph}} ${\cal G}$ (sometimes referred to as a quiver) is a labeled directed multi-graph defined by a set $O$ of {\em objects}, a set $A$ of {\em arrows},  along with two morphisms $s: A \rightarrow O$ and $t: A \rightarrow O$ that specify the domain and co-domain of each arrow.  In this graph, we define the set of composable pairs of arrows by the set 
\[ A \times_O A = \{\langle g, f  \rangle | \ g, f \in A, \ \ s(g) = t(f) \} \]

A {\bf {category}} ${\cal C}$ is a graph ${\cal G}$ with two additional functions: {${\bf id}:$} $O \rightarrow A$, mapping each object $c \in C$ to an arrow {${\bf id}_c$} and $\circ: A \times_O A \rightarrow A$, mapping each pair of composable morphisms $\langle f, g \rangle$ to their composition $g \circ f$. 
\end{definition}

It is worth emphasizing that no assumption is made here of the finiteness of a graph, either in terms of its associated objects (vertices) or arrows (edges). Indeed, it is entirely reasonable to define categories whose graphs contain an infinite number of edges. A simple example is the group $\mathbb{Z}$ of integers under addition, which can be represented as a single object, denoted $\{ \bullet \}$ and an infinite number of morphisms $f: \bullet \rightarrow \bullet$, each of which represents an integer, where composition of morphisms is defined by addition. In this example, all morphisms are invertible. In a general category with more than one object, a {\em groupoid} defines a category all of whose morphisms are invertible. A central principle in category theory is to avoid the use of equality, which is pervasive in mathematics, in favor of a more general notion of {\em isomorphism} or weaker versions of it. Many examples of categories can be given that are relevant to specific problems in AI and ML.  Some examples of categories of common mathematical structures are illustrated below. 

\begin{itemize} 

\item {\bf Set}: The canonical example of a category is {\bf Set}, which has as its objects, sets, and morphisms are functions from one set to another. The {\bf Set} category will play a central role in our framework, as it is fundamental to the universal representation constructed by Yoneda embeddings. 

\item {\bf Top:} The category {\bf Top} has topological spaces as its objects, and continuous functions as its morphisms. Recall that a topological space $(X, \Xi)$ consists of a set $X$, and a collection of subsets $\Xi$ of $X$ closed under finite intersection and arbitrary unions. 

\item {\bf Group:} The category {\bf Group} has groups as its objects, and group homomorphisms as its morphisms. 

\item {\bf Graph:} The category {\bf Graph} has graphs (undirected) as its objects, and graph morphisms (mapping vertices to vertices, preserving adjacency properties) as its morphisms. The category {\bf DirGraph} has directed graphs as its objects, and the morphisms must now preserve adjacency as defined by a directed edge. 

\item {\bf Poset:} The category {\bf Poset} has partially ordered sets as its objects and order-preserving functions as its morphisms. 

\item {\bf Meas:} The category {\bf Meas} has measurable spaces as its objects and measurable functions as its morphisms. Recall that a measurable space $(\Omega, {\cal B})$ is defined by a set $\Omega$ and an associated $\sigma$-field of subsets {\cal B} that is closed under complementation, and arbitrary unions and intersections, where the empty set $\emptyset \in {\cal B}$. 

\end{itemize} 

\subsection{A Categorical Theory of Transformer Models} 

To illustrate the power of the simplicial sets and objects framework, we want to briefly explain how we can use it to define a novel hierarchical framework for generative AI, where each morphism $[m] \rightarrow [n]$ can be mapped into a Transformer module \cite{DBLP:conf/nips/VaswaniSPUJGKP17}.  It is straightforward to extend our discussion below to  other building blocks of generative AI systems, including structured state space sequence models \cite{DBLP:conf/iclr/GuGR22} or image diffusion models \cite{DBLP:conf/nips/SongE19}. As with all generative AI systems, the fundamental structure of a Transformer model is that it is a compositional structure made up of modular components, each of which computes a {\em permutation-equivariant} function over the vector space $\mathbb{R}^{d \times n}$ of $n$-length sequences of tokens, each embedded in a space of dimension $d$. We can define a commutative diagram showing the permutation equivariant property as shown below. 

To begin with, following \cite{DBLP:conf/lics/FongST19}, we can generically  define a neural network layer of type $(n_1, n_2)$ as a subset $C \subseteq [n_1] \times [n_2]$ where $n_1, n_2 \in \mathbb{N}$ are natural numbers, and $[n] = \{1, \ldots, n \}$. Notice how these can be viewed as the objects of a simplicial category $\Delta$. These numbers $n_1$ and $n_2$ serve to define the number of inputs and outputs of each layer, $C$ is a set of connections, and $(i, j) \in C$ means that node $i$ is connected to node $j$ in the network diagram. It is straightforward, but perhaps tedious, to define activation functions $\sigma: \mathbb{R} \rightarrow \mathbb{R}$ for each layer, but essentially each network layer defines a parameterized function $I: \mathbb{R}^{|C| + n_2} \times \mathbb{R}^{n_1} \rightarrow \mathbb{R}^{n_2}$, where the $\mathbb{R}^|C|$ define the edge weights of each network edge and the $\mathbb{R}^{n_2}$ factor encodes individual unit biases. We can specialize these to Transformer models, in particular, noting that the Transfomer models compute specialized types of permutation-equivariant functions as defined by the commutative diagram below. 

\[\begin{tikzcd}
	X && Y && Z \\
	\\
	XP && YP && ZP
	\arrow["f", from=1-1, to=1-3]
	\arrow["P", from=1-3, to=3-3]
	\arrow["P"', from=1-1, to=3-1]
	\arrow["f"', from=3-1, to=3-3]
	\arrow["g", from=1-3, to=1-5]
	\arrow["g"', from=3-3, to=3-5]
	\arrow["P", from=1-5, to=3-5]
\end{tikzcd}\]

In the above commutative diagram, vertices are objects, and arrows are morphisms that define the action of a Transformer block. Here, $X \in \mathbb{R}^{d \times n}$ is a $n$-length sequence of tokens of dimensionality $d$. $P$ is a permutation matrix. The function $f$ computed by a Transformer block is such that $f(XP) = f(X) P$. This property is defined in the above diagram by setting $Y = f(X) P$, which can be computed in two ways, either first by permuting the input by the matrix $P$, and then applying $f$, or by 

Let us understand the permutation equivariant property of the Transformer model in a bit more detail. Our notation for the Transformer model is based on \cite{DBLP:conf/iclr/YunBRRK20}, although there are countless variations in the literature that we do not discuss further. These can be adapted into our categorical framework fairly straightforwardly based on the approach outlined below. Transformer models are inherently compositional, which makes them particularly convenient to model using category theory. \begin{definition}\cite{DBLP:conf/iclr/YunBRRK20,DBLP:conf/nips/VaswaniSPUJGKP17}
    A {\bf Transformer} block is a sequence-to-sequence function mapping $\mathbb{R}^{d \times n} \rightarrow \mathbb{R}^{d \times n}$. There are generally two layers: a {\em self-attention} layer \cite{DBLP:conf/nips/VaswaniSPUJGKP17} and a token-wise feedforward layer. We assume tokens are embedded in a space of dimension $d$. Specifically, we model the inputs $X \in \mathbb{R}^{d \times n}$ to a Transformer block as $n$-length sequences of tokens in $d$ dimensions, where each block computes the following function defined as $t^{h,m,r}: \mathbb{R}^{d \times n}: \mathbb{R}^{d \times n}$: 

\begin{eqnarray*}
    \mbox{Attn}(X) &=& X + \sum_{i=1}^h W^i_O W^i_V X \cdot \sigma[W^i_K X)^T W^i_Q X] \\
    \mbox{FF}(X) &=& \mbox{Attn}(X) + W_2 \cdot \mbox{ReLU}(W_1 \cdot \mbox{Attn}(X) + b_1 {\bf 1}^T_n, 
\end{eqnarray*}

where $W^i_O \in \mathbb{R}^{d \times n}$, $W^i_K, W^i_Q, W^i_Q \in \mathbb{R}^{d \times n}$, $W_2 \in \mathbb{R}^{d \times r}$, $W_1 \in \mathbb{R}^{r \times d}$, and $b_1 \in \mathbb{R}^r$. The output of a Transformer block is $FF(X)$. Following convention, the number of ``heads" is $h$, and each ``head"  size $m$
are the principal parameters of the attention layer, and the size of the ``hidden" feed-forward layer is $r$. 

\end{definition}

Transformer models take as input objects $X \in \mathbb{R}^{d \times n}$ representing $n$-length sequences of tokens in $d$ dimensions, and act as morphisms that represent permutation equivariant functions $f: \mathbb{R}^{d \times n} \rightarrow \mathbb{R}^{d \times n}$ such that $f(XP) = f(X)P$ for any permutation matrix.  \cite{DBLP:conf/iclr/YunBRRK20} show that the actual function computed by the Transformer model defined above is a permutation equivariant mapping. 

Categories are compositional structures, which can be built out of smaller objects. Concretely, we define a category of transformers ${\cal C}_{T}$  where the objects are vectors $x \in \mathbb{R}^{d \times n}$ representing sequences of $d$-dimensional tokens of length $n$, and the composable arrows are {\em permutation-invariant} functions ${\cal T}^{h,m,r}$ comprised of a composition of  transformer blocks $t^{h,m,r}$ of $h$ heads of size $m$ each, and a feedforward layer of $r$ hidden nodes.  Objects in a category interact with each other through arrows or morphisms. In the category ${\cal C}_T$ of Transformer models, the morphisms are the equivariant maps $f$ by which one Transformer model block can be composed with another. 

\begin{definition}
    The category ${\cal C}_T$ of Transformer models is defined as follows: 
    \begin{itemize}
        \item The objects $\mbox{Obj({\cal C})}$ are defined as vectors $X \in \mathbb{R}^{d \times n}$ denoting $n$-length sequences of tokens of dimension $d$. 

        \item The arrows or morphisms of the category ${\cal C}_T$ are defined as a family of sequence-to-sequence functions and defined as: 

        \[ T^{h,m,r} \coloneqq \{f: \mathbb{R}^{d \times n} \rightarrow \mathbb{R}^{d \times n} \ | \ \mbox{where} \ f(XP) = X P, \ \mbox{for some permutation matrix} \ P \} \]
    \end{itemize}
\end{definition}

\subsection{Constructing Simplicial Transformers from Transformer Categories}

We now show how we can define a novel hierarchical theory of simplicial Transformers, first by embedding the category of Transformers into a simplicial set by computing the {\em nerve} of a functor that maps ${\cal C}_T$ into the simplicial set $S^T_\bullet$.  The nerve of a category ${\cal C}$ enables embedding ${\cal C}$ into the category of simplicial objects, which is a fully faithful embedding \cite{Lurie:higher-topos-theory,richter2020categories}. 

\begin{definition} 
\label{fully-faithful} 
Let ${\cal F}: {\cal C} \rightarrow {\cal D}$ be a functor from category ${\cal C}$ to category ${\cal D}$. If for all arrows $f$ the mapping $f \rightarrow F f$

\begin{itemize}
    \item injective, then the functor ${\cal F}$ is defined to be {\bf {faithful}}. 
    \item surjective, then the functor ${\cal F}$ is defined to be {\bf {full}}.  
    \item bijective, then the functor ${\cal F}$ is defined to be {\bf {fully faithful}}. 
\end{itemize}
\end{definition} 

\begin{definition}
The {\bf {nerve}} of a category ${\cal C}$ is the set of composable morphisms of length $n$, for $n \geq 1$.  Let $N_n({\cal C})$ denote the set of sequences of composable morphisms of length $n$.  
\[ \{ C_o \xrightarrow[]{f_1} C_1 \xrightarrow[]{f_2} \ldots \xrightarrow[]{f_n} C_n \ | \ C_i \ \mbox{is an object in} \ {\cal C}, f_i \ \mbox{is a morphism in} \ {\cal C} \} \] 
\end{definition}

The set of $n$-tuples of composable arrows in {\cal C}, denoted by $N_n({\cal C})$,  can be viewed as a functor from the simplicial object $[n]$ to ${\cal C}$.  Note that any nondecreasing map $\alpha: [m] \rightarrow [n]$ determines a map of sets $N_m({\cal C}) \rightarrow N_n({\cal C})$.  The nerve of a category {\cal C} is the simplicial set $N_\bullet: \Delta \rightarrow N_n({\cal C})$, which maps the ordinal number object $[n]$ to the set $N_n({\cal C})$.

The importance of the nerve of a category comes from a key result \cite{kerodon,richter2020categories}, showing it defines a full and faithful embedding of a category: 

\begin{theorem}
The {\bf {nerve functor}} $N_\bullet:$ {\bf {Cat}} $\rightarrow$ {\bf {Set}} is fully faithful. More specifically, there is a bijection $\theta$ defined as: 
\[ \theta: {\bf Cat}({\cal C}, {\cal C'}) \rightarrow {\bf Set}_\Delta (N_\bullet({\cal C}), N_\bullet({\cal C'}) \] 
\end{theorem}

Unfortunately, the left adjoint to the nerve functor is not a full and faithful encoding of a simplicial set back into a suitable category. Note that a functor $G$ from a simplicial object $X$ to a category ${\cal C}$ can be lossy. For example, we can define the objects of ${\cal C}$ to be the elements of $X_0$, and the morphisms of ${\cal C}$ as the elements $f \in X_1$, where $f: a \rightarrow b$, and $d_0 f = a$, and $d_1 f = b$, and $s_0 a, a \in X$ as defining the identity morphisms {${\bf 1}_a$}. Composition in this case can be defined as the free algebra defined over elements of $X_1$, subject to the constraints given by elements of $X_2$. For example, if $x \in X_2$, we can impose the requirement that $d_1 x = d_0 x \circ d_2 x$. Such a definition of the left adjoint would be quite lossy because it only preserves the structure of the simplicial object $X$ up to the $2$-simplices. The right adjoint from a category to its associated simplicial object, in contrast, constructs a full and faithful embedding of a category into a simplicial set.  In particular, the  nerve of a category is such a right adjoint.

\section{Layer 3 of GAIA: Universal Properties and the Category of Elements} 

\label{layer3}

A central and unifying principle in GAIA is that every pair of categorical layers is synchronized by a functor, along with a universal arrow.  In this section, we introduce some additional ideas from category theory, including the fundamental Yoneda Lemma \cite{maclane:71} that states that all objects in a (generative AI) category can be defined in terms of their interactions. To understand the significance of this powerful lemma, keep in mind that it applies to any (small) category. In particular, we defined in Section~\ref{layer2} the category of Transformer models as equivariant functions over Euclidean spaces. What the Yoneda Lemma implies here is that any Transformer model building block can be defined (upto isomorphism) purely in terms of the interactions it makes with other Transformer building blocks. This somewhat strange parameterization provides deep insight into how objects in categories behave. In a concrete sense, Transformer models are based on defining words by their context in sentences, and an enriched form of the  Yoneda Lemma can be directly applied to model the statistical representations of words learned by Transformers \cite{bradley:enriched-yoneda-llms}. 

The bottom layer of GAIA is a (Grothendieck) category of elements \cite{riehl2017category} that essentially ``grounds" out the universal coalgebras at layer 2 of GAIA in terms of the concrete data that was used to build the foundation models in GAIA. Intuitively, Layer 3 stores the ``training data" in a general relational structure. We first define how to construct the category of elements in a relational database.  In particular, at the simplicial top layer, generative AI operators such as  face and degeneracy operators define ``graph surgery" \cite{pearl-book} operations on generative AI models, or in terms of ``copy", ``delete" operators in ``string diagram surgery" defined on symmetric monoidal categories \cite{string-diagram-surgery}. These ``surgery" operations at the next level may translate down to operations on probability distributions, measurable spaces, topological spaces, or chain complexes. This process follows a standard construction used widely in mathematics, for example group representations associate with any group $G$, a left {\bf k}-module $M$ representation that enables modeling abstract group operations by operations on the associated modular representation.  These concrete representations must satisfy the universal arrow property for them to be faithful.  

\subsection{Natural Transformations and Universal Arrows}

Since we now can have multiple functors between the category {\tt Para} and the category {\tt Learn}, for example traditional backpropagation or a stochastic approximation ``zeroth-order" approximation, we can compare these two functors using natural transformations. Given any two functors $F: C \rightarrow D$ and $G: C \rightarrow D$ between the same pair of categories, we can define a mapping between $F$ and $G$ that is referred to as a natural transformation. These are defined through a collection of mappings, one for each object $c$ of $C$, thereby defining a morphism in $D$ for each object in $C$. 

\begin{definition}
    Given categories $C$ and $D$, and functors $F, G: C \rightarrow D$, a {\bf natural transformation} $\alpha: F \Rightarrow G$ is defined by the following data: 

    \begin{itemize}
        \item an arrow $\alpha_c: Fc \rightarrow Gc$ in $D$ for each object $c \in C$, which together define the components of the natural transformation. 
        \item For each morphism $f: c \rightarrow c'$, the following commutative diagram holds true: 

\[\begin{tikzcd}
	Fc &&& Gc \\
	\\
	{Fc'} &&& {Gc'}
	\arrow["{\alpha_c}", from=1-1, to=1-4]
	\arrow["Ff"', from=1-1, to=3-1]
	\arrow["{\alpha_{c'}}"', from=3-1, to=3-4]
	\arrow["Gf", from=1-4, to=3-4]
\end{tikzcd}\]

    \end{itemize}
    A {\bf natural isomorphism} is a natural transformation $\alpha: F \Rightarrow G$ in which every component $\alpha_c$ is an isomorphism. 
\end{definition}

A fundamental universal construction in category theory, called the {\em {universal arrow}} lies at the heart of many useful results, principally the Yoneda lemma that shows how object identity itself emerges from the structure of morphisms that lead into (or out of) it. 

\begin{definition}
Given a functor $S: D \rightarrow C$ between two categories, and an object $c$ of category $C$, a {\bf {universal arrow}} from $c$ to $S$ is a pair $\langle r, u \rangle$, where $r$ is an object of $D$ and $u: c \rightarrow Sr$ is an arrow of $C$, such that the following universal property holds true: 

\begin{itemize} 
\item For every pair $\langle d, f \rangle$ with $d$ an object of $D$ and $f: c \rightarrow Sd$ an arrow of $C$, there is a unique arrow $f': r \rightarrow d$ of $D$ with $S f' \circ u = f$. 
\end{itemize}
\end{definition}

\begin{definition}
If $D$ is a category and $H: D \rightarrow$ {\bf {Set}} is a set-valued functor, a {\bf {universal element}} associated with the functor $H$ is a pair $\langle r, e \rangle$ consisting of an object $r \in D$ and an element $e \in H r$ such that for every pair $\langle d, x \rangle$ with $x \in H d$, there is a unique arrow $f: r \rightarrow d$ of $D$ such \mbox{that $(H f) e = x$. }
\end{definition}

\begin{theorem}
Given any functor $S: D \rightarrow C$, the universal arrow $\langle r, u: c \rightarrow Sr \rangle$ implies a bijection exists between the {\bf {Hom}} sets 
\[ \mbox{{\bf Hom}}_{D}(r, d) \simeq \mbox{{\bf Hom}}_{C}(c, Sd) \]
\end{theorem}

A special case of this natural transformation that transforms the identity morphism {\bf {1}}$_r$ leads us to the Yoneda lemma.

\subsection{Yoneda Lemma}

The Yoneda Lemma states that the set of all morphisms into an object $d$ in a category ${\cal C}$, denoted as {\bf Hom}$_{\cal C}(-,d)$ and called the {\em contravariant functor} (or presheaf),  is sufficient to define $d$ up to isomorphism. The category of all presheaves forms a {\em category of functors}, and is denoted $\hat{{\cal C}} = $ {\bf Set}$^{{\cal C}^{op}}$.We will briefly describe two concrete applications of this lemma to two important areas in AI and ML in this section: reasoning about causality and reasoning about distances. The Yoneda lemma plays a crucial role in this paper because it defines the concept of a {\em universal representation} in category theory. We first show that associated with universal arrows is the corresponding induced isomorphisms between {\bf {Hom}} sets of morphisms in categories. This universal property then leads to the Yoneda lemma. 

\begin{center}
 \begin{tikzcd}
  D(r,r) \arrow{d}{D(r, f')} \arrow{r}{\phi_r}
    & C(c, Sr) \arrow[]{d}{C(c, S f')} \\
  D(r,d)  \arrow[]{r}[]{\phi_d}
&C(c, Sd)\end{tikzcd}
 \end{center} 

 As the two paths shown here must be equal in a commutative diagram, we get the property that a bijection between the {\bf {Hom}} sets holds precisely when $\langle r, u: c \rightarrow Sr \rangle$ is a universal arrow from $c$ to $S$. Note that for the case when the categories $C$ and $D$ are small, meaning their {\bf Hom} collection of arrows forms a set, the induced functor {\bf {Hom}}$_C(c, S - )$ to {\bf Set} is isomorphic to the functor {\bf {Hom}}$_D(r, -)$. This type of isomorphism defines a universal representation, and is at the heart of the causal reproducing property (CRP) defined below. 

\begin{lemma}
{\bf {Yoneda lemma}}: For any functor $F: C \rightarrow {\bf Set}$, whose domain category $C$ is ``locally small" (meaning that the collection of morphisms between each pair of objects forms a set), any object $c$ in $C$, there is a bijection 

\[ \mbox{Hom}(C(c, -), F) \simeq Fc \]

that defines a natural transformation $\alpha: C(c, -) \Rightarrow F$ to the element $\alpha_c(1_c) \in Fc$. This correspondence is natural in both $c$ and $F$. 
\end{lemma}

There is of course a dual form of the Yoneda Lemma in terms of the contravariant functor $C(-, c)$ as well using the natural transformation $C(-, c) \Rightarrow F$. A very useful way to interpret the Yoneda Lemma is through the notion of universal representability through a covariant or contravariant functor.

\begin{definition}
    A {\bf universal representation} of an object $c \in C$ in a category $C$ is defined as a contravariant functor $F$ together with a functorial representation $C(-, c) \simeq F$ or by a covariant functor $F$ together with a representation $C(c, -) \simeq F$. The collection of morphisms $C(-, c)$ into an object $c$ is called the {\bf presheaf}, and from the Yoneda Lemma, forms a universal representation of the object. 
\end{definition}

Later in this paper, we will see how the Yoneda Lemma gives us a novel perspective on how to construct universal representers in non-symmetric generalized metric spaces that are essential to defining ``nonsymmetric attention" in large language models.

A key distinguishing feature of category theory is the use of diagrammatic reasoning. However, diagrams are also viewed more abstractly as functors mapping from some indexing category to the actual category. Diagrams are useful in understanding universal constructions, such as limits and colimits of diagrams. To make this somewhat abstract definition concrete, let us look at some simpler examples of universal properties, including co-products and quotients (which in set theory correspond to disjoint unions). Coproducts refer to the universal property of abstracting a group of elements into a larger one.

 Before we formally the concept of limit and colimits, we consider some examples.  These notions generalize the more familiar notions of Cartesian products and disjoint unions in the category of {\bf {Sets}}, the notion of meets and joins in the category {\bf {Preord}} of preorders, as well as the  least upper bounds and greatest lower bounds in lattices, and many other concrete examples from mathematics. 

\begin{example} 
If  we consider a small  ``discrete'' category ${\cal D}$ whose only morphisms are identity arrows, then the colimit of a functor ${\cal F}: {\cal D} \rightarrow {\cal C}$ is the {\em categorical coproduct} of ${\cal F}(D)$ for $D$, an object of category {\cal D}, is denoted as 
\[ \mbox{Colimit}_{\cal D} F = \bigsqcup_D {\cal F}(D) \]

In the special case when the category {\cal C} is the category {\bf {Sets}}, then the colimit of this functor is simply the disjoint union of all the sets $F(D)$ that are mapped from objects $D \in {\cal D}$. 
\end{example} 

\begin{example} 
Dual to the notion of colimit of a functor is the notion of {\em limit}. Once again, if we consider a small  ``discrete'' category ${\cal D}$ whose only morphisms are identity arrows, then the limit of a functor ${\cal F}: {\cal D} \rightarrow {\cal C}$ is the {\em categorical product} of ${\cal F}(D)$ for $D$, an object of category {\cal D}, is denoted as 
\[ \mbox{limit}_{\cal D} F = \prod_D {\cal F}(D) \]

In the special case when the category {\cal C} is the category {\bf {Sets}}, then the limit of this functor is simply the Cartesian product of all the sets $F(D)$ that are mapped from objects $D \in {\cal D}$. 
\end{example} 

Category theory relies extensively on {\em universal constructions}, which satisfy a universal property. One of the central building blocks is the identification of universal properties through formal diagrams.  Before introducing these definitions in their most abstract form, it greatly helps to see some simple examples. 

 We can illustrate the limits and colimits in diagrams using pullback and pushforward mappings.

\begin{tikzcd}
    & Z\arrow[r, "p"] \arrow[d, "q"]
      & X \arrow[d, "f"] \arrow[ddr, bend left, "h"]\\
& Y \arrow[r, "g"] \arrow[drr, bend right, "i"] &X \sqcup Y \arrow[dr, "r"]  \\ 
& & & R 
\end{tikzcd}

An example of a universal construction is given by the above commutative diagram, where the coproduct object $X \sqcup Y$ uniquely factorizes any mapping $h: X \rightarrow R$, such that any mapping $i: Y \rightarrow R$, so that $h = r \circ f$, and furthermore $i = r \circ g$. Co-products are themselves special cases of the more general notion of co-limits. Figure~\ref{univpr}  illustrates the fundamental property of a {\em {pullback}}, which along with {\em pushforward}, is one of the core ideas in category theory. The pullback square with the objects $U,X, Y$ and $Z$ implies that the composite mappings $g \circ f'$ must equal $g' \circ f$. In this example, the morphisms $f$ and $g$ represent a {\em {pullback}} pair, as they share a common co-domain $Z$. The pair of morphisms $f', g'$ emanating from $U$ define a {\em {cone}}, because the pullback square ``commutes'' appropriately. Thus, the pullback of the pair of morphisms $f, g$ with the common co-domain $Z$ is the pair of morphisms $f', g'$ with common domain $U$. Furthermore, to satisfy the universal property, given another pair of morphisms $x, y$ with common domain $T$, there must exist another morphism $k: T \rightarrow U$ that ``factorizes'' $x, y$ appropriately, so that the composite morphisms $f' \ k = y$ and $g' \ k = x$. Here, $T$ and $U$ are referred to as {\em cones}, where $U$ is the limit of the set of all cones ``above'' $Z$. If we reverse arrow directions appropriately, we get the corresponding notion of pushforward. So, in this example, the pair of morphisms $f', g'$ that share a common domain represent a pushforward pair. 
As Figure~\ref{univpr}, for any set-valued functor $\delta: S \rightarrow$ {\bf {Sets}}, the Grothendieck category of elements $\int \delta$ can be shown to be a pullback in the diagram of categories. Here, {${\bf Set}_*$} is the category of pointed sets, and $\pi$ is a projection that sends a pointed set $(X, x \in X)$ to its \mbox{underlying set $X$.}

\begin{figure}[h]
\hspace{21pt}


\centering
\begin{tikzcd}
  T
  \arrow[drr, bend left, "x"]
  \arrow[ddr, bend right, "y"]
  \arrow[dr, dotted, "k" description] & & \\
    & 
    U\arrow[r, "g'"] \arrow[d, "f'"]
      & X \arrow[d, "f"] \\
& Y \arrow[r, "g"] &Z
\end{tikzcd}
\begin{tikzcd}
  T
  \arrow[drr, bend left, "x"]
  \arrow[ddr, bend right, "y"]
  \arrow[dr, dotted, "k" description] & & \\
    & 
    \int \delta \arrow[r, "\delta'"] \arrow[d, "\pi_\delta"]
      & {\bf Set}_* \arrow[d, "\pi"] \\
& S \arrow[r, "\delta"] & {\bf Set}
\end{tikzcd}
\caption{(\textbf{Left})
 Universal Property of pullback mappings. (\textbf{Right}) The Grothendieck category of elements $\int \delta$ of any set-valued functor $\delta: S \rightarrow$ {\bf {Set}} can be described as a pullback in the diagram of categories. Here, {\bf Set}$_*$ is the category of pointed sets $(X, x \in X)$, and $\pi$ is the ``forgetful" functor that sends a pointed set $(X, x \in X)$ into the underlying set $X$.  } 
\label{univpr}
\end{figure}
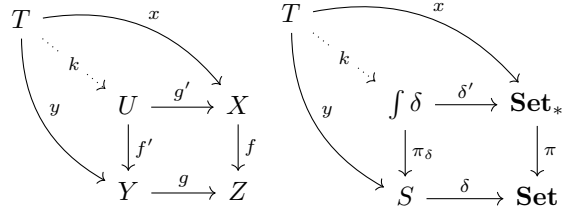

We can now proceed to define limits and colimits more generally. We define a {\em diagram} $F$ of {\em shape} $J$ in a category $C$ formally as a functor $F: J \rightarrow C$. We want to define the somewhat abstract concepts of {\em limits} and {\em colimits}, which will play a central role in this paper in identifying properties of AI and ML techniques.  A convenient way to introduce these concepts is through the use of {\em universal cones} that are {\em over} and {\em under} a diagram.

For any object $c \in C$ and any category $J$, the {\em constant functor} $c: J \rightarrow C$ maps every object $j$ of $J$ to $c$ and every morphism $f$ in $J$ to the identity morphisms $1_c$. We can define a constant functor embedding as the collection of constant functors $\Delta: C \rightarrow C^J$ that send each object $c$ in $C$ to the constant functor at $c$ and each morphism $f: c \rightarrow c'$ to the constant natural transformation, that is, the natural transformation whose every component is defined to be the morphism $f$. 

\begin{definition}
    A {\bf cone over} a diagram $F: J \rightarrow C$ with the {\bf summit} or {\bf apex} $c \in C$ is a natural transformation $\lambda: c \Rightarrow F$ whose domain is the constant functor at $c$. The components $(\lambda_j: c \rightarrow Fj)_{j \in J}$ of the natural transformation can be viewed as its {\bf legs}. Dually, a {\bf cone under} $F$ with {\bf nadir} $c$ is a natural transformation $\lambda: F \Rightarrow c$ whose legs are the components $(\lambda_j: F_j \rightarrow c)_{j \in J}$.

\[\begin{tikzcd}
	&& c &&&& Fj &&&& Fk \\
	\\
	{F j} &&&& Fk &&&& c
	\arrow["{\lambda_j}", from=1-3, to=3-1]
	\arrow["{\lambda_k}"', from=1-3, to=3-5]
	\arrow["{F f}", from=3-1, to=3-5]
	\arrow["Ff", from=1-7, to=1-11]
	\arrow["{\lambda_j}"', from=1-7, to=3-9]
	\arrow["{\lambda_k}", from=1-11, to=3-9]
\end{tikzcd}\]
    
\end{definition}

Cones under a diagram are referred to usually as {\em cocones}. Using the concept of cones and cocones, we can now formally define the concept of limits and colimits more precisely. 

\begin{definition}
    For any diagram $F: J \rightarrow C$, there is a functor 

    \[ \mbox{Cone}(-, F): C^{op} \rightarrow \mbox{{\bf Set}} \]

    which sends $c \in C$ to the set of cones over $F$ with apex $c$. Using the Yoneda Lemma, a {\bf limit} of $F$ is defined as an object $\lim F \in C$ together with a natural transformation $\lambda: \lim F \rightarrow F$, which can be called the {\bf universal cone} defining the natural isomorphism 

    \[ C(-, \lim F) \simeq \mbox{Cone}(-, F) \]

    Dually, for colimits, we can define a functor 

    \[ \mbox{Cone}(F, -): C \rightarrow \mbox{{\bf Set}} \]

    that maps object $c \in C$ to the set of cones under $F$ with nadir $c$. A {\bf colimit} of $F$ is a representation for $\mbox{Cone}(F, -)$. Once again, using the Yoneda Lemma, a colimit is defined by an object $\mbox{Colim} F \in C$ together with a natural transformation $\lambda: F \rightarrow \mbox{colim} F$, which defines the {\bf colimit cone} as the natural isomorphism 

    \[ C(\mbox{colim} F, -) \simeq \mbox{Cone}(F, -) \]
\end{definition}

Limit and colimits of diagrams over arbitrary categories can often be reduced to the case of their corresponding diagram properties over sets. One important stepping stone is to understand how functors interact with limits and colimits. 

\begin{definition}
    For any class of diagrams $K: J \rightarrow C$, a functor $F: C \rightarrow D$ 

    \begin{itemize}
        \item {\bf preserves} limits if for any diagram $K: J \rightarrow C$ and limit cone over $K$, the image of the cone defines a limit cone over the composite diagram $F K: J \rightarrow D$. 

        \item {\bf reflects} limits if for any cone over a diagram $K: J \rightarrow C$ whose image upon applying $F$ is a limit cone for the diagram $F K: J \rightarrow D$ is a limit cone over $K$

        \item {\bf creates} limits if whenever $FK : J \rightarrow D$ has a limit in $D$, there is some limit cone over $F K$ that can be lifted to a limit cone over $K$ and moreoever $F$ reflects the limits in the class of diagrams. 
    \end{itemize}
\end{definition}

To interpret these abstract definitions, it helps to concretize them in terms of a specific universal construction, like the pullback defined above $c' \rightarrow c \leftarrow c''$ in $C$. Specifically, for pullbacks: 

\begin{itemize} 

\item A functor $F$ {\bf preserves pullbacks} if whenever $p$ is the pullback of  $c' \rightarrow c \leftarrow c''$ in $C$, it follows that $Fp$ is the pullback of  $Fc' \rightarrow Fc \leftarrow Fc''$ in $D$.

\item A functor $F$ {\bf reflects  pullbacks}  if  $p$ is the pullback of  $c' \rightarrow c \leftarrow c''$ in $C$ whenever $Fp$ is the pullback of  $Fc' \rightarrow Fc \leftarrow Fc''$ in $D$.

\item A functor $F$ {\bf creates pullbacks} if there exists some $p$ that is the pullback of  $c' \rightarrow c \leftarrow c''$ in $C$ whenever there exists a $d$ such  that $d$ is the pullback of  $Fc' \rightarrow Fc \leftarrow Fc''$ in $F$.

\end{itemize} 

\subsection*{Universality of Diagrams}

In the category {\bf {Sets}}, we know that every object (i.e., a set) $X$ can be expressed as a coproduct (i.e., disjoint union)  of its elements $X \simeq \sqcup_{x \in X} \{ x \}$, where $x \in X$. Note that we can view each element $x \in X$ as a morphism $x: \{ * \} \rightarrow X$ from the one-point set to $X$. The categorical generalization of this result is called the {\em {density theorem}} in the theory of sheaves. First, we define the key concept of a {\em comma category}. 

\begin{definition}
Let $F: {\cal D} \rightarrow {\cal C}$ be a functor from category ${\cal D}$ to ${\cal C}$. The {\bf {comma category}} $F \downarrow {\cal C}$ is one whose objects are pairs $(D, f)$, where $D \in {\cal D}$ is an object of ${\cal D}$ and $f \in$ {\bf {Hom}}$_{\cal C}(F(D), C)$, where $C$ is an object of ${\cal C}$. Morphisms in the comma category $F \downarrow {\cal C}$ from $(D, f)$ to $(D', f')$, where $g: D \rightarrow D'$, such that $f' \circ F(g) = f$. We can depict this structure through the following commutative diagram: 
\begin{center} 
\begin{tikzcd}[column sep=small]
& F(D) \arrow{dl}[near start]{F(g)} \arrow{dr}{f} & \\
  F(D')\arrow{rr}{f'}&                         & C
\end{tikzcd}
\end{center} 
\end{definition} 

We first introduce the concept of a {\em {dense}} functor: 

\begin{definition}
Let {\cal D} be a small category, {\cal C} be an arbitrary category, and $F: {\cal D} \rightarrow {\cal D}$ be a functor. The functor $F$ is {\bf {dense}} if for all objects $C$ of ${\cal C}$, the natural transformation 
\[ \psi^C_F: F \circ U \rightarrow \Delta_C, \ \ (\psi^C_F)_{({\cal D}, f)} = f\]
is universal in the sense that it induces an isomorphism $\mbox{Colimit}_{F \downarrow C} F \circ U \simeq C$. Here, $U: F \downarrow C \rightarrow {\cal D}$ is the projection functor from the comma category $F \downarrow {\cal C}$, defined by $U(D, f) = D$. 

\end{definition} 

A fundamental consequence of the category of elements is that every object in the functor category of presheaves, namely contravariant functors from a category into the category of sets, is the colimit of a diagram of representable objects, via the Yoneda lemma. Notice this is a generalized form of the density notion from the category {\bf {Sets}}.

\begin{theorem}
\label{presheaf-theorem}
{\bf {Universality of Diagrams}}: In the functor category of presheaves {\bf {Set}}$^{{\cal C}^{op}}$, every object $P$ is the colimit of a diagram of representable objects, in a canonical way. 
\end{theorem}

\subsection{Universal Arrows and Elements}

We explore the universal arrow property more deeply in this section, showing how it provides the conceptual basis behind  the (metric) Yoneda Lemma, and Grothendieck's category of elements. 

A special case of the universal arrow property is that of universal element, which as we will see below plays an important role in the GAIA architecture in defining a suitably augmented category of elements, based on a construction introduced by Grothendieck. 

\begin{definition}
If $D$ is a category and $H: D \rightarrow {\bf Set}$ is a set-valued functor, a {\bf universal element} associated with the functor $H$ is a pair $\langle r, e \rangle$ consisting of an object $r \in D$ and an element $e \in H r$ such that for every pair $\langle d, x \rangle$ with $x \in H d$, there is a unique arrow $f: r \rightarrow d$ of $D$ such that $(H f) e = x$. 
\end{definition}

\begin{example}
Let $E$ be an equivalence relation on a set $S$, and consider the quotient set $S/E$ of equivalence classes, where $p: S \rightarrow S/E$ sends each element $s \in S$ into its corresponding equivalence class. The set of equivalence classes $S/E$ has the property that any function $f: S \rightarrow X$ that respects the equivalence relation can be written as $f s = f s'$ whenever $s \sim_E s'$, that is, $f = f' \circ p$, where the unique function $f': S/E \rightarrow X$. Thus, $\langle S/E, p \rangle$ is a universal element for the functor $H$. 
\end{example}

 \subsection{The Category of Elements}

We turn next to define the category of elements, based on a construction by Grothendieck, and illustrate how it can serve as the basis for inference at each layer of the UCLA architecture. In particular, \citet{SPIVAK_2013} shows how the category of elements can be used to define SQL queries in a relational database. 

\begin{definition}
Given a set-valued functor $\delta: {\cal C} \rightarrow {\bf Set}$ from some category ${\cal C}$, the induced {\bf category of elements} associated with $\delta$ is a pair $(\int \delta, \pi_\delta)$, where $\int \delta \in$ {\bf Cat} is a category in the category of all categories {\bf Cat}, and $\pi_\delta: \int \delta \rightarrow {\cal C}$ is a functor that ``projects" the category of elements into the corresponding original category ${\cal C}$. The objects and arrows of $\int \delta$ are defined as follows: 
\begin{itemize}
    \item $\mbox{Ob}(\int \delta) = \{ (s, x) | x \in \mbox{Ob}({\cal c}), x \in \delta s \} $. 

    \item {\bf Hom}$_{\int \delta}((s, x), (s', x')) = \{f: s \rightarrow s' | \delta f (x) = x' \}$
\end{itemize}
\end{definition}

\begin{example}
To illustrate the category of elements construction, let us consider the toy climate change DAG model shown in Figure~\ref{climate-change}. Let the category {\cal C} be defined by this DAG model, where the objects Ob({\cal C}) are defined by the four vertices, and the arrows {\bf Hom}$_{\cal C}$ are defined by the four edges in the model. The set-valued functor $\delta: {\cal C} \rightarrow {\bf Set}$ maps each object (vertex) in {\cal C} to a set of instances, thereby turning the causal DAG model into an associated set of tables. For example, {\bf Climate Change} is defined as a table of values, which could be modeled as a multinomial variable taking on a set of discrete values, and for each of its values, the arrow from {\bf Climate Change} to {\bf Rainfall} maps each specific value of {\bf Climate Change} to a value of {\bf Rainfall}, thereby indicating a causal effect of climate change on the amount of rainfall in California. Im the figure, {\bf Climate Change} is mapped to three discrete levels (marked $1$, $2$ and $3$). Rainfall amounts are discretized as well into low (marked "L"), medium (marked "M"), high (marked "H"), or extreme (marked "E"). Wind speeds are binned into two levels (marked "W" for weak, and "S" for strong). Finally, the percentage of California wildfires is binned between $5$ to $30$. Not all arrows that exist in the Grothendieck category of elements are shown, for clarity. 

 \end{example}
 \begin{figure}[h] 
 \caption{A toy DAG model of climate change to illustrate the category of elements construction.  \label{climate-change}}
 \vskip 0.1in
\centering
\begin{minipage}{0.7\textwidth}
\includegraphics[scale=0.3]{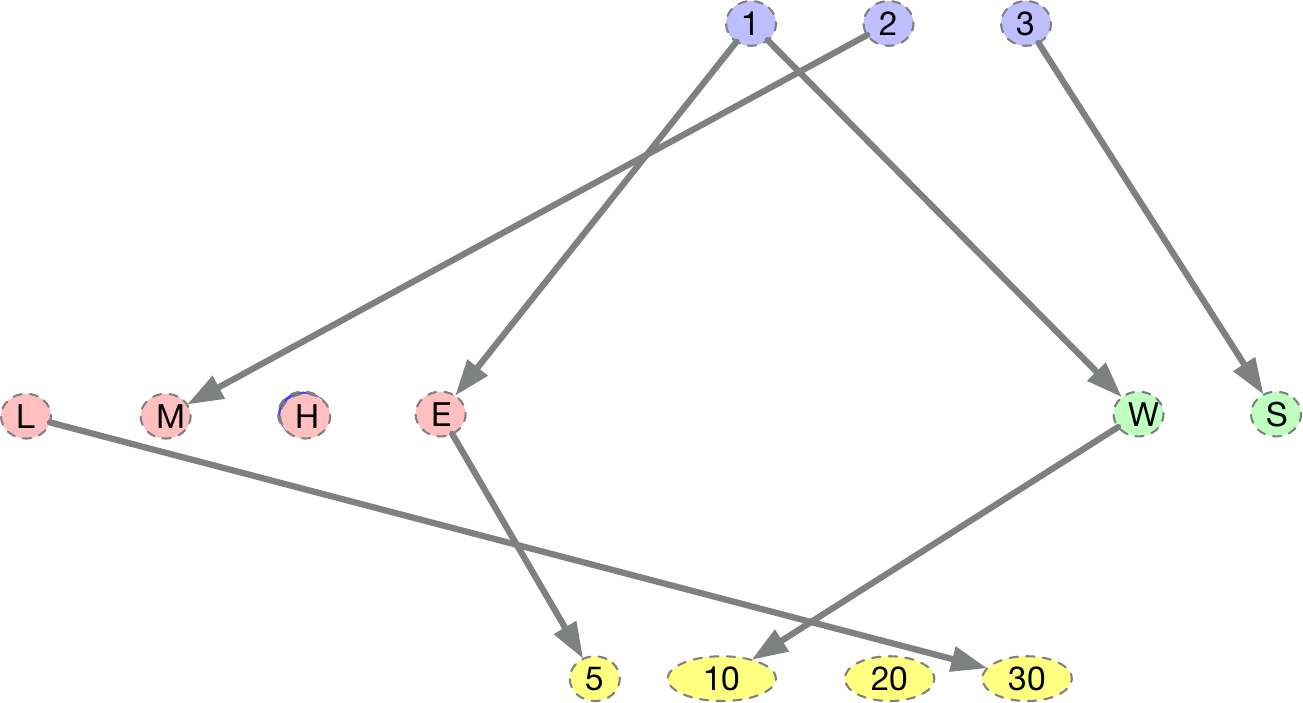}
\end{minipage} 
 \end{figure} 

 Many properties of Grothendieck's construction can be exploited (some of these are discussed in the context of relational database queries in \cite{SPIVAK_2013}), but for our application, we are primarily interested in the associated class of lifting problems that define queries in a generative AI model.

 \subsection{Lifting Problems in Generative AI}
 
\begin{definition}
 If $S$ is a collection of morphisms in category  ${\cal C}$, a morphism $f: A \rightarrow B$ has the {\bf left lifting property with respect to S} if it has the left lifting property with respect to every morphism in $S$. Analogously, we say a morphism $p: X \rightarrow Y$ has the {\bf right lifting property with respect to S} if it has the right lifting property with respect to every morphism in $S$.
 \end{definition}

 Many properties of Grothendieck's construction can be exploited (some of these are discussed in the context of relational database queries in \cite{SPIVAK_2013}), but for our application to generative AI, we are primarily interested in the associated class of lifting problems that can be used to define queries and build foundation models.  
 
\begin{definition}
 If $S$ is a collection of morphisms in category  ${\cal C}$, a morphism $f: A \rightarrow B$ has the {\bf {left lifting property with respect to S}} if it has the left lifting property with respect to every morphism in $S$. Analogously, we say a morphism $p: X \rightarrow Y$ has the {\bf {right lifting property with respect to S}} if it has the right lifting property with respect to every morphism in $S$.
 \end{definition}

We now turn to sketch some examples of the application of lifting problems for generative AI. Many problems in causal inference on graphs involve some particular graph property.  To formulate it as a lifting problem, we will use the following generic template, following the initial application of lifting problems to database queries proposed by \citet{SPIVAK_2013}. 

\begin{center}
 \begin{tikzcd}
  Q \arrow{d}{f} \arrow{r}{\mu}
    & \int \delta  \arrow[]{d}{p} \\
  R \arrow[ur,dashed, "h"] \arrow[]{r}[]{\nu}
&{\cal C} \end{tikzcd}
 \end{center} 

 Here, $Q$ is a generic query that we want answered, which could range from a database query, as in the original setting studied by \citet{SPIVAK_2013}, but more interestingly, it could be a particular graph property relating to generative AI. By suitably modifying the base category, the lifting problem formulation can be used to encode a diverse variety of problems in generative AI inference. $R$ represents a fragment of the complete generative AI model ${\cal C}$, and $\delta$ is the category of elements defined above. Finally, $h$ gives all solutions to the lifting problem. 
 
\begin{example}
 \textls[-15]{Consider the category of directed graphs defined by the category ${\cal G}$, where \mbox{Ob(${\cal G}$) = \{V, E\},}} and the morphisms of ${\cal G}$ are given as {\bf {Hom}}$_{\cal G}$ = \{s, t\}, where $s: E \rightarrow V$ and $t: E \rightarrow V$ define the source and terminal nodes of each vertex. Then, the category of all directed graphs is precisely defined by the category of all functors $\delta: {\cal G} \rightarrow$ {\bf {Set}}. Any particular graph is defined by the functor $X: {\cal G} \rightarrow$ {\bf {Set}}, where the function $X(s): X(E) \rightarrow X(V)$ assigns to every edge its source vertex. For causal inference, we may want to check some property of a graph, such as the property that every vertex in $X$ is the source of some edge. The following lifting problem ensures that every vertex has a source edge in the graph. The category of elements $\int \delta$ shown below refers to a construction introduced by Grothendieck, which will be defined in more detail later.
 
 \begin{center}
 \begin{tikzcd}
  V (\bullet) \arrow{d}{f} \arrow{r}{\mu}
    & \int \delta  \arrow[]{d}{p} \\
  \{ E (\bullet) \xrightarrow[]{s} V (\bullet) \} \arrow[ur,dashed, "h"] \arrow[]{r}[]{\nu}
&{\cal G} \end{tikzcd}
 \end{center} 

 \end{example}

\begin{example} 

As another example of the application of lifting problems to causal inference, let us consider the problem of determining whether two causal DAGs, $G_1$ and $G_2$ are Markov \mbox{equivalent \cite{anderson-annals}.} A key requirement here is that the immoralities of $G_1$ and $G_2$ must be the same, that is, if $G_1$ has a collider $A \rightarrow B \leftarrow C$, where there is no edge between $A$ and $C$, then $G_2$ must also have the same collider, and none others. We can formulate the problem of finding colliders as the following lifting problem. Note that the three vertices $A$, $B$ and $C$ are bound to an actual graph instance through the category of elements $\int \delta$ (as was illustrated above), using the top right morphism $\mu$. The bottom left morphism $f$ binds these three vertices to some collider. The bottom right morphism $\nu$ requires this collider to exist in the causal graph ${\cal G}$ with the same bindings as found by $\mu$. The dashed morphisms $h$ finds all solutions to this lifting problem, that is, all colliders involving the vertices $A$, $B$ and $C$. 
 \begin{center}
 \begin{tikzcd}
  \{A (\bullet), B (\bullet), C (\bullet) \} \arrow{d}{f} \arrow{r}{\mu}
    & \int \delta  \arrow[]{d}{p} \\
  \{ A (\bullet) \rightarrow B (\bullet) \leftarrow  C (\bullet)\} \arrow[ur,dashed, "h"] \arrow[]{r}[blue]{\nu}
&{\cal G} \end{tikzcd}
 \end{center} 

 \end{example}

If the category of elements is defined by a functor mapping a database schema into a table of instances, then the associated lifting problem corresponds to familiar problems like SQL queries in relational databases \cite{SPIVAK_2013}. In our application, we can use the same machinery to formulate causal inference queries by choosing the categories appropriately.  To complete the discussion, we now make the connection between universal arrows and the core notion of universal representations via the Yoneda Lemma.

\subsection{Kan Extension}

 It is well known in category theory that ultimately every concept, from products and co-products, limits and co-limits, and ultimately even the Yoneda Lemma (see below), can be derived as special cases of the Kan extension \citep{maclane:71}. Kan extensions intuitively are a way to approximate a functor ${\cal F}$ so that its domain can be extended from a category ${\cal C}$ to another category  ${\cal D}$.  Because it may be impossible to make commutativity work in general, Kan extensions rely on natural transformations to make the extension be the best possible approximation to ${\cal F}$ along ${\cal K}$.  We want to briefly show Kan extensions can be combined with the category of elements defined above to construct ``migration functors'' that map from one generative AI model into another. These migration functors were originally defined in the context of database migration \cite{SPIVAK_2013}, but can also be applied to  generative AI inference. By suitably modifying the category of elements from a set-valued functor $\delta: {\cal C} \rightarrow$ {\bf {Set}}, to some other category, such as the category of topological spaces, namely $\delta: {\cal C} \rightarrow$ {\bf {Top}}, we can extend the migration functors into solving more abstract generative AI inference questions.  Here, for simplicity, we restrict our focus to Kan extensions for migration functors over the category of elements defined over instances of a generative AI model. 

\begin{definition}
A {\bf {left Kan extension}} of a functor $F: {\cal C} \rightarrow {\cal E}$ along another functor $K: {\cal C} \rightarrow {\cal D}$, is a functor $\mbox{Lan}_K F: {\cal D} \rightarrow {\cal E}$ with a natural transformation $\eta: F \rightarrow \mbox{Lan}_F \circ K$ such that for any other such pair $(G: {\cal D} \rightarrow {\cal E}, \gamma: F \rightarrow G K)$, $\gamma$ factors uniquely through $\eta$. In other words, there is a unique natural transformation $\alpha: \mbox{Lan}_F \implies G$. \\
%
\begin{center}
\begin{tikzcd}[row sep=2cm, column sep=2cm]
\mathcal{C}  \ar[dr, "K"', ""{name=K}]
            \ar[rr, "F", ""{name=F, below, near start, bend right}]&&
\mathcal{E}\\
& \mathcal{D}  \ar[ur, bend left, "\text{Lan}_KF", ""{name=Lan, below}]
                \ar[ur, bend right, "G"', ""{name=G}]
                
%
\arrow[Rightarrow, "\exists!", from=Lan, to=G]
\arrow[Rightarrow, from=F, to=K, "\eta"]
\end{tikzcd}
\end{center}
\end{definition}

A {\bf {right Kan extension}} can be defined similarly.  To understand the significance of Kan extensions for causal inference, we note that under a causal intervention, when a causal category $S$ gets modified to $T$, evaluating the modified generative AI model over a database of instances can be viewed as an example of Kan extension.

 Let $\delta: S \rightarrow$ {\bf {Set}} denote the original generative AI model defined by the category $S$ with respect to some dataset. Let $\epsilon: T \rightarrow$ {\bf {Set}} denote the effect of some change in the category $S$ to $T$, such as deletion of a morphism, as illustrated in Figure~\ref{kan-causal-intervention}. Intuitively, we can consider three cases: the {\em {pullback}} $\Delta_F$ along $F$, which maps the effect of a deletion back to the original model, the {\em {left pushforward}} $\Sigma_F$ and the {\em {right pushforward}} $\prod_F$, which can be seen as adjoints to the pullback $\Delta_F$.

\begin{figure}[h] 
\includegraphics[scale=0.3]{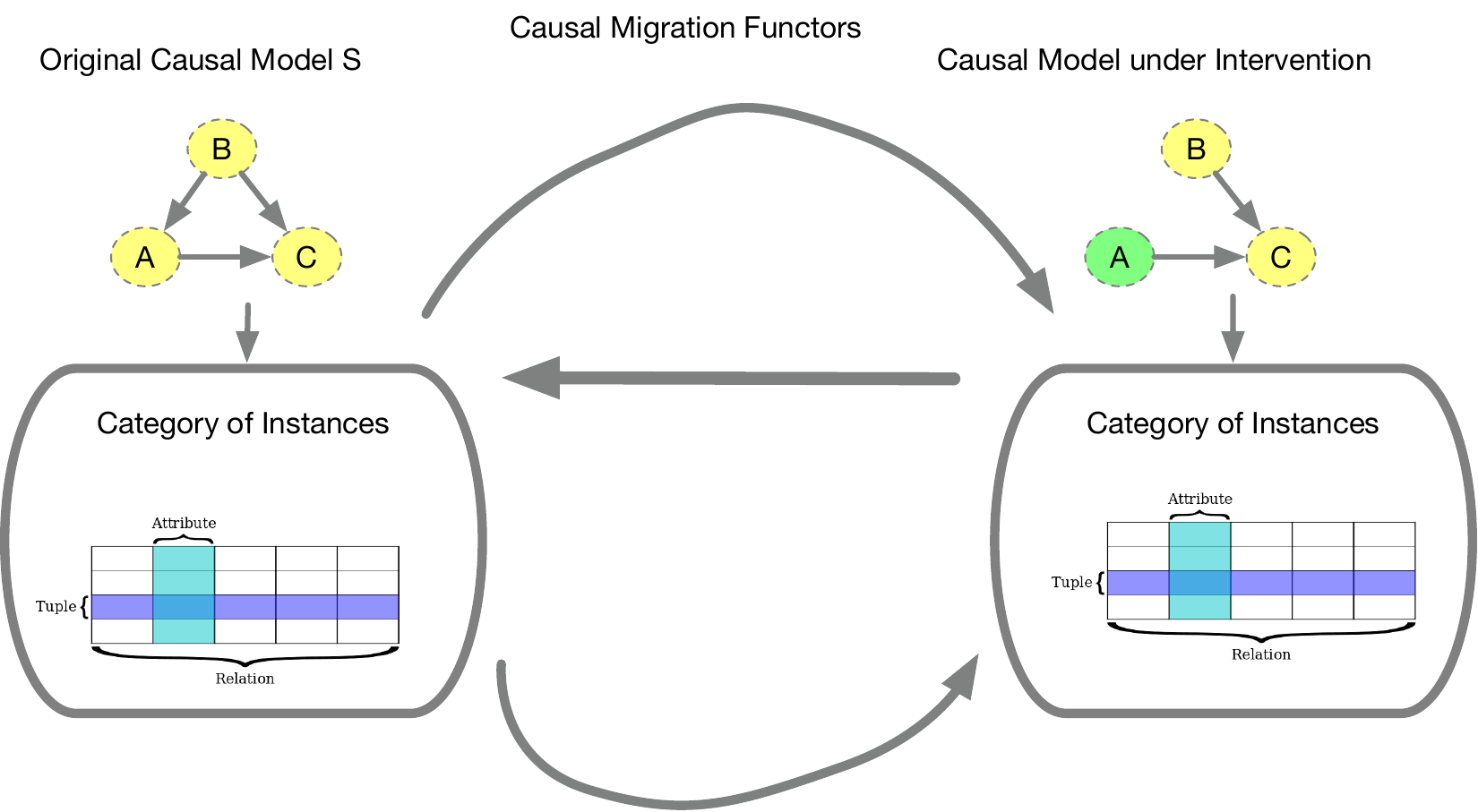}
\caption{Kan  extensions are useful in modeling the effects of modifications of generative AI models, such as deletion of morphisms, where in this toy example of a model over three objects $A, B$, and $C$, the object $A$ is intervened upon, eliminating the morphism into it from object $B$.   \label{kan-causal-intervention}}
 \end{figure} 

Following \cite{SPIVAK_2013}, we can define three {\em  migration functors} that evaluate the impact of a modification of a generative AI model with respect to a dataset of instances.

\begin{enumerate} 
\item The functor $\Delta_F: \epsilon \rightarrow \delta$ sends the functor $\epsilon: T \rightarrow$ {\bf {Set}} to the composed functor $\delta \circ F: S \rightarrow$ {\bf {Set}}. 

\item The functor $\Sigma_F: \delta \rightarrow \epsilon$ is the left Kan extension along $F$, and can be seen as the left adjoint to $\Delta_F$. 

The functor $\prod_F: \delta \rightarrow \epsilon$ is the right Kan extension along $F$, and can be seen as the right adjoint to $\Delta_F$. 
\end{enumerate}

To understand how to implement these functors, we use the following proposition that is stated in \cite{SPIVAK_2013} in the context of database queries, which we are restating in the setting of generative AI. 

\begin{theorem}
Let $F: S \rightarrow T$ be a functor. Let $\delta: S \rightarrow$ {\bf {Set}} and $\epsilon: T \rightarrow$ {\bf {Set}} be two set-valued functors, which can be viewed as two instances of a generative AI model defined by the category $S$ and $T$. If we view $T$ as the generative AI category that results from a modification caused by some modification on $S$ (e.g., deletion of an edge), then there is a commutative diagram linking the category of elements between $S$ and $T$. 
\begin{center}
 \begin{tikzcd}
  \int \delta \arrow{d}{\pi_\delta} \arrow{r}{}
    & \int \epsilon \arrow[]{d}{\pi_\epsilon} \\
  S \arrow[]{r}[]{F}
&T \end{tikzcd}
 \end{center} 

\end{theorem}
\begin{proof}
To check that the above diagram is a pullback, that is, $\int \delta \simeq S \times_T \int \delta$, or in words, the fiber product, we can check the existence of the pullback component wise by comparing the set of objects and the set of morphisms in $\int \delta$ with the respective sets in $S \times_T \int \epsilon$.  
\end{proof}
For simplicity, we defined the migration functors above with respect to an actual dataset of instances. More generally, we can compose the set-valued functor $\delta: S \rightarrow {\bf Set}$ with a functor ${\cal T}:$ {\bf {Set}} $\rightarrow$ {\bf {Top}} to the category of topological spaces to derive a Kan extension formulation of the definition of an intervention. We discuss this issue in Section~\ref{homotopy} on  homotopy in generative AI. 

\subsection{The Metric Yoneda Lemma} 

One disadvantage of current generative AI systems, such as large language models, is that they are based a symmetric model of distances. The Yoneda Lemma \cite{maclane:71},  one of the most celebrated results in category theory, can be used to build universal representers in non-symmetric generalized metric spaces, leading to a metric Yoneda Lemma \cite{BONSANGUE19981}. Stated in simple terms, the Yoneda Lemma states the mathematical objects are determined (up to isomorphism) by the interactions they make with other objects in a category. We will show the surprising results of applying this lemma to problems involving computing distances between objects in a metric space. 

A  general principle in machine learning (see Figure~\ref{metric-space}) to discriminate two objects (e.g., probability distributions, images, text documents etc.) is to compare them in a suitable metric space. We now describe a category of generalized metric spaces, where a metric form of the Yoneda Lemma gives us surprising insight.  Often, in category theory, we want to work in an enriched category.  One of the most interesting ways to design categories for applications in AI and ML is to look to augment the basic structure of a category with additional properties. For example, the collection of morphisms from an object $x$ to an object $y$ in a category ${\cal C}$ often has additional structure, besides just being a set. Often, it satisfies additional properties, such as forming a space of some kind such as a vector space or a topological space. We can think of such categories as {\em enriched} categories that exploit some desirable properties. We will illustrate one such example of primary importance to applications in AI and ML that involve measuring the distance between two objects. A distance function is assumed to return some non-negative value between $0$ and $\infty$, and we will view distances as defining enriched $[0, \infty]$ categories. We summarize some results here from \cite{BONSANGUE19981}. 

\begin{figure}[h]
\centering
\begin{minipage}{0.9\textwidth}
\includegraphics[scale=0.8]{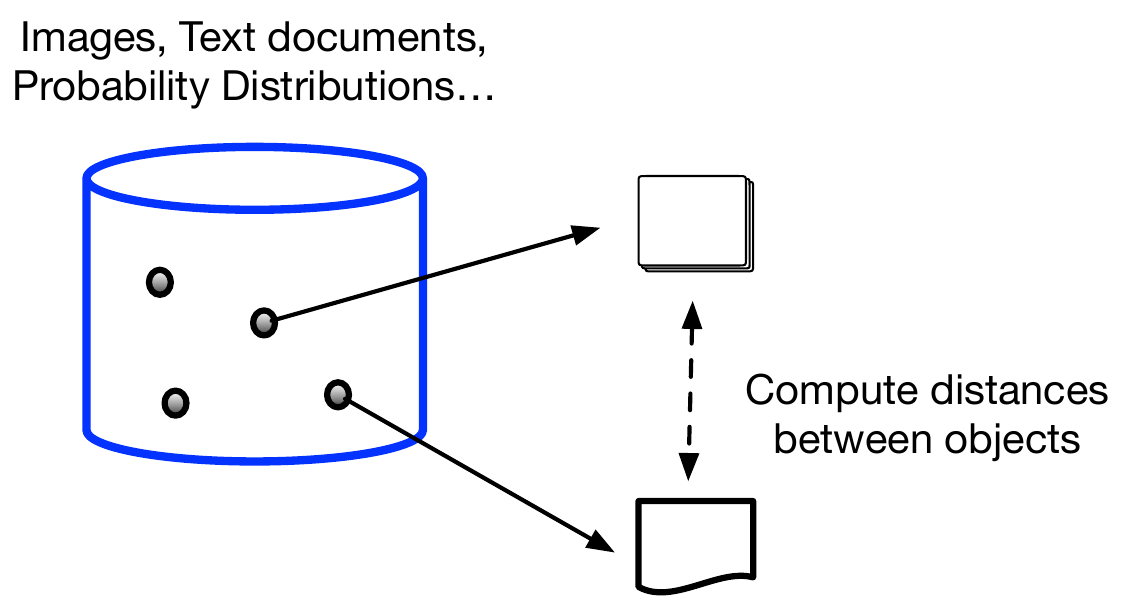}
\end{minipage}
\caption{Many algorithms in AI and ML involve computing distances between objects in a {\em metric space}. Interpreting distances categorically leads to powerful ways to reason about generalized metric spaces.}   
\label{metric-space}
\end{figure}

Figure~\ref{metric-space} illustrates a common motif among many AI and ML algorithms: define a problem in terms of computing distances between a group of objects. Examples of objects include points in $n$-dimensional Euclidean space, probability distributions, text documents represented as strings of tokens, and images represented as matrices. More abstractly, a {\em generalized metric space} $(X, d)$ is a set $X$ of objects, and a non-negative function $X(-,-): X \times X \rightarrow [0, \infty]$ that satisfies the following properties: 

\begin{enumerate}
    \item $X(x,x) = 0$: distance between the same object and itself is $0$. 

    \item $X(x,z) \leq X(x,y) + X(y,z)$: the famous {\em triangle inequality} posits that the distance between two objects cannot exceed the sum of distances between each of them and some other intermediate third object.
\end{enumerate}

In particular, generalized metric spaces are not required to be {\em symmetric}, or satisfy the property that if the distance between two objects $x$ and $y$ is $0$ implies $x$ must be identical to $y$, or finally that distances must be finite. These additional three properties listed below are what defines the usual notion of a {\em metric} space: 

\begin{enumerate}
    \item If $X(x, y) = 0$ and $X(y, x) = 0$ then $x = y$. 

\item $X(x, y) = X(y, x)$. 

\item $X(x, y) < \infty$. 
\end{enumerate}

In fact, we can subsume the previous discussion of causal inference under the notion of generalized metric spaces by defining a category around {\em preorders} $(P, \leq)$, which are relations that are reflexive and transitive, but not symmetric. Causal inference fundamentally involves constructing a preorder over the set of variables in a domain. Here are some examples of generalized metric spaces: 

\begin{enumerate}
    \item Any preorder $(P, \leq)$ such that all $p, q, r \in P$, if $p \leq q$ and $q \leq r$, then, $p \leq r$, and $p \leq p$, where 

   \[ P(p,q) =     \left\{ \begin{array}{rcl}
         0 & \mbox{if}
         & p \leq q \\ \infty  & \mbox{if} & p \not \leq q
                \end{array}\right\} \] 

                 \item The set of strings $\Sigma^*$ over some alphabet defined as the set $\Sigma$ where the distance between two strings $u$ and $v$ is defined as 

   \[ \Sigma^*(u,v) =     \left\{ \begin{array}{rcl}
         0 & \mbox{if}
         & u \ \mbox{is a prefix of} \ v \\ 2^{-n} & \mbox{otherwise} & \mbox{where} \ n \ \mbox{is the longest common prefix of } \ u \ \mbox{and} \ v
                \end{array}\right\} \] 

                 \item The set of non-negative distances $[0,\infty]$ where the distance between two objects $u$ and $v$ is defined as 

   \[ [0,\infty](u,v) =     \left\{ \begin{array}{rcl}
         0 & \mbox{if}
         & u \geq  v \\ v - u & \mbox{otherwise} & \mbox{where} \ r < s 
                \end{array}\right\} \] 

                 \item The powerset ${\cal P}(X)$ of all subsets of a standard metric space, where the distance between two subsets $V, W \subseteq X$ is defined as

   \[  {\cal P}(X)(V, W) = \inf \{ \epsilon > 0 | \forall v \in V, \exists w \in W, X(v, w) \leq \epsilon \} \] 

which is often referred to as the {\em non-symmetric Hausdorff distance}. 
                
\end{enumerate}

Generalized metric spaces can be shown to be $[0, \infty]$-enriched categories as the collection of all morphisms between any two objects itself defines a category.  In particular, the category $[0,\infty]$ is a complete and co-complete symmetric monoidal category. It is a category because objects are the non-negative real numbers, including $\infty$, and for two objects $r$ and $s$ in $[0,\infty]$, there is an arrow from $r$ to $s$ if and only if $r \leq s$. It is complete and co-complete because all equalizers and co-equalizers exist as there is at most one arrow between any two objects. The categorical product $r \sqcap s$ of two objects $r$ and $s$ is simply $\max\{r,s\}$, and the categorical coproduct $r \sqcup s$  is simply $\min\{r,s\}$. More generally, products are defined by supremums, and coproducts are defined by infimums. Finally, the {\em monoidal } structure is induced by defining the tensoring of two objects through ``addition":   

\[ +: [0, \infty] \times [0, \infty] \rightarrow [0,\infty]\]

where $r + s$ is simply their sum, and where as usual $r + \infty = \infty + r = \infty$. 

The category $[0,\infty]$ is also a {\em compact closed} category, which turns out to be a fundamentally important property, and can be simply explained in this case as follows. We can define an ``internal hom functor" $[0,\infty](-, -)$ between any two objects $r$ and $s$ in $[0, \infty]$ the distance $[0,\infty]$ as defined above, and the {\em co pre-sheaf} $[0,\infty](t,-)$ is {\em right adjoint} to $t + -$ for any $t \in [0, \infty]$. 

\begin{theorem}
For all $r, s$ and $t \in [0,\infty]$, 

\[ t + s \geq r \ \ \ \mbox{if and only if} \  \  \ s \geq [0,\infty](t,r)\]
\end{theorem}

We will explain the significance of compact closed categories for reasoning about AI and ML systems in more detail later, but in particular, we note that reasoning about feedback requires using compact closed categories to represent ``dual" objects that are diagrammatically represented by arrows that run in the ``reverse" direction from right to left (in addition to the usual convention of information flowing from left to right from inputs to outputs in any process model). 

We can also define a category of generalized metric spaces, where each generalized metric space itself as an object, and for the morphism between generalized metric spaces $X$ and $Y$, we can choose a {\em non-expansive function} $f: X \rightarrow Y$ which has the {\em contraction property}, namely 

\[ Y(f(x), f(y)) \leq c \cdot X(x,y) \]

where $0 < c < 1$ is assumed to be some real number that lies in the unit interval. The category of generalized metric spaces will turn out to be of crucial importance in this paper as we will use a central result in category theory -- the Yoneda Lemma -- to give a new interpretation to distances. 

Finally, let us state a ``metric" version of the Yoneda Lemma specifically for the case of $[0,\infty]$-enriched categories in generalized metric spaces: 

\begin{theorem}\cite{BONSANGUE19981}
({\bf Yoneda Lemma for generalized metric spaces}): Let $X$ be a generalized metric space. For any $x \in X$, let 

\[ X(-, x): X^{\mbox{op}} \rightarrow [0, \infty], \ \ y \longmapsto X(y, x)\]
\end{theorem}

Intuitively, what the generalized metric version of the Yoneda Lemma is stating is that it is possible to represent an element of a generalized metric space by its co-presheaf, exactly analogous to what we will see below in the next section for causal inference! If we use the notation

\[ \hat{X} = [0, \infty]^{X^{\mbox{op}}}\]

to indicate the set of all non-expansive functions from $X^{\mbox{op}}$ to $[0, \infty]$, then the Yoneda embedding defined by $y \longmapsto X(y, x)$ is in fact a non-expansive function, and itself an element of $\hat{X}$! Thus, it follows from the general Yoneda Lemma that for any other element $\phi$ in $\hat{X}$, 

\[ \hat{X}(X(-, x), \phi) = \phi(x) \]

Another fundamental result is that the Yoneda embedding for generalized metric spaces is an {\em isometry}. Again, this is exactly analogous to what we see below for causal inference, which we will denote as the causal reproducing property. 

\begin{theorem}
The Yoneda embedding $y: X \rightarrow \hat{X}$, defined for $x \in X$ by $y(x) = X(-, x)$ is {\em isometric}, that is, for all $x, x' \in X$, we have: 

\[ X(x, x') = \hat{X}(y(x), y(x')) = \hat{X}(X(-, x), X(-, x'))\]
\end{theorem}

Once again, we will see a remarkable resemblance of this result to the Causal Representer Theorem below.  With the metric Yoneda Lemma in hand, we can now define a framework for solving static UIGs in generalized metric spaces. 

 \begin{definition}
Two objects $c$ and $d$ are isomorphic in a generalized metric space category $X$ if they are isometrically mapped into the category $\hat{X}$ by the Yoneda embedding $c \rightarrow X(-, c)$ and $d \rightarrow X(-, d)$ such that $X(c,d) = \hat{X}(X(-, c), X(-, d))$, where they can be defined isomorphically by a suitable pair of suitable natural transformations. 
 \end{definition}

\subsection{Adjoint Functors} 

Adjoint functors naturally arise in a number of contexts, among the most important being between ``free" and ``forgetful" functors. Let us consider a canonical example that  is of prime significance in many applications in AI and ML. 

\begin{figure}[h] 
\centering
\caption{Adjoint functors provide an elegant characterization of the relationship between the category of statistical generative AI models and that of causal generative AI models. Statistical models can be viewed as the result of applying a ``forgetful" functor to a causal model that drops the directional structure in a causal model, whereas causal models can be viewed as ``words" in a ``free" algebra that results from the left adjoint functor to the forgetful functor.  \label{causalstatistical}}
\vskip 0.1in
\begin{minipage}{0.7\textwidth}
\vskip 0.1in
\includegraphics[scale=0.45]{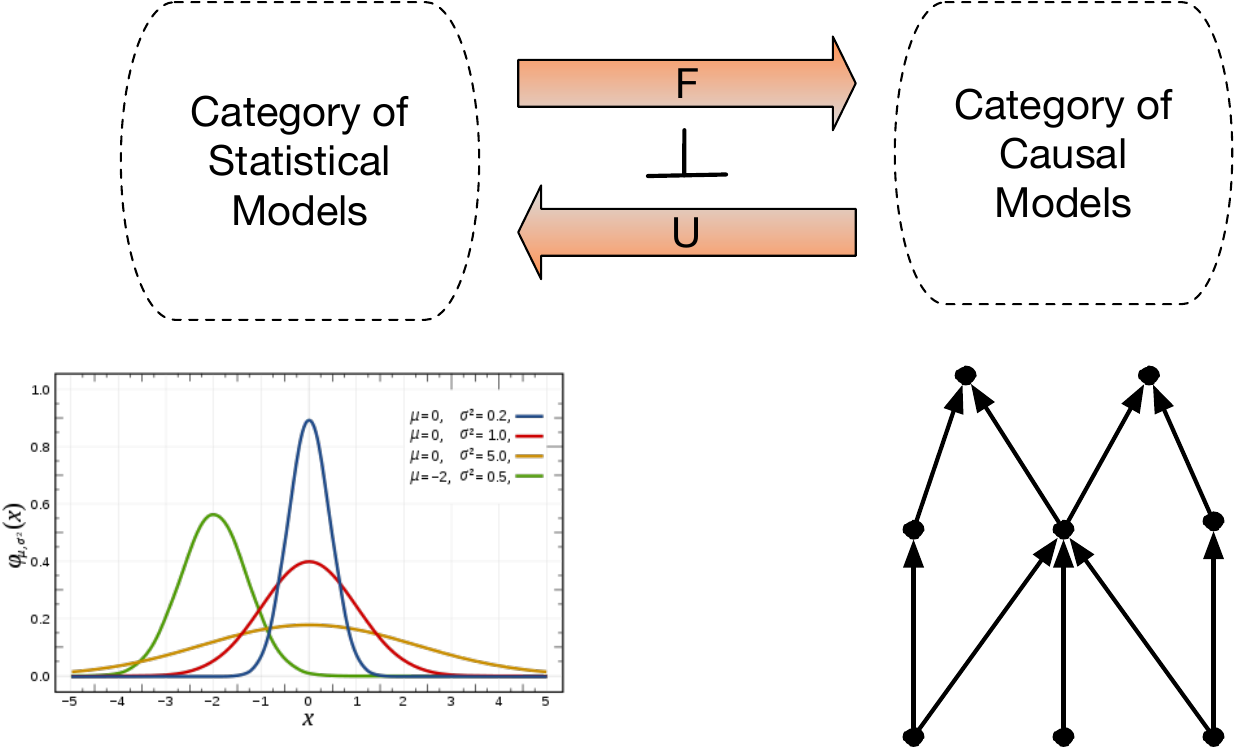}
\end{minipage}
\end{figure}

Figure~\ref{causalstatistical} provides a high level overview of the relationship between a category of statistical generative AI models and a category of causal generative AI models that can be seen as being related by a pair of adjoint ``forgetful-free" functors. A statistical model can be abstractly viewed in terms of its conditional independence properties. More concretely, the category of {\em separoids}, defined in Section 2, consists of objects called separoids $(S, \leq)$, which are semilattices with a preordering $\leq$ where the elements $x, y, z \in S$ denote entities  in a statistical model. We define a ternary relation $(\bullet \perp \bullet | \bullet) \subseteq S \times S \times S$, where $(x \perp y | z)$ is interpreted as the statement $x$ is conditionally independent of $y$ given $z$ to denote a relationship between triples that captures abstractly the property that occurs in many applications in AI and ML. For example, in statistical ML, a sufficient statistic $T(X)$ of some dataset $X$, treated as a random variable, is defined to be any function for which the conditional independence relationship $(X \perp  \theta| T(X))$, where $\theta \in \mathbb{R}^k$ denotes the parameter vector of some statistical model $P(X)$ that defines the true distribution of the data. Similarly, in causal inference, $(x \perp y | z)  \Rightarrow p(x, y, z) = p(x | z) p(y | z)$ denotes a statement about the probabilistic conditional independence of $x$ and $y$ given $z$. In causal inference, the goal is to recover a partial order defined as a directed acyclic graph (DAG) that ascribes causality among a set of random variables from a dataset specifying a sample of their joint distribution. It is well known that without non-random interventions, causality cannot be inferred uniquely, since because of Bayes rule, there is no way to distinguish causal generative AI models such as $x \rightarrow y \rightarrow z$ from the reverse relationship $z \rightarrow y \rightarrow x$. In both these models, $x \perp z | y$ and because of Bayes inversion, one model can be recovered from the other. We can define a ``free-forgetful" pair of adjoint functors between the category of conditional independence relationships, as defined by separoid objects, and the category of causal generative AI models parameterized by DAG models.

We first review some basic material relating to adjunctions defined by adjoint functors, before proceeding to describe the theory of monads, as the two are intimately related. Our presentation of adjunctions and monads is based on Riehl's excellent textbook on category theory \cite{riehl2017category} to which the reader is referred to for a more detailed explanation.
Adjunctions are defined by an opposing pair of functors $F: C \leftrightarrow D: G$ that can be defined more precisely as follows. 

\begin{definition}
    An {\bf adjunction} consists of a pair of functors $F: C \rightarrow D$ and $G: D \rightarrow C$, where $F$ is often referred to {\em left adjoint} and  $G$ is referred to as the {\em right adjoint}, that result in the following isomorphism relationship holding between their following sets of homomorphisms in categories $C$ and $D$: 

    \[ D(Fc, d) \simeq C(c, Gd) \]
\end{definition}

We can express the isomorphism condition more explicitly in the form of the following commutative diagram: 

\begin{center}
\begin{tikzcd}
  D(Fc, d) \arrow[r, "\simeq"] \arrow[d, "k_*"]
    & C(c, Gd) \arrow[d, "Gk_*" ] \\
  D(Fc, d') \arrow[r,  "\simeq"]
& C(c, Gd')
\end{tikzcd}
\end{center}

Here, $k: d \rightarrow d'$ is any morphism in $D$, and $k_*$ denotes the ``pullback" of $k$ with the mapping $f: Fc \rightarrow d$ to yield the composite mapping $k \circ f$. The adjunction condition holds that the transpose of this composite mapping is equal to the composite mapping $g: c \rightarrow Gd$ with $G k: Gd \rightarrow G d'$. We can express this dually as well, as follows:

\begin{center}
\begin{tikzcd}
  D(Fc, d) \arrow[r, "\simeq"] \arrow[d, "Fh^*"]
    & C(c, Gd) \arrow[d, "h^*" ] \\
  D(Fc', d) \arrow[r,  "\simeq"]
& C(c', Gd')
\end{tikzcd}
\end{center}

where now $h: c' \rightarrow c$ is a morphism in $C$, and $h^*$ denote the ``pushforward" of $h$. Once again, the adjunction condition is a statement that the transpose of the composite mapping $f \circ Fh: F c' \rightarrow d$ is identical to the composite of the mappings $h: c \rightarrow c'$ with $f: c \rightarrow Gd$.

It is common to denote adjoint functors in this turnstile notation, indicating that $F: C \rightarrow D$ is left adjoint to $G: D \rightarrow C$, or more simply as $F \vdash G$. 

\[
        \begin{tikzcd}
            \mathcal{D}\arrow[r, shift left=.75ex, "G"{name=G}] & \mathcal{C}\arrow[l, shift left=.75ex, "F"{name=F}] 
            \arrow[phantom, from=F, to=G, "\dashv" rotate=90].      
        \end{tikzcd}
    \]

We can use the concept of universal arrows introduced in Section 2 to give more insight into adjoint functors. The adjunction condition for a pair of adjoint functors $F \vdash G$ 

\[ D(Fc, d) \simeq C(c, Gd) \]

implies that for any object $c \in C$, the object $Fc \in D$ represents the functor $C(c, G -): D \rightarrow {\bf Set}$. Recall from the Yoneda Lemma that the natural isomorphism $D(Fc, -) \simeq C(c, G-)$ is determined by an element of $C(c,GFc)$, which can be viewed as the transpose of $1_{Fc}$. Denoting such elements as $\eta_c$, they can be assembled jointly into the natural transformation $\eta: 1_C \rightarrow GF$. Below we will see that this forms one of the conditions for an endofunctor to define a monad. 

\begin{theorem}
    The {\bf unit} $\eta: 1_C \rightarrow GF$ is a natural transformation defined by an adjunction $F \vdash G$, whose component $\eta_c: c \rightarrow GF c$ is defined to be the transpose of the identity morphism $1_{Fc}$. 
\end{theorem}

{\bf Proof:} We need to show that for every $f: c \rightarrow c'$, the following diagram commutes, which follows from the definition of adjunction and the isomorphism condition that it imposes, as well as the obvious commutativity of the second transposed diagram below the first one. 

\begin{center}
\begin{tikzcd}
  c \arrow[r, "\eta_c"] \arrow[d, "f"]
    & GF c\arrow[d, "GF f" ] \\
  c' \arrow[r,  "\eta_{c'}"]
& GF c'
\end{tikzcd}
\end{center}

\begin{center}
\begin{tikzcd}
  Fc \arrow[r, "1_{Fc}"] \arrow[d, "Ff"]
    & Fc \arrow[d, "F f" ] \\
  Fc' \arrow[r,  "1_{Fc'}"]
& Fc'
\end{tikzcd}
\end{center}

The dual of the above theorem leads to the second major component  of an adjunction. 

\begin{theorem}
    The {\bf counit} $\epsilon: FG \Rightarrow 1_D$ is a natural transformation defined by an adjunction $F \vdash G$, whose components $\epsilon_c: F G d \rightarrow d$ at $d$ is  defined to be the transpose of the identity morphism $1_{Gd}$. 
\end{theorem}

Adjoint functors interact with universal constructions, such as limits and colimits, in ways that turn out to be important for a variety of applications in AI and ML. We state the main results here, but refer the reader to \cite{riehl2017category} for detailed proofs. Before getting to the general case, it is illustrative to see the interaction of limits and colimits with adjoint functors for preorders. Recall from above that separoids are defined by a preorder $(S, \leq)$ on a join lattice of elements from a set $S$. Given two separoids $(S, \leq_S)$ and $(T, \leq_T)$, we can define the functors $F: S \rightarrow T$ and $G: T \rightarrow S$ to be order-preserving functions such that 

\[ Fa \leq_T b \ \ \ \mbox{if and only if} \ \ \ a \leq_S Gb \]

Such an adjunction between preorders is often called a {\em Galois connection}.  For preorders, the limit is defined by the {\em meet} of the preorder, and the colimit is defined by the {\em join} of the preorder. We can now state a useful result. For a fuller discussion of preorders and their applications from a category theory perspective, see \cite{fong2018seven}. 

\begin{theorem}
    {\bf Right adjoints preserve meets in a preorder}: Let $f: P \rightarrow Q$ be left adjoint to $g: Q \rightarrow P$, where $P, Q$ are both preorders, and $f$ and $g$ are monotone order-preserving functions. For any subset $A \subseteq Q$, let $g(A) = \{ g(a) | a \in Q \}$. If $A$ has a meet $\bigwedge A \in Q$, then $g(A)$ has a meet $\wedge g(A) \in P$, and we can see that $g(\wedge A) \simeq \bigwedge g(A) $, that is, right adjoints preserve meets. Similarly, left adjoints preserve meets, so that if $A \subset P$ such that $\bigvee A \in P$ then $f(A)$ has a join $\vee f(A) \in Q$ and we can set $f(\vee A) \simeq \bigvee f(A)$, so that left adjoints preserve joins. 
\end{theorem}

{\bf Proof:} The proof is not difficult in this special case of the category being defined as a preorder. If $f: P \rightarrow Q$ and $g: Q \rightarrow P$ are monotone adjoint maps on preorders $P, Q$, and $A \subset Q$ is any subset such that its meet is $m = \wedge A$. Since $g$ is monotone, $g(m) \leq g(a), \ \forall a \in A$, hence it follows that $g(m) \leq g(A)$. To show that $g(m$ is the greatest lower bound, if we take any other lower bound $b \leq g(a), \ \forall a \in A$, then we want to show that $b \leq g(m)$. Since $f$ and $g$ are adjoint, for every $p \in  P, q \in Q$, we have

\[ p \leq g(f(p)) \ \ \ \mbox{and} \ \ \ f(g(q)) \leq q \]

Hence, $f(b) \leq a$ for all $a \in A$, which implies $f(b)$ is  a lower bound for $A$ on $Q$. Since the meet $m$ is the greatest lower bound, we have $f(b) \leq m$. Using the Galois connection, we see that $b \leq g(m)$, and hence showing that $g(m)$ is the greatest lower bound as required. An analogous proof follows to show that left adjoints preserve joins. $\qed$

We can now state the more general cases for any pair of adjoint functors, as follows. 

\begin{theorem}
    A category ${\cal C}$ admits all limits of diagrams indexed by a small category ${\cal J}$ if and only if the constant functor $\Delta: {\cal C} \rightarrow {\cal C}^{{\cal J}}$ admits a right adjoint, and admits all colimits of ${\cal J}$-indexed diagrams if and only if $\Delta$ admits a left adjoint. 
\end{theorem}

By way of explanation, the constant functor $c: J \rightarrow C$ sends every object of $J$ to $c$ and every morphism of $J$ to the identity morphism  $1_c$. Here, the constant functor $\Delta$ sends every object $c$ of $C$ to the constant diagram $\Delta c$, namely the functor that maps each object $i$ of $J$ to the object $c$ and each morphism of $J$ to the identity $1_c$. The theorem follows from the definition of the universal properties of colimits and limits. Given any object $c \in C$, and any diagram (functor) $F \in {\cal C}^{{\cal J}}$, the set of morphisms ${\cal C}^{{\cal J}}(\Delta c, F)$ corresponds to the set of natural transformations from the constant ${\cal J}$-diagram at $c$ to the diagram $F$. These natural transformations precisely correspond to the cones over $F$ with summit $c$ in the definition given earlier in Section 2. It follows that there is an object $\lim F \in {\cal C}$ together with an isomorphism 

\[ {\cal C}^{{\cal J}}(\Delta c, F) \simeq {\cal C}(c, \lim F) \]

We can now state the more general result that we showed above for the special case of adjoint functors on preorders. 

\begin{theorem}
    Right adjoints preserve limits, whereas left adjoints preserve colimits. 
\end{theorem}

\section{The Coend and End of GAIA: Integral Calculus for Generative AI}

\label{coend} 

In this section, we introduce a powerful abstract integral calculus for generative AI based  on the theory of coends and ends \cite{yoneda-end,loregian_2021}. 

We build on two foundational results in category theory: the metric Yoneda Lemma \cite{BONSANGUE19981} shows how to construct universal representations of generative AI models in generalized metric spaces where symmetry does not hold;  and a categorical integral calculus also introduced by Yoneda \cite{yoneda-end} based on (co)ends, (initial) final objects in a category of (co)wedges. \cite{loregian_2021} provides an excellent book-length treatment of Yoneda's categorical integral calculus of (co)ends.  We define two classes of generative AI  modes based on coends and ends. Coend generative AI models are defined by dinatural transformations between bifunctors $F: {\cal C}^{op} \times {\cal C} \rightarrow {\cal D}$ that combine a contravariant and covariant action. Here,  ${\cal C}$ represents a generic category of generative AI models, modeled as a {\em twisted arrow} category.   The co-domain category ${\cal D}$ is the category ${\bf Meas}$ of measurable spaces for generative AI models based on ends, and the category ${\bf Top} $ of topological spaces for the generative AI models based on coends. Recent theoretical results have shown that the traditional Transformer model is a universal approximator of sequences, despite the restriction of permutation equivariance, due to the use of absolute positional encoding of input tokens, which leads to poor generalization on long sequences. Modifications, such as relative positional encoding, impose limitations on the universal approximability of the traditional Transformer. We conjecture that coend generative AI models provide  a non-symmetric measure of distance, and furthermore, capture higher-order interactions between tokens using the structure of simplicial sets. 

\begin{figure}[h]
\centering
\includegraphics[scale=.4]{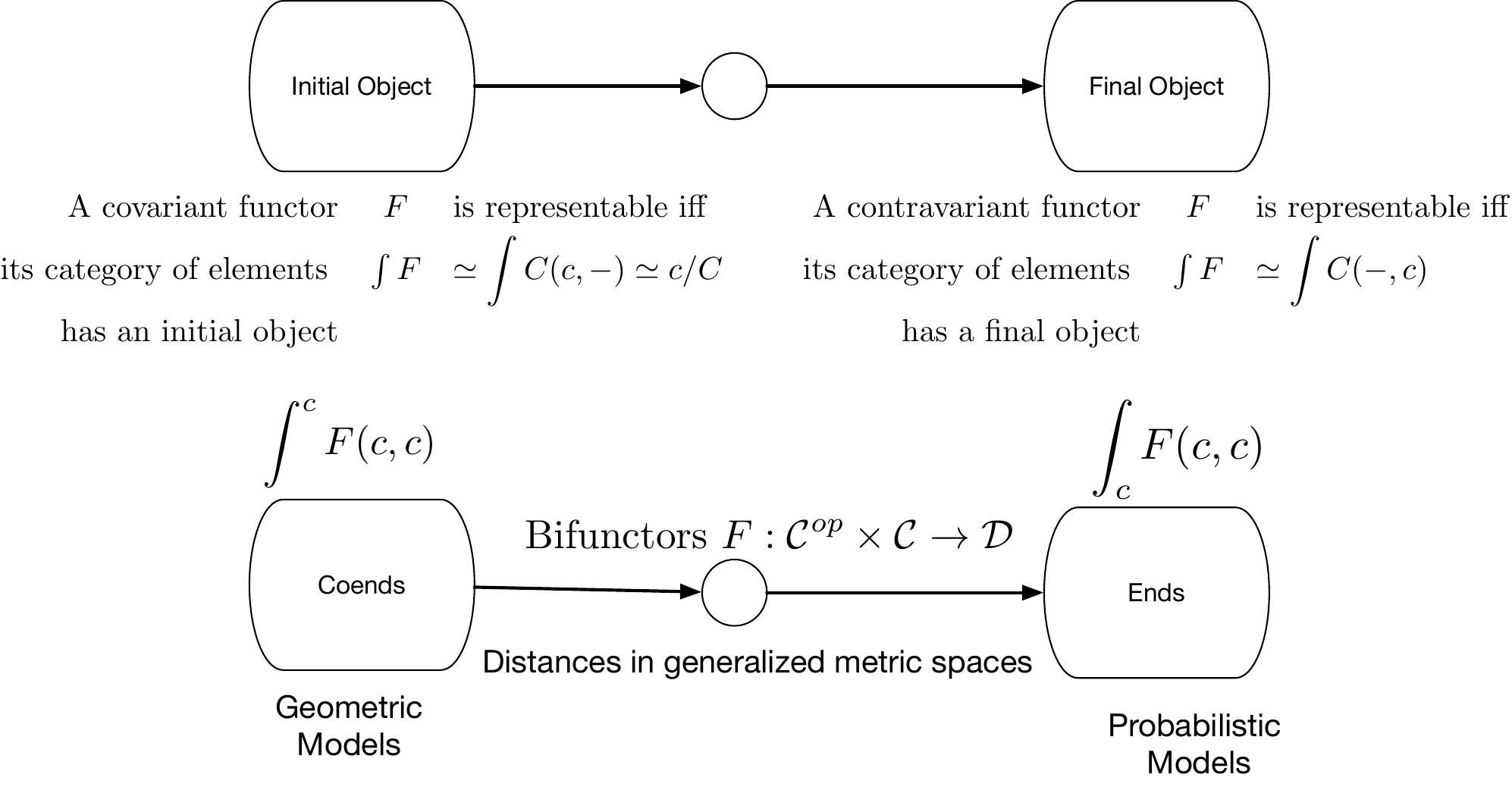}
\caption{The theoretical foundation of GAIA is based on two celebrated results of Yoneda. The first (top row) shows that Yoneda embeddings $\yo(x) = C(-, x)$ are universal representers of objects in a category. We use this result to define universal representers of generative AI models. The second (bottom row) is based on Yoneda's categorical ``integral calculus" using coends and ends \cite{yoneda-end}, which  defines two classes of generative AI models ranging from  probabilistic models to topological models.}  
\label{yonedathms} 
\end{figure}

Figure~\ref{yonedathms} illustrates the two fundamental insights developed by Yoneda that form the theoretical core of our GAIA framework. The celebrated Yoneda Lemma \cite{maclane:71} asserts that objects in a category ${\cal C}$ can be defined purely in terms of their interactions with other objects. This interaction is modeled by {\em contravariant} or {\em covariant} functors: 

\[ {\cal C}(-, x): {\cal C}^{op} \rightarrow {\bf Sets},  \ \ \  {\cal C}(x, -): {\cal C} \rightarrow {\bf Sets} \]

The {\em Yoneda embedding} $x \rightarrow {\cal C}(-, x)$ is sometimes denoted as $\yo(x)$ for the Japanese Hiragana symbol for {\tt yo}, serves as a {\em universal representer}, and generalizes many other similar ideas in machine learning, such as representers $K(-, x)$ in kernel methods \cite{kernelbook} and representers of causal information \cite{DBLP:journals/entropy/Mahadevan23}. There are many variants of the Yoneda Lemma, including versions that map the functors ${\cal C}(-, x)$ and ${\cal C}(x, -)$ into an {\em enriched} category. In particular, \cite{bradley:enriched-yoneda-llms} contains an extended discussion of the use of an enriched Yoneda Lemma to model natural language interactions that result from using a large language model. In particular, we build on the metric Yoneda Lemma \cite{BONSANGUE19981} that defines a universal representer in generalized metric spaces, where distances are non-symmetric. The second major insight from Yoneda \cite{yoneda-end} is based on a powerful concept of the {\em coend} and {\em end} of a {\em bifunctor} $F: {\cal C}^{op} \times {\cal C} \rightarrow {\cal D}$ that combines both a {\em contravariant} and a {\em covariant} action.  We build on the insight that probabilistic generative models, or using distances in some metric space, correspond to final or initial objects in a category of wedges, defined by  bifunctors, and the arrows are dinatural transformations. These initial or terminal objects correspond to coends and ends. Bifunctors $F: {\cal C}^{op} \times {\cal C} \rightarrow {\cal D} $ can be used to construct universal representers of distance functions in generalized metric spaces leading to a ``metric Yoneda Lemma" \cite{BONSANGUE19981}. 

Recent universal approximation results \cite{DBLP:conf/iclr/YunBRRK20} have shown that the category ${\cal C}_{T}$ of transformers is dense in the parent category of all permutation-equivariant functions on (compact) vector spaces ${\cal C}_{PE}$ defined by vectors $x \in \mathbb{R}^{n \times d}$ over arbitrary continuous permutation equivariant functions. We define a twisted arrow category ${\cal C}^{TW}_{PE}$, which has as its objects the equivariant maps of ${\cal C}_{PE}$, and commutative diagrams over pairs of equivariant maps $f,g$ in ${\cal C}_{PE}$ as its morphisms. To define the (co)ends of Transformer models, we define a category of wedges defined by bifunctors $F: ({\cal C}^{TW}_{PE})^{op} \times {\cal C}^{TW}_{PE} \rightarrow {\cal D}$ that contravariantly and covariantly map Transformer models into ${\cal D}$, the codomain category, which may be the category {\bf Meas} of measurable spaces, or the category of distances $[0, \infty]$ defined by $l_p$ norms over permutation equivariant functions. We use the metric Yoneda Lemma to construct a {\em universal representer} of Transformer models in a generalized metric space. Building on Yoneda's categorical calculus of (co)ends, we define the end $\int_c F(c,c)$ of Transformer models  as the final object in the category of wedges, whereas the coend $\int^c F(c,c)$ of Transformer modes are defined as the initial object in the category of cowedges, both defined over dinatural transformations between bifunctors over transformer models. Ends induce probabilistic generative models over sequences of tokens implemented as Transformer models, whereas coends lead to {\em Geometric Transformer Models} (GTMs),  a new class of generative sequence models defined by the  topological embedding of (fuzzy) simplicial sets.

\subsection{Ends and Coends} 

We will analyze generative AI models in  the category of {\em wedges}, which are defined by a collection of objects comprised of bifunctors $F: {\cal C}^{op} \times C \rightarrow {\cal D}$, and a collection of arrows between each pair of bifunctors $F, G$ called a {\em dinatural transformation} (as an abbreviation for diagonal natural transformation). We will see below that the initial and terminal objects in the category of wedges correspond to a beautiful idea first articulated by Yoneda called the {\em coend} or {\em end} \cite{yoneda-end}. \cite{loregian_2021} has an excellent treatment of coend calculus, which we will use below. 

\begin{definition}
    Given a pair of bifunctors $F, G: {\cal C}^{op} \times {\cal C} \rightarrow {\cal D}$, a {\bf dinatural transformation} is defined as follows: 

\[\begin{tikzcd}
	&& {F(c',c)} \\
	{F(c,c)} &&&& {F(c',c')} \\
	\\
	{G(c,c)} &&&& {G(c',c')} \\
	&& {G(c,c')}
	\arrow["{F(f,c)}", from=1-3, to=2-1]
	\arrow["{F(c',f)}"', from=1-3, to=2-5]
	\arrow[dashed, from=2-1, to=4-1]
	\arrow[dashed, from=2-5, to=4-5]
	\arrow["{G(c,f)}", from=4-1, to=5-3]
	\arrow["{G(f,c)}"', from=4-5, to=5-3]
\end{tikzcd}\]

\end{definition}

As \cite{loregian_2021} observes, just as a natural transformation interpolates between two regular functors $F$ and $G$ by filling in the gap between their action on a morphism $Ff$ and $Fg$ on the codomain category, a dinatural transformation ``fills in the gap" between the top of the hexagon above and the bottom of the hexagon. 

We can define a {\em constant bifunctor} $\Delta_d: {\cal C}^{op} \times {\cal C} \rightarrow {\cal D}$ by the object it maps everything to, namely the input pair of objects $(c, c') \rightarrow d$ are both mapped to the object $d \in {\cal D}$, and the two input morphisms $(f, f') \rightarrow {\bf 1}_d$ are both mapped to the identity morphism on $d$. We can now define {\em wedges} and {\em cowedges}. 

\begin{definition}
    A {\bf wedge} for a bifunctor $F: {\cal C}^{op} \times {\cal C} \Rightarrow {\cal D}$ is a dinatural transformation $\Delta_d \rightarrow F$ from the constant functor on the object $d \in {\cal D}$ to $F$. Dually, we can define a {\bf cowedge} for a bifunctor $F$ by the dinatural transformation $P \Rightarrow \Delta_d$. 
\end{definition}

We can now define a {\em category of wedges}, each of whose objects are wedges, and for arrows, we choose arrows in the co-domain category that makes the diagram below commute. 

\begin{definition}
    Given a fixed bifunctor $F: {\cal C}^{op} \times {\cal C} \rightarrow {\cal D}$, we define the {\bf category of wedges} ${\cal W}(F)$ where each object is a wedge $\Delta_d \Rightarrow F$ and given a pair of wedges $\Delta_d \Rightarrow F$ and $\Delta_d' \Rightarrow F$, we choose an arrow $f: d \rightarrow d'$ that makes the following diagram commute: 

\[\begin{tikzcd}
	d &&&& {d'} \\
	\\
	&& {F(c,c)}
	\arrow["f", from=1-1, to=1-5]
	\arrow["{\alpha_{cc}}"', from=1-1, to=3-3]
	\arrow["{\alpha'_{cc}}", from=1-5, to=3-3]
\end{tikzcd}\]
Analogously, we can define a {\bf category of cowedges} where each object is defined as a cowedge $F \Rightarrow \Delta_d$. 
\end{definition}

With these definitions in place, we can once again define the universal property in terms of initial and terminal objects. In the category of wedges and cowedges, these have special significance for formulating and solving UIGs, as we will see in the next section. 

\begin{definition}
    Given a bifunctor $F: {\cal C}^{op} \times {\cal C} \rightarrow {\cal D}$, the {\bf end} of $F$ consists of a terminal wedge $\omega: \underline{{\bf end}}(F) \Rightarrow F$. The object $\underline{{\bf end}}(F) \in D$ is itself called the end. Dually, the {\bf coend} of $F$ is the initial object in the category of cowedges $F \Rightarrow \underline{{\bf coend}}(F)$, where the object $\underline{{\bf coend}}(F) \in {\cal D}$ is itself called the coend of $F$.  
\end{definition}

Remarkably, probabilities can be formally shown to define ends of a category \cite{Avery_2016}, and topological embeddings of datasets, as implemented in popular dimensionality reduction methods like UMAP \cite{umap}, correspond to coends \cite{maclane:71}.  These connections suggest the canonical importance of the category of wedges and cowedges in formulating and solving UIGs. First, we introduce another universal construction, the Kan extension, which turns out to be the basis of every other concept in category theory.

\subsection{Sheaves and Topoi in GAIA}

\label{sheavestopoi} 

So far, we have assumed that the parameter spaces for generative AI are vector spaces $\mathbb{R}^n$, as is typically assumed in deep learning \cite{deeplearningreview-2009}. But there are excellent reasons to consider more abstract spaces, and in particular, we describe here an important category of sheaves and topoi \cite{maclane:sheaves} where some of the most interesting results in category theory, like the (metric) Yoneda Lemma, find their application. 

In this section, we define an important categorical structure defined by sheaves and topoi \cite{maclane:sheaves}.  Yoneda embeddings $\yo(x): {\cal C}^{op} \rightarrow {\bf Sets}$ define (pre)sheaves, which satisfy a number of crucial properties that make it remarkably similar to the category of {\bf Sets}. The sheaf condition plays an important role in many applications of machine learning, from dimensionality reduction \cite{umap} to causal inference \cite{DBLP:journals/entropy/Mahadevan23}. \cite{maclane:sheaves} provides an excellent overview of sheaves and topoi, and how remarkably they unify much of mathematics, from geometry to logic and topology. We will give only the briefest of overviews here, and apply in the main ideas to the study of UIGs. 

\begin{figure}[h] 
\centering
\caption{Two applications of sheaf theory in AI: (top) minimizing travel costs in weighted graphs satisfies the sheaf principle, one example of which is the Bellman optimality principle in dynamic programming \cite{DBLP:books/lib/Bertsekas05} and reinforcement learning \cite{bertsekas:rlbook,DBLP:books/lib/SuttonB98} (bottom): Approximating a function over a topological space must satisfy the sheaf condition. \label{sheaves}}
\vskip 0.1in
\begin{minipage}{0.7\textwidth}
\vskip 0.1in
\includegraphics[scale=0.35]{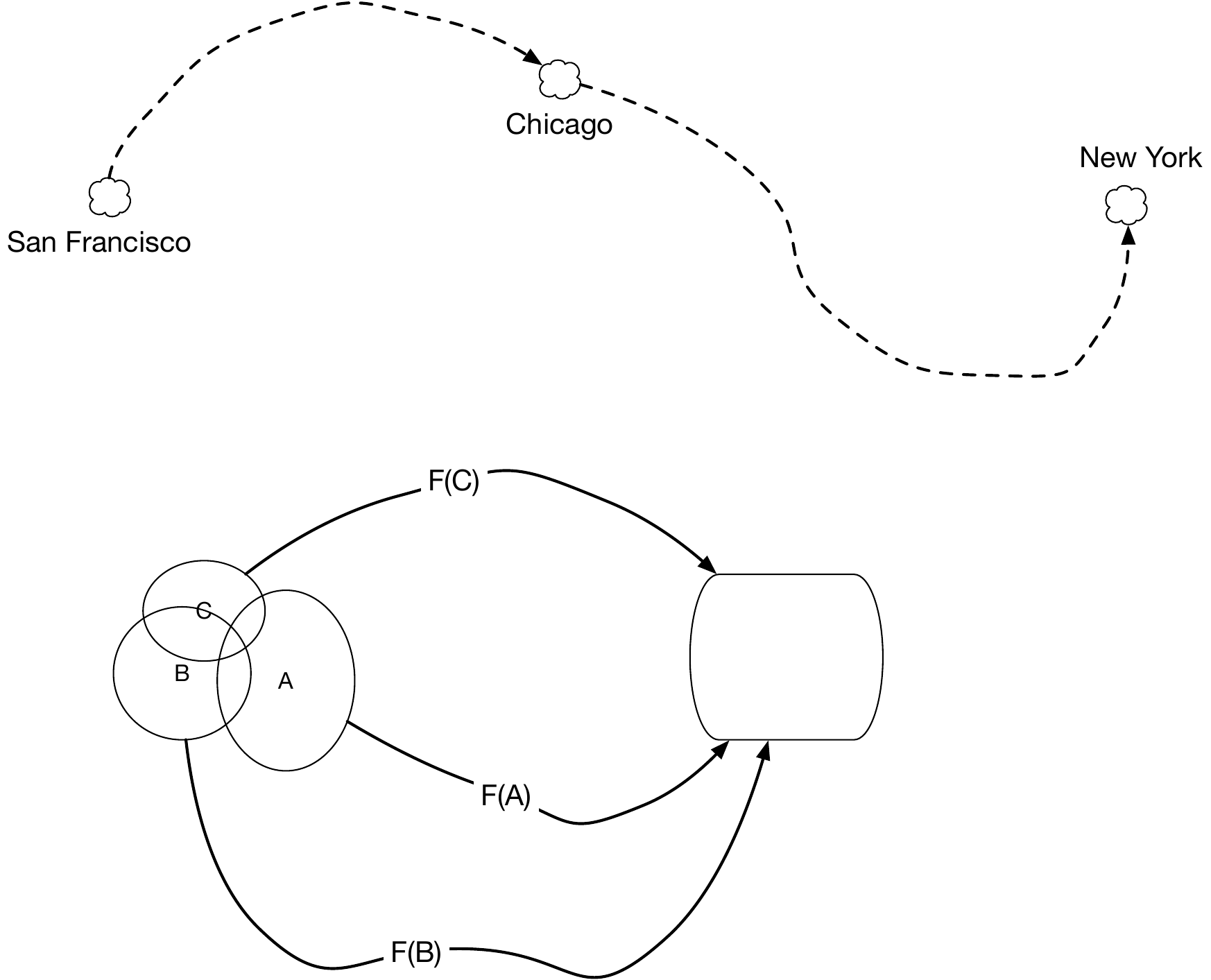}
\end{minipage}
\end{figure}

Figure~\ref{sheaves} gives two concrete examples of sheaves. In a minimum cost transportation problem, say using optimal transport \cite{ot} or reinforcement learning \cite{DBLP:books/lib/SuttonB98}, any optimal solution has the property that any restriction of the solution must also be optimal. In RL, this sheaf principle is codified by the Bellman equation, and leads to the fundamental principle of dynamic programming \cite{DBLP:books/lib/Bertsekas05}. Consider routing candy bars from San Francisco to New York city. If the cheapest way to route candy bars is through Chicago, then the restriction of the overall route to the (sub) route from Chicago to New York City must also be optimal, otherwise it is possible to find a shortest overall route by switching to a lower cost route. Similarly, in function approximation with real-valued functions $F: {\cal C} \rightarrow \mathbb{R}$, where ${\cal C}$ is the category of topological spaces, the (sub)functions $F(A), F(B)$ and $F(C)$ restricted to the open sets $A$, $B$ and $C$ must agree on the values they map the elements in the intersections $A \cap B$, $A \cap C$, $A \cap B \cap C$ and so on. Similarly, in causal inference, any probability distribution that is defined over a causal generative AI model must satisfy the sheaf condition in that any restriction of the causal model to a submodel must be consistent, so that two causal submodels that overlap in their domains must agree on the common elements. 

Sheaves can be defined over arbitrary categories, and we introduce the main idea by focusing on the category of sheaves over {\bf Sets}. 

\begin{definition}\cite{maclane:sheaves}
    A {\bf sheaf} of sets $F$ on a topological space $X$ is a functor $F: {\cal O}^{op} \rightarrow {\bf Sets} $ such that each open covering $U = \bigcup_i U_i, i \in I$ of an open set $O$ of $X$ yields an equalizer diagram
\[
\xymatrix{
FU\ar@{-->}[r]^e&} 
\begin{tikzcd}
 \prod_i FU_i \ar[r,shift left=.75ex,"p"]
  \ar[r,shift right=.75ex,swap,"q"]
&
\prod_{i,} F(U_i \cap U_j)
\end{tikzcd}
\]

The above definition succinctly captures what Figure~\ref{sheaves} shows for the example of approximating functions: the value of each subfunction must be consistent over the shared elements in the intersection of each open set. 

\begin{definition}
    The category $\mbox{Sh}(X)$ of sheaves over a space $X$ is a full subcategory of the functor category ${\bf Sets}^{{\cal O}(X)^{op}}$.
\end{definition}

\subsection*{Grothendieck Topologies}

We can generalize the notion of sheaves to arbitrary categories using the Yoneda embedding $\yo(x) = {\cal C}(-, x)$. We explain this generalization in the context of a more abstract topology on categories called the {\em Grothendieck topology} defined by {\em sieves}. A sieve can be viewed as a {\em subobject} $S \subseteq \yo(x)$ in the presheaf ${\bf Sets}^{{\cal C}^{op}}$, but we can define it more elegantly as a family of morphisms in ${\cal C}$, all with codomain $x$ such that

\[ f \in S \Longrightarrow f \circ g \in S \]

Figure~\ref{sieves} illustrates the idea of sieves. A simple way to think of a sieve is as a {\em right ideal}. We can define that more formally as follows: 

\begin{definition}
    If $S$ is a sieve on $x$, and $h: D \rightarrow x$ is any arrow in category ${\cal C}$, then 

    \[ h^* = \{g \ | \ \mbox{cod}(g) = D, hg \in S \}\]
\end{definition}

\begin{figure}[t] 
\centering
\caption{Sieves are subobjects of of $\yo(x)$ Yoneda embeddings of a category ${\cal C}$, which generalizes the concept of sheaves over sets in Figure~\ref{sheaves}. \label{sieves}}
\vskip 0.1in
\begin{minipage}{0.7\textwidth}
\vskip 0.1in
\includegraphics[scale=0.35]{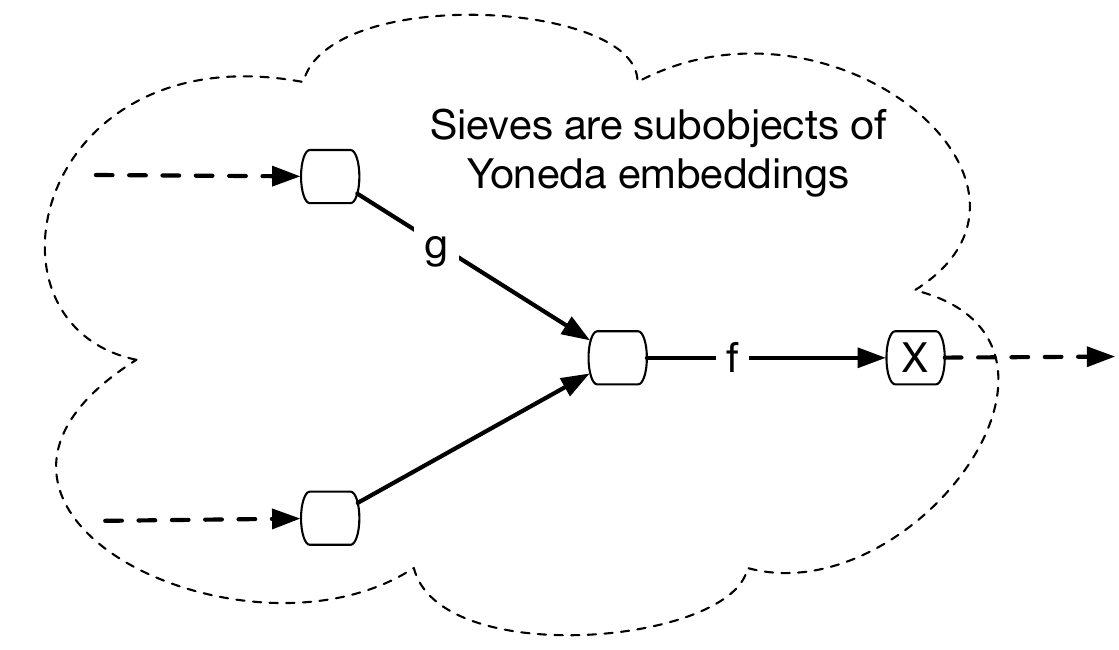}
\end{minipage}
\end{figure}

\begin{definition}\cite{maclane:sheaves}
A {\bf Grothendieck topology} on a category ${\cal C}$ is a function $J$ which assigns to each object $x$  of ${\cal C}$ a collection $J(x)$ of sieves on $x$ such that
\begin{enumerate}
    \item the maximum sieve $t_x = \{ f | \mbox{cod}(f) = x \}$ is in $J(x) $. 
    \item If $S \in J(x)$  then $h^*(S) \in J(D)$ for any arrow $h: D \rightarrow x$. 
    \item If $S \in J(x)$ and $R$ is any sieve on $x$, such that $h^*(R) \in J(D)$ for all $h: D \rightarrow x$, then $R \in J(C)$. 
\end{enumerate}
\end{definition}

We can now define categories with a given Grothendieck topology as {\em sites}. 

\begin{definition}
    A {\bf site} is defined as a pair $({\cal C}, J)$ consisting of a small category ${\cal C}$ and a Grothendieck topology $J$ on ${\cal C}$. 
\end{definition}

An intuitive way to interpret a site is as a generalization of the notion of a topology on a space $X$, which is defined as a set $X$ together with a collection of open sets ${\cal O}(X)$. The sieves on a category play the role of ``open sets". 

\end{definition}

\subsection*{Exponential Objects and Cartesian Closed Categories}

To define a topos, we need to understand the category of {\bf Sets} a bit more. Clearly, the single point set $\{ \bullet \}$ is a terminal object for {\bf Sets}, and the binary product of two sets $A \times B$ can always be defined. Furthermore, given two sets $A$ and $B$, we can define $B^A$ as the exponential object representing the set of all functions $f: A \rightarrow B$. We can define exponential objects in any category more generally as follows. 

\begin{definition}
    Given any category ${\cal C}$ with products, for a fixed object $x$ in ${\cal C}$, we can define the functor 

    \[ x \times - :  \rightarrow {\cal C}\]

    If this functor has a right adjoint, which can be denoted as 

    \[ (-)^x: {\cal C} \rightarrow {\cal C} \]

    then we say $x$ is an {\bf exponentiable} object of ${\cal C}$. 
\end{definition}

\begin{definition}
    A category ${\cal C}$ is {\bf Cartesian closed} if it has finite products (which is equivalent to saying it has a terminal object and binary products) and if all objects in ${\cal C}$ are {\em exponentiable}. 
\end{definition}

A result that is of foundational importance to this paper is that the category defined by Yoneda embeddings is Cartesian closed. 

\begin{theorem}\cite{maclane:sheaves}
    For any small category ${\cal C}$, the functor category ${\bf Sets}^{{\cal C}^{op}}$ is Cartesian closed
\end{theorem}

For a detailed proof, the reader is referred to \cite{maclane:sheaves}. A further result of significance is the {\em density theorem}, which can be seen as the generalization of the simple result that any set $S$ can be defined as the union of single point sets $\bigcup_{x \in S} \{ x \}$. 

\begin{theorem}\cite{maclane:sheaves}
    In a functor category ${\bf Sets}^{{\cal C}^{op}}$, any object $x$ is the colimit of a diagram of representable objects in a canonical way. 
\end{theorem}

Recall that an object is representable if it is isomorphic to a Yoneda embedding $\yo(x)$. This result has numerous applications to AI and ML, among them to causal inference \cite{DBLP:journals/entropy/Mahadevan23} and universal decision models \cite{sm:udm}. 

\subsection*{Subobject Classifiers} 

A topos builds on the property of subobject classifiers in {\bf Sets}. Given any subset $S \subset X$, we can define $S$ as the monic arrow $S \hookrightarrow X$ defined by the inclusion of $S$ in $X$, or as the characteristic function $\phi_S$ that is equal to $1$ for all elements $x \in X$ that belong to $S$, and takes the value $0$ otherwise. We can define the set ${\bf 2} = \{0, 1 \}$ and treat {\bf true} as the inclusion $\{1 \}$ in ${\bf 2}$. The characteristic function $\phi_S$ can then be defined as the pullback of {\bf true} along $\phi_S$. 

\[\begin{tikzcd}
	S &&& {{\bf 1}} \\
	\\
	X &&& {{\bf 2}}
	\arrow["m", tail, from=1-1, to=3-1]
	\arrow[from=1-1, to=1-4]
	\arrow["{{\bf true}}"{description}, tail, from=1-4, to=3-4]
	\arrow["{\phi_S}"{description}, dashed, from=3-1, to=3-4]
\end{tikzcd}\]

We can now define subobject classifiers in a category ${\cal C}$ as follows. 

\begin{definition}
    In a category ${\cal C}$ with finite limits, a {\bf subobject classifier} is a {\em monic} arrow ${\bf true}: {\bf 1} \rightarrow \Omega$, such that to every other monic arrow $S \hookrightarrow X$ in ${\cal C}$, there is a unique arrow $\phi$ that forms the following pullback square: 

\[\begin{tikzcd}
	S &&& {{\bf 1}} \\
	\\
	X &&& \Omega
	\arrow["m", tail, from=1-1, to=3-1]
	\arrow[from=1-1, to=1-4]
	\arrow["{{\bf true}}"{description}, tail, from=1-4, to=3-4]
	\arrow["{\phi}"{description}, dashed, from=3-1, to=3-4]
\end{tikzcd}\]
    
\end{definition}

This definition can be rephrased as saying that the subobject functor is representable. In other words, a subobject of an object $x$ in a category ${\cal C}$ is an equivalence class of monic arrows $m: S \hookrightarrow  x$. 

\cite{maclane:sheaves} provide many examples of subobject classifiers.  \cite{vigna2003guided} gives a detailed description of the topos of graphs.  

\subsection*{Heyting Algebras} 

A truly remarkable finding is that the logic of topoi is not classical Boolean logic, but intuitionistic logic defined by {\em Heyting algebras}. 

\begin{definition}
    A {\bf Heyting algebra} is a poset with all finite products and coproducts, which is Cartesian closed. That is, a Heyting algebra is a lattice with ${\bf 0}$ and ${\bf 1}$ which has to each pair of elements $x$ and $y$ an exponential $y^x$. The exponential is written $x \Rightarrow y$, and defined as the adjunction 

    \[ z \leq (x \Rightarrow y) \ \ \mbox{if and only if} \ \ z \wedge x \leq y\]
\end{definition}

Alternatively, $x \Rightarrow y$ is a least upper bound for all those elements $z$ with $z \wedge x \leq y$. Therefore, for the particular case of $y$, we get that $y \leq (x \Rightarrow y)$. In the figure below, the arrows show the partial ordering relationship. As a concrete example, for a topological space $X$ the set of open setx ${\cal O}(X)$ is a Heyting algebra. The binary intersections and unions of open sets yield open sets. The empty set $\emptyset$ represents ${\bf 0}$ and the complete set $X$ represents ${\bf 1}$. Given any two open sets $U$ and $V$, the exponential object $U \Rightarrow W$ is defined as the union $\bigcup_i W_i$ of all open sets $W_i$ for which $W \cap U \subset V$. 

\[\begin{tikzcd}
	&&& {x \Rightarrow y} \\
	x && y \\
	& {x \wedge y}
	\arrow[from=3-2, to=2-3]
	\arrow[from=2-3, to=1-4]
	\arrow[from=3-2, to=2-1]
\end{tikzcd}\]

Note that in a Boolean algebra, we define implication as the relationship 

\[ (x \Rightarrow y) \equiv \neg x \vee y \]

This property, which is sometimes referred to as the ``law of the excluded middle" (because if $x = y$, then this translates to $\neg x \vee x = {\bf true}$), does not hold in a Heyting algebra. For example, on a real line $\mathbb{R}$, if we define the open sets by the open intervals $(a, b), a, b \in \mathbb{R}$, the complement of an open set need not be open. 

We can now state what is a truly remarkable result about the subobjects of a (pre)sheaf. 

\begin{theorem}\cite{maclane:sheaves}
For any functor category $\hat{C} = {\bf Sets}^{{\cal C}^{op}}$ of a small category ${\cal C}$, the partially ordered set $\mbox{Sub}_{\hat{C}}(x)$ of subobjects of $x$, for any object $x$ of $\hat{C}$ is a Heyting algebra. 
\end{theorem}

This result has deep implications for a lot of applications in AI and ML that are based modeling presheaves, including causal inference and decision making. It implies that the proper logic to employ in these settings is intuitionistic logic, not classical logic as is often used in AI \cite{pearl-book,fagin,halpern:ac}. 

Finally, we can now define the category of topoi. 

\begin{definition}
    A {\bf topos} is a category ${\cal E}$ with 
    \begin{enumerate}
        \item A pullback for every diagram $X \rightarrow B \leftarrow Y$. 

        \item A terminal object ${\bf 1}$. 

        \item An object $\Omega$ and a monic arrow ${\bf true}: 1 \rightarrow \Omega$ such that any monic $m: S \hookrightarrow B$, there is a unique arrow $\phi: B \rightarrow \Omega$ in ${\cal E}$ for which the following square is a pullback: 
        
        \[\begin{tikzcd}
	S &&& {{\bf 1}} \\
	\\
	X &&& \Omega
	\arrow["m", tail, from=1-1, to=3-1]
	\arrow[from=1-1, to=1-4]
	\arrow["{{\bf true}}"{description}, tail, from=1-4, to=3-4]
	\arrow["{\phi}"{description}, dashed, from=3-1, to=3-4]
\end{tikzcd}\]

\item To each object $x$ an object $P x$ and an arrow $\epsilon_x: x \times P x \rightarrow \Omega$ such that for every arrow $f: x \times y \rightarrow \Omega$, there is a unique arrow $g: y \rightarrow P x$ for which the following diagrams commute: 

\[\begin{tikzcd}
	y && {x \times y} &&& \Omega \\
	\\
	Px && {x \times P x} &&& \Omega
	\arrow["g", dashed, from=1-1, to=3-1]
	\arrow["f", from=1-3, to=1-6]
	\arrow["{\epsilon_x}", from=3-3, to=3-6]
	\arrow["{1 \times g}"{description}, dashed, from=1-3, to=3-3]
\end{tikzcd}\]

    \end{enumerate}
\end{definition}

\subsection{Topological Embedding of Simplicial Sets}

Simplicial sets can be embedded in a topological space using coends \cite{maclane:71}, which is the basis for a popular machine learning method for reducing the dimensionality of data called UMAP (Uniform Manifold Approximation and Projection) \cite{umap}. 

\begin{definition}
 The {\bf geometric realization} $|X|$ of a simplicial set $X$ is defined as the topological space 

 \[ |X| = \bigsqcup_{n \geq 0} X_n \times \Delta^n / ~\sim \]

 where the $n$-simplex $X_n$ is assumed to have a {\em discrete} topology (i.e., all subsets of $X_n$ are open sets), and $\Delta^n$ denotes the {\em topological} $n$-simplex 

 \[ \Delta^n = \{(p_0, \ldots, p_n) \in \mathbb{R}^{n+1} \ | \ 0 \leq p_i \leq 1, \sum_i p_i = 1 \]

 The spaces $\Delta^n, n \geq 0$ can be viewed as {\em cosimplicial} topological spaces with the following degeneracy and face maps: 

 \[ \delta_i(t_0, \ldots, t_n) = (t_0, \ldots, t_{i-1}, 0, t_i, \ldots, t_n)  \ \mbox{for} \ 0 \leq i \leq n\]

 \[ \sigma_j(t_0, \ldots, t_n) = (t_0, \ldots, t_{j} + t_{j+1}, \ldots, t_n) \ \mbox{for} \ 0 \leq i \leq n\]

 Note that $\delta_i: \mathbb{R}^n \rightarrow \mathbb{R}^{n+1}$, whereas $\sigma_j: \mathbb{R}^n \rightarrow \mathbb{R}^{n-1}$. 

 The equivalence relation $\sim$ above that defines the quotient space is  given as: 

 \[ (d_i(x), (t_0, \ldots, t_n)) \sim (x, \delta_i(t_0, \ldots, t_n) )\]

 \[ (s_j(x), (t_0, \ldots, t_n)) \sim (x, \sigma_j (t_0, \ldots, t_n)) \]
\end{definition}

\subsection*{Topological Embeddings as Coends}

We now bring in the perspective that topological embeddings can be interpreted as coends as well. Consider the functor 

\[ F: \Delta^o \times \Delta \rightarrow \mbox{Top} \] 

where 

\[ F([n], [m]) = X_n \times \Delta^m \]

where $F$ acts {\em contravariantly} as a functor from $\Delta$ to ${\bf Sets}$ mapping $[n] \mapsto X_n$,  and {\em covariantly} mapping $[m] \mapsto \Delta^m$ as a functor from $\Delta$ to the category $\mbox{Top}$ of topological spaces. 

\subsection{The Geometric Transformer Model}

In this section, we define the Geometric Transformer Model (GTM), which arises as a coend object defined by the topological embedding of a simplicial set defined over $n$-length sequences of tokens of dimension $d$. Given the restrictions on space, we can only give a very brief explanation, and a more detailed analysis is the topic of a future paper. 

Given a category of generative AI models, such as Transformers defined as permutation equivariant functions over $\mathbb{R}^{d \times n}$, it is possible to construct simplicial sets by constructing the {\em nerve} of the category. So, for example, the nerve of the category ${\cal C}_T$ of Transformers is a simplicial set $\mbox{Transformer}_\bullet$, comprised of a sequence of composable morphisms of length $n \geq 0$, each defining a Transformer block. Given this simplicial set, we can now construct a topological realization of it as a coend object 

\[ \int^n (\mbox{Transformer}_\bullet n) \cdot \Delta n \] 

where $\mbox{Transformer}_\bullet: \Delta^{op} \rightarrow {\cal C}_T$ is a contravariant functor from the simplicial category $\Delta$ into the category of Transformers, and $\Delta: |\Delta | \rightarrow {\bf Top}$ is a functor from the topological $n$-simplex realization of the simplicial category $\Delta$ into topological spaces ${\bf Top}$. As \cite{maclane:71} explains it picturesquely, the ``coend formula describes the geometric realization in one gulp". The formula says essentially to take the disjoint union of affine $n$-simplices, one for each $t \in \mbox{Transformers}_\bullet n$, and glue them together using the face and degeneracy operations defined as arrows of the simplicial category $\Delta$. In more concrete terms, this coend formula is essentially what the UMAP method implements for point cloud data in Euclidean space. Here, we are generalizing this application to construct topological realization of generative AI models, such as Transformers.

\subsection{The End of GAIA: Monads and Categorical Probability} 

We now turn to discuss the ends of generative AI models, where we first show that categorically speaking, probabilities are defined as end objects \cite{Avery_2016}. This notion requires  defining particular tyoes of functors called monads more formally, and relate them to adjoint functors. Categorically speaking, probabilities are essentially  monads \cite{Avery_2016}. Like the case with coalgebras, which we discussed extensively in previous Sections, monads also are defined by an endofunctor on a category, but one that has some special properties. These additional properties make monads possess algebraic structure, which leads to many interesting properties. Monads provide a categorical foundation for probability, based on the property that the set of all distributions on a measurable space is itself a measurable space.  The well-known {\em Giry} monad  been also shown to arise as the {\em codensity monad}  of a forgetful functor from the category of convex sets with affine maps to the category of measurable spaces \cite{Avery_2016}. Our goal in this paper is to apply monads to shed light into causal inference.   We first review the basic definitions of monads, and then discuss monad algebras, which provide ways of characterizing categories. 

Consider the pair of adjoint free and forgetful functors between graphs and categories. Here, the domain category is {\bf Cat}, the category of all categories whose objects are categories and whose morphisms are functors. The co-domain category is the category {\bf Graph} of all graphs, whose objects are directed graphs, and whose morphisms are graph homomorphisms. Here, a monad $T = U \circ F$ is induced by composing the ``free" functor $F$ that maps a graph into its associated ``free" category, and the ``forgetful" functor $U$ that maps a category into its associated graph. The monad $T$  in effect takes a directed graph $G$ and computes its transitive closure $G_{tc}$. More precisely,  for every (directed) graph $G$, there is a universal arrow from $G$ to the ``forgetful" functor $U$ mapping the category {\bf Cat} of all categories to {\bf Graph}, the category of all (directed) graphs, where for any category $C$, its associated graph is defined by $U(C)$. 

To understand this functor, simply consider a directed graph $U(C)$ as a category $C$ forgetting the rule for composition. That is, from the category $C$, which associates to each pair of composable arrows $f$ and $g$, the composed arrow $g \circ f$, we derive the underlying graph $U(G)$ simply by forgetting which edges correspond to elementary functions, such as $f$ or $g$, and which are composites. The universal arrow from a graph $G$ to the forgetful functor $U$  is defined as a pair $\langle G, u: G \rightarrow U(C) \rangle$, where $u$ is a a graph homomorphism. This arrow possesses the following {\em universal property}: for every other pair $\langle D, v: G \rightarrow H \rangle$, where $D$ is a category, and $v$ is an arbitrary graph homomorphism, there is a functor  $f': C \rightarrow D$, which is an arrow in the category {\bf Cat} of all categories, such that {\em every} graph homomorphism $\phi: G \rightarrow H$ uniquely factors through the universal graph homomorphism $u: G \rightarrow U(C)$  as the solution to the equation $\phi = U(f') \circ u$, where $U(f'): U(C) \rightarrow H$ (that is, $H = U(D)$).  Namely, the dotted arrow defines a graph homomorphism $U(f')$ that makes the triangle diagram ``commute", and the associated ``extension" problem of finding this new graph homomorphism $U(f')$ is solved by ``lifting" the associated category arrow $f': C \rightarrow D$. In causal inference using graph-based models, the transitive closure graph is quite important in a number of situations. It can be the initial target of a causal discovery algorithm that uses conditional independence oracles. It is also common in graph-based causal inference \cite{pearl-book} to model causal effects through a directed acyclic graph (DAG) $G$, which specifies its algebraic structure, and through a set of probability distributions on $G$ that specifies its semantics $P(G)$. Often, reasoning about causality in a DAG requires examining paths that lead from some vertex $x$, representing a causal variable, to some other vertex $y$. The process of constructing the transitive closure of a DAG provides a simple example of a causal monad. 

\begin{definition}
    A {\bf monad} on a category $C$ consists of

    \begin{itemize} 
    \item An endofunctor $T: C \rightarrow C$
    \item A {\bf unit} natural transformation $\eta: 1_C \Rightarrow T$ 
    \item A {\bf multiplication} natural transformation $\mu: T^2 \rightarrow T$
    \end{itemize} 
    such that the following commutative diagram in the category $C^C$ commutes (notice the arrows in this diagram are natural transformations as each object in the diagram is a functor). 
\end{definition}

\begin{center}

\[\begin{tikzcd}
	{T^3} &&&& {T^2} \\
	\\
	\\
	{T^2} &&&& T
	\arrow["{T \mu}", from=1-1, to=1-5]
	\arrow["{\mu T}"', from=1-1, to=4-1]
	\arrow["\mu"', from=4-1, to=4-5]
	\arrow["\mu", from=1-5, to=4-5]
\end{tikzcd}\]

\[\begin{tikzcd}
	T &&& {T^2} &&& T \\
	\\
	\\
	&&& T
	\arrow["{T \eta}"', from=1-7, to=1-4]
	\arrow["\mu", from=1-4, to=4-4]
	\arrow["{\eta T}", from=1-1, to=1-4]
	\arrow["{1_T}"', from=1-1, to=4-4]
	\arrow["{1_T}"', from=1-7, to=4-4]
\end{tikzcd}\]

\end{center}

It is useful to think of monads as the ``shadow" cast by an adjunction on the category corresponding to the co-domain of the right adjoint $G$. Consider the following pair of adjoint functors $F \vdash G$. 

\[
        \begin{tikzcd}
            \mathcal{C}\arrow[r, shift left=.75ex, "F"{name=F}] & \mathcal{D}\arrow[l, shift left=.75ex, "G"{name=G}] 
            \arrow[phantom, from=F, to=G, "\dashv" rotate=270].      
        \end{tikzcd}  \ \ \ \ \eta: 1_C \Rightarrow UF, \ \ \ \epsilon: FU \Rightarrow 1_D
    \]

In the language of ML, if we treat category $C$ as representing ``labeled training data" where we have full information, and category $D$ as representing a new domain for which we have no labels, what can we conclude about category $D$ from the information we have from the adjunction? The endofunctor $UF$ on $C$ is of course available to us, as is the natural transformation $\eta: 1_C \Rightarrow UF$. The map $\epsilon_A: FG A \rightarrow A$ for any object $A \in D$ is an endofunctor on $D$, about which we have no information. However, the augmented natural transformation $U \epsilon F A: G F G F A \rightarrow G F A$ can be studied in category $C$. From this data, what can we conclude about the objects in category $D$? In response to the natural question of whether every monad can be defined by a pair of adjoint functors, two solutions arose that came about from two different pairs of adjoint functors. These are referred to as the {\em Eilenberg-Moore} category and the {\em Kleisli} category \cite{maclane:71}. 

\subsection*{Codensity Monads and Probability} 

A striking recent finding is that categorical probability structures, such as Giry monads, are in essence {\em codensity monads} that result from extending a certain functor along itself \cite{Avery_2016}. 

\begin{definition}

A {\bf {codensity monad}} $T^{\cal F}$ of a functor ${\cal F}$ is the right Kan extension of ${\cal F}$ along itself (if it exists). The codensity monad inherits the university property from the Kan extension.  

\begin{center}
\begin{tikzcd}[row sep=huge, column sep=huge] 
 \mathcal{C} \arrow[dr, "\mathcal{F}"'{name=F}] 
 \arrow[rr, "\mathcal{F}", ""{name=H, below}] && \mathcal{E} \\ 
 & |[alias=D]| \mathcal{E} \arrow[ur, swap, dashed,
 "\operatorname{T}^{\mathcal{F}}"] 
 \arrow[Rightarrow, to=H, from=D, "\eta",shorten >=1em,shorten <=1em] 
\end{tikzcd} 
\end{center} 

\end{definition} 

Codensity monads can also be written using Yoneda's abstract integral calculus as ends: 

\[ T^{\cal F} e = \int_{c \in C} [{\cal E}(e, {\cal F}c), {\cal F}c]\]

Here, the notation $[A, m]$, where $A$ is any set, and $m$ is any object of a category ${\cal M}$, denotes the product in ${\cal M}$ of $A$ copies of $m$. 

\begin{definition}
    A {\bf convex set} $c$ is a convex subset of a real vector space, where for all $x, y \in c$, and for all $r \in [0,1]$, the convex combination $r x + (1 - r) y \in c$. An {\bf affine} map $h: c \rightarrow c'$ is a function such that $h(x +_r y) = h(x) +_r h(y)$ where $x +_r y = r x + (1 - r) y, r \in [0,1]$. 
\end{definition}

To define categorical probability as codensity monads, we need to  define forgetful functors from the category ${\cal C}'$ of compact convex subsets of $\mathbb{R}{^n}$ with affine maps to the category ${\bf Meas}$ of measurable spaces and measurable functions. In addition, let ${\cal D}'$ be the category ${\cal C}'$ with the object $d_0$ adjoined, where $d_0$ is the convex set of convergent sequences in the unit interval $I = [0,1]$. 

\begin{theorem}\cite{Avery_2016}
    The ${\cal C}'$ be the category of compact convex subsets of $\mathbb{R}^n$ for varying $n$ with affine maps between them, and let ${\cal D}'$ be the same with the object $d_0$ adjoined. Then, the codensity monads of the forgetful functors $U': {\cal C}' \rightarrow {\bf Meas}$ and $V': {\cal D}' \rightarrow {\bf Meas}$ are the finitely additive Giry monad and the Giry monad respectively. 
\end{theorem}

The well-known Giry monad  defines probabilities in both the discrete case and the continuous case (over Polish spaces) in terms of endofunctor on the category of measurable spaces. We can view Transformers as essentially defining a Giri monad over the space of all sequences of tokens representing strings of words in natural language. In effect, Transformers are an end, and there is much more to be described here than we have space in this paper. The complete analysis of the ends of GAIA models is the topic of a subsequent paper.

\section{Homotopy and Classifying Spaces of Generative AI Models}

\label{homotopy}

In this section, we introduce the concept of a {\em classifying space} of a generative AI model, such as a Transformer network or a stable diffusion step or a structured state space sequence model. Each of these define composable morphisms. The sequence of such composable morphisms defines a simplicial set through the nerve functor, and the classifying space corresponds to the topological realization of the simplicial set. This construction of a classifying space is an example of homotopy theory in categories \cite{richter2020categories}, which gives us ways to abstractly compare generative AI models. 

\subsection{Homotopy in Categories} 

To motivate the need to consider {\em homotopical equivalence}, we consider the following problem: a generative AI system can be used to construct summaries of documents, which raises the question of how to decide if a document summary reflects the actual document. If we view a document as an object in a category, then the question becomes one of deciding object equivalence in a looser sense of homotopy, namely is there an invertible transformation between the original document and its summary?  We discuss how to construct the topological embedding of an arbitrary category by embedding it into a simplicial set by constructing its nerve, and then finding the topological embedding of the nerve using the {\em homotopy colimit} \cite{richter2020categories}. First, we discuss the topological embedding of a simplicial set, and formulate it in terms of computing a coend.  As another example, causal generative AI models can only be determined up to some equivalence class from data, and while many causal discovery algorithms assume arbitrary interventions can be carried out to  discover the unique structure, such interventions are generally impossible to do in practical applications. The concept of {\em {essential graph}} \cite{anderson-annals} is based on defining a ``quotient space'' of graphs, but similar issues arise more generally for non-graph based models as well. Thus, it is useful to understand how to formulate the notion of equivalent classes of causal generative AI models in an arbitrary category. For example, given the conditional independence structure $A \CI B | C$, there are at least three different symmetric monoidal categorical representations that all satisfy this conditional independence \cite{fong:ms,string-diagram-surgery,fritz:jmlr}, and we need to define the quotient space over all such equivalent categories. 

In our previous work on causal homotopy \cite{sm:homotopy}, we exploited the connection between causal DAG graphical models and finite topological spaces. In particular, for a DAG model $G = (V, E)$, it is possible to define a finite space topology ${\cal T} = (V, {\cal O})$, whose open sets ${\cal O}$ are subsets of the vertices $V$ such that each vertex $x$ is associated with an open set $U_x$ defined as the intersection of all open sets that contain $x$. This structure is referred to an {\em Alexandroff} topology, which can be shown to emerge from universal representers defined by Yoneda embeddings in generalized metric spaces.  An intuitive way to construct an Alexandroff topology is to define the open set for each variable $x$ by the set of its ancestors $A_x$, or by the set of its descendants $D_x$. This approach transcribes a DAG graph into a finite topological space, upon which the mathematical tools of algebraic topology can be applied to construct homotopies among equivalent causal generative AI models. Our approach below generalizes this construction to simplicial objects, as well as general categories. 

\subsection{The Category of Fractions: Localizing Invertible Morphisms in a Generative AI Category}

One way to pose the question of homotopy is to ask whether a category can be reduced in some way such that all invertible morphisms can be ``localized" in some way. The problem of defining a category with a given subclass of invertible morphisms, called the category of fractions \citep{gabriel1967calculus}, is another concrete illustration of the close relationships between categories and graphs. \citet{borceux_1994} has a detailed discussion of the ``calculus of fractions'', namely how to define a category where a subclass of morphisms are to be treated as isomorphisms. The formal definition is as follows: 

\begin{definition}
Consider a category ${\cal C}$ and a class $\Sigma$ of arrows of ${\cal C}$. The {\bf {category of fractions}} ${\cal C}(\Sigma^{-1})$ is said to exist when a category ${\cal C}(\Sigma^{-1})$ and a functor $\phi: {\cal C} \rightarrow {\cal C}(\Sigma^{-1})$ can be found with the following properties: 

\begin{enumerate}
    \item $\forall f, \phi(f)$ is an isomorphism. 
    \item If ${\cal D}$ is a category, and $F: {\cal C} \rightarrow {\cal D}$ is a functor such that for all morphisms $f \in \Sigma$, $F(f)$ is an isomorphism, then there exists a unique functor $G: {\cal C}(\Sigma^{-1}) \rightarrow {\cal D}$ such that $G \circ \phi = F$. 
\end{enumerate}
\end{definition}

A detailed construction of the category of fractions is given in \cite{borceux_1994}, which uses the underlying directed graph skeleton associated with the category.

\subsection{Homotopy of Simplicial Generative AI Objects}

We will discuss homotopy in categories more generally now.  This notion of homotopy generalizes the notion of homotopy in topology, which defines why an object like a coffee cup is topologically homotopic to a doughnut (they have the same number of ``holes''). 

 \begin{definition}
 Let $C$ and $C'$ be a pair of objects in a category ${\cal C}$. We say $C$ is {\bf {a retract}} of $C'$ if there exists maps $i: C \rightarrow C'$ and $r: C' \rightarrow C$ such that $r \circ i = \mbox{id}_{\cal C}$. 
 \end{definition}
 
 \begin{definition}
 Let ${\cal C}$ be a category. We say a morphism $f: C \rightarrow D$ is a {\bf {retract of another morphism}} $f': C \rightarrow D$ if it is a retract of $f'$ when viewed as an object of the functor category {\bf {Hom}}$([1], {\cal C})$. A collection of morphisms $T$ of ${\cal C}$ is {\bf {closed under retracts}} if for every pair of morphisms $f, f'$ of ${\cal C}$, if $f$ is a retract of $f'$, and $f'$  is in $T$, then $f$ is also in $T$. 
 \end{definition}

 \begin{definition}
  Let X and Y be simplicial sets, and suppose we are given a pair of morphisms $f_0, f_1: X \rightarrow Y$. A {\bf {homotopy}} from $f_0$ to $f_1$ is a morphism $h: \Delta^1 \times X \rightarrow Y$ satisfying $f_0 = h |_{{0} \times X}$ and $f_1 = h_{ 1 \times X}$. 
 \end{definition}

 \subsection*{Classifying Spaces and Homotopy Colimits of Generative AI Models}

Building on the intuition proposed above, we now introduce a construction of a topological space associated with the nerve of a category. As we saw above, the nerve of a category is a full and faithful embedding of a category as a simplicial object. 

\begin{definition}
The {\bf {classifying space}} of a category ${\cal C}$ is the topological space associated with the nerve of the category $|N_\bullet {\cal C}|$
\end{definition}

To understand the classifying space $|N_\bullet {\cal C}|$ of a category ${\cal C}$, let us go over some simple examples to gain some insight. 

\begin{example}
For any set $X$, which can be defined as a discrete category ${\cal C}_X$ with no non-trivial morphisms, the classifying space $|N_\bullet {\cal C}_X|$ is just the discrete topology over $X$ (where the open sets are all possible subsets of $X$). 
\end{example}

\begin{example} 
If we take  a partially ordered set $[n]$, with its usual order-preserving morphisms, then the nerve of $[n]$ is isomorphic to the representable functor $\delta(-, [n])$, as shown by the Yoneda Lemma, and in that case, the classifying space is just the topological space $\Delta_n$ defined above. 
\end{example}

\begin{definition}
The {\bf {homotopy colimit}} of the nerve of the category of elements associated with the set-valued functor $\delta: {\cal C} \rightarrow$ {\bf {Set}} mapping the  category ${\cal C}$ into the category of {\bf Sets}, namely $N_\bullet \left(\int \delta \right)$. 
\end{definition}

 We can extend the above definition straightforwardly to these cases using an appropriate functor ${\cal T}$: {\bf {Set}} $\rightarrow$ {\bf {Top}}, or alternatively ${\cal M}$: {\bf {Set}} $\rightarrow$ {\bf {Meas}}. These augmented constructions can then be defined with respect to a more general notion called the {\em {homotopy colimit}} \cite{richter2020categories} of a generative AI model. 

\begin{definition}
The {\bf  {topological homotopy colimit}} $\mbox{hocolim}_{{\cal T} \circ \delta}$ of a category ${\cal C}$, along with its associated category of elements associated with  a set-valued functor $\delta: {\cal C} \rightarrow$ {\bf {Set}}, and a topological functor ${\cal T}$: {\bf {Set}} $\rightarrow$ {\bf {Top}} is isomorphic to topological space associated with the nerve of the category of elements, that is  $\mbox{hocolim}_{{\cal T} \circ \delta} \simeq|N_\bullet \left(\int \delta \right) |$. 
\end{definition}

\subsection{The Singular Homology of a Generative AI Model} 

Our goal is to define an abstract notion of an object in terms of its underlying classifying space as a category, and show how it can be useful in defining homotopy. We will also clarify how it relates to determining equivalences among objects, namely homotopical invariance, and also how it sheds light on UIGs. We build on the topological realization of $n$-simplices defined above.  Define the set of all morphisms $\mbox{Sing}_n(X) = {\bf Hom}_{\bf Top}(\Delta_n, |{\cal N}_\bullet({\cal C})|)$ as the set of singular $n$-simplices of $|{\cal N}_\bullet({\cal C})|$. 

\begin{definition}
For any topological space defined  by  $|{\cal N}_\bullet({\cal C})|$,  the {\bf {singular homology groups}} {$H_*(|{\cal N}_\bullet({\cal C})|; {\bf Z})$} are defined as the homology groups of a chain complex 
\[ \ldots \xrightarrow[]{\partial} {\bf Z}(\mbox{Sing}_2(|{\cal N}_\bullet({\cal C})|)) \xrightarrow[]{\partial} {\bf Z}(\mbox{Sing}_1(|{\cal N}_\bullet({\cal C})|)) \xrightarrow[]{\partial} {\bf Z}(\mbox{Sing}_0(|{\cal N}_\bullet({\cal C})|)) \] 
where {${\bf Z}(\mbox{Sing}_n(|{\cal N}_\bullet({\cal C})|))$} denotes the free Abelian group generated by the set $\mbox{Sing}_n(|{\cal N}_\bullet({\cal C})|)$ and the differential $\partial$ is defined on the generators by the formula 
\[ \partial (\sigma) = \sum_{i=0}^n (-1)^i d_i \sigma \] 
\end{definition}

Intuitively, a chain complex builds a sequence of vector spaces that can be used to construct an algebraic invariant of a generative AI model from its classifying space  by choosing the left {\bf {k}} module {${\bf Z}$} to be a vector space. Each differential $\partial$ then becomes a linear transformation whose representation is constructed by modeling its effect on the basis elements in \mbox{each {${\bf Z}(\mbox{Sing}_n(X))$.}} 

\begin{example}
Let us illustrate the singular homology groups defined by an integer-valued multiset~\cite{studeny2010probabilistic} used to model conditional independence. Imsets over a DAG of three variables $N = \{a, b, c \} $ can be viewed as a finite discrete topological space. For this topological space $X$, the singular homology groups $H_*(X; {\bf Z})$ are defined as the homology groups of a \mbox{chain complex} 
\[  {\bf Z}(\mbox{Sing}_3(X)) \xrightarrow[]{\partial}  {\bf Z}(\mbox{Sing}_2(X)) \xrightarrow[]{\partial} {\bf Z}(\mbox{Sing}_1(X)) \xrightarrow[]{\partial} {\bf Z}(\mbox{Sing}_0(X)) \] 
where {${\bf Z}(\mbox{Sing}_i(X))$} denotes the free Abelian group generated by the set $\mbox{Sing}_i(X)$ and the differential $\partial$ is defined on the generators by the formula 
\[ \partial (\sigma) = \sum_{i=0}^4 (-1)^i d_i \sigma \] 

The set $\mbox{Sing}_n(X)$ is the set of all morphisms {${\bf Hom}_{Top}(|\Delta_n|, X)$}. For an imset over the three variables $N = \{a, b, c \}$, we can define the singular $n$-simplex $\sigma$ as: 
\[ \sigma: |\Delta^4| \rightarrow X \ \ \mbox{where} \ \ |\Delta^n | = \{t_0, t_1, t_2, t_3 \in [0,1]^4 : t_0 + t_1 + t_2 + t_3 = 1 \} \] 

The $n$-simplex $\sigma$ has a collection of faces denoted as $d_0 \sigma, d_1 \sigma, d_2 \sigma$ and $ d_3 \sigma$.  If we pick the $k$-left module {${\bf Z}$} as the vector space over real numbers $\mathbb{R}$, then the above chain complex represents a sequence of vector spaces that can be used to construct an algebraic invariant of a topological space defined by the integer-valued multiset.  Each differential $\partial$ then becomes a linear transformation whose representation is constructed by modeling its effect on the basis elements in each {${\bf Z}(\mbox{Sing}_n(X))$}. An alternate approach to constructing a chain homology for an integer-valued multiset is to use M\"obius inversion to define the chain complex in terms of the nerve of a category (see our recent work on categoroids \citep{categoroids} for details). 
\end{example}

\section{Summary and Future Work} 

In this paper, we proposed a theoretical blueprint for a ``next-generation"  Generative AI Architecture (GAIA) that potentially lie beyond the scope of what is achievable with compositional learning methods such as backpropagation, the longstanding algorithmic workhorse of deep learning. Backpropagation can be conceptualized as a sequence of modules, where each module updates its parameters based on information it receives from downstream modules, and in turn, transmits information back to upstream modules to guide their updates. GAIA is based on a fundamentally different {\em hierarchical model}. Modules in GAIA are organized into a simplicial complex, much like business units in a company.  Each $n$-simplicial complex acts like a manager: it receives updates from its superiors and transmits information back to its $n+1$ subsimplicial complexes that are its subordinates. To ensure this simplicial generative AI organization behaves coherently, GAIA builds on the  mathematics of the higher-order category theory of simplicial sets and objects. Computations in GAIA, from query answering to foundation model building, are posed in terms of lifting diagrams over simplicial objects.  The problem of machine learning in GAIA is modeled as ``horn" extensions of simplicial sets: each sub-simplicial complex tries to update its parameters in such a way that a lifting diagram is solved. Traditional approaches used in generative AI using backpropagation can be used to solve ``inner" horn extension problems, but addressing ``outer horn" extensions requires a more elaborate framework.

At the top level, GAIA uses the simplicial category of ordinal numbers with objects defined as $[n], n \geq 0$ and arrows defined as weakly order-preserving mappings $f: [n] \rightarrow [m]$, where $f(i) \leq f(j), i \leq j$. This top-level structure can be viewed as a combinatorial ``factory" for constructing, manipulating, and destructing complex objects that can be built out of modular components defined over categories. The second layer of GAIA defines the building blocks of generative AI models as universal coalgebras over categories that can be defined using current generative AI approaches, including Transformers that define a category of permutation-equivariant functions on vector spaces, structured state-space models that define a category over linear dynamical systems, or image diffusion models that define a probabilistic coalgebra over ordinary differential equations. The third layer in GAIA is a category of elements over a (relational) database that defines the data over which foundation models are built. GAIA formulates the machine learning problem of building foundation models as extending functors over categories, rather than interpolating functions on sets or spaces, which yields  canonical solutions called left and right Kan extensions.  GAIA uses the metric Yoneda Lemma to construct universal representers of objects in non-symmetric generalized metric spaces. GAIA uses a categorical integral calculus of (co)ends to define two families of  generative AI systems. GAIA models based on coends correspond to topological generative AI systems, whereas GAIA systems based on ends correspond to probabilistic generative AI systems. 

Much of this paper has been devoted to a theoretical study of the GAIA framework for generative AI. Of course, the actual implementation and testing of GAIA is ultimately the only proof of its practical utility as a computing framework for generative AI. We anticipate the problem of designing GAIA systems is a multi-year (possibly multi-decade!) project, since it requires harnessing many sophisticated mathematical ideas into practical algorithms and hardware implementations. What makes this challenge feasible is the existence of algorithms that solve very restricted types of machine learning  problems that are already using similar technology such as used in GAIA. UMAP \cite{umap} is an elegant dimensionality reduction that constructs a simplicial set from high-dimensional image or textual data, and then constructs functors that map the data into a topological space. In effect, \cite{umap} construct a coend, although that is not the way the paper describes it. But  as \cite{maclane:71} makes clear, topological realizations are in fact coend objects. The algorithm in UMAP works on data in $\mathbb{R}^n$ and it transparently extends to sequence data used by Transformer models, which lie  in $\mathbb{R}^{d \times n}$. There are interesting questions in how to extend UMAP into a full GAIA model by using simplicial learning, which is a topic of a future paper.


\newpage 

\begin{thebibliography}{73}
\providecommand{\natexlab}[1]{#1}
\providecommand{\url}[1]{\texttt{#1}}
\expandafter\ifx\csname urlstyle\endcsname\relax
  \providecommand{\doi}[1]{doi: #1}\else
  \providecommand{\doi}{doi: \begingroup \urlstyle{rm}\Url}\fi

\bibitem[Aczel(1988)]{Aczel1988-ACZNS}
P.~Aczel.
\newblock \emph{Non-Well-Founded Sets}.
\newblock {CSLI{}} Lecture Notes, Palo Alto, CA, USA, 1988.

\bibitem[Andersson et~al.(1997)Andersson, Madigan, and Perlman]{anderson-annals}
S.~A. Andersson, D.~Madigan, and M.~D. Perlman.
\newblock {A characterization of Markov equivalence classes for acyclic digraphs}.
\newblock \emph{The Annals of Statistics}, 25\penalty0 (2):\penalty0 505 -- 541, 1997.
\newblock \doi{10.1214/aos/1031833662}.
\newblock URL \url{https://doi.org/10.1214/aos/1031833662}.

\bibitem[Avery(2016)]{Avery_2016}
T.~Avery.
\newblock Codensity and the giry monad.
\newblock \emph{Journal of Pure and Applied Algebra}, 220\penalty0 (3):\penalty0 1229–1251, Mar. 2016.
\newblock ISSN 0022-4049.
\newblock \doi{10.1016/j.jpaa.2015.08.017}.
\newblock URL \url{http://dx.doi.org/10.1016/j.jpaa.2015.08.017}.

\bibitem[Baez and Stay(2010)]{Baez_2010}
J.~Baez and M.~Stay.
\newblock Physics, topology, logic and computation: A rosetta stone.
\newblock In \emph{New Structures for Physics}, pages 95--172. Springer Berlin Heidelberg, 2010.
\newblock \doi{10.1007/978-3-642-12821-9_2}.
\newblock URL \url{https://doi.org/10.1007%2F978-3-642-12821-9_2}.

\bibitem[Barwise and Moss(1996)]{barwise}
J.~Barwise and L.~S. Moss.
\newblock \emph{Vicious circles - on the mathematics of non-wellfounded phenomena}, volume~60 of \emph{{CSLI} lecture notes series}.
\newblock {CSLI}, 1996.
\newblock ISBN 978-1-57586-009-1.

\bibitem[Bengio(2009)]{deeplearningreview-2009}
Y.~Bengio.
\newblock Learning deep architectures for {AI}.
\newblock \emph{Foundations and Trends in Machine Learning}, 2\penalty0 (1):\penalty0 1--127, 2009.

\bibitem[Bertsekas(2019)]{bertsekas:rlbook}
D.~Bertsekas.
\newblock \emph{Reinforcement Learning and Optimal Control}.
\newblock Athena Scientific, 2019.

\bibitem[Bertsekas(2005)]{DBLP:books/lib/Bertsekas05}
D.~P. Bertsekas.
\newblock \emph{Dynamic programming and optimal control, 3rd Edition}.
\newblock Athena Scientific, 2005.
\newblock ISBN 1886529264.
\newblock URL \url{https://www.worldcat.org/oclc/314894080}.

\bibitem[Boardman and Vogt(1973)]{weakkan}
M.~Boardman and R.~Vogt.
\newblock \emph{Homotopy invariant algebraic structures on topological spaces}.
\newblock Springer, Berlin, 1973.

\bibitem[Bommasani et~al.(2022)Bommasani, Hudson, Adeli, Altman, Arora, von Arx, Bernstein, Bohg, Bosselut, Brunskill, Brynjolfsson, Buch, Card, Castellon, Chatterji, Chen, Creel, Davis, Demszky, Donahue, Doumbouya, Durmus, Ermon, Etchemendy, Ethayarajh, Fei-Fei, Finn, Gale, Gillespie, Goel, Goodman, Grossman, Guha, Hashimoto, Henderson, Hewitt, Ho, Hong, Hsu, Huang, Icard, Jain, Jurafsky, Kalluri, Karamcheti, Keeling, Khani, Khattab, Koh, Krass, Krishna, Kuditipudi, Kumar, Ladhak, Lee, Lee, Leskovec, Levent, Li, Li, Ma, Malik, Manning, Mirchandani, Mitchell, Munyikwa, Nair, Narayan, Narayanan, Newman, Nie, Niebles, Nilforoshan, Nyarko, Ogut, Orr, Papadimitriou, Park, Piech, Portelance, Potts, Raghunathan, Reich, Ren, Rong, Roohani, Ruiz, Ryan, Ré, Sadigh, Sagawa, Santhanam, Shih, Srinivasan, Tamkin, Taori, Thomas, Tramèr, Wang, Wang, Wu, Wu, Wu, Xie, Yasunaga, You, Zaharia, Zhang, Zhang, Zhang, Zhang, Zheng, Zhou, and Liang]{fm}
R.~Bommasani, D.~A. Hudson, E.~Adeli, R.~Altman, S.~Arora, S.~von Arx, M.~S. Bernstein, J.~Bohg, A.~Bosselut, E.~Brunskill, E.~Brynjolfsson, S.~Buch, D.~Card, R.~Castellon, N.~Chatterji, A.~Chen, K.~Creel, J.~Q. Davis, D.~Demszky, C.~Donahue, M.~Doumbouya, E.~Durmus, S.~Ermon, J.~Etchemendy, K.~Ethayarajh, L.~Fei-Fei, C.~Finn, T.~Gale, L.~Gillespie, K.~Goel, N.~Goodman, S.~Grossman, N.~Guha, T.~Hashimoto, P.~Henderson, J.~Hewitt, D.~E. Ho, J.~Hong, K.~Hsu, J.~Huang, T.~Icard, S.~Jain, D.~Jurafsky, P.~Kalluri, S.~Karamcheti, G.~Keeling, F.~Khani, O.~Khattab, P.~W. Koh, M.~Krass, R.~Krishna, R.~Kuditipudi, A.~Kumar, F.~Ladhak, M.~Lee, T.~Lee, J.~Leskovec, I.~Levent, X.~L. Li, X.~Li, T.~Ma, A.~Malik, C.~D. Manning, S.~Mirchandani, E.~Mitchell, Z.~Munyikwa, S.~Nair, A.~Narayan, D.~Narayanan, B.~Newman, A.~Nie, J.~C. Niebles, H.~Nilforoshan, J.~Nyarko, G.~Ogut, L.~Orr, I.~Papadimitriou, J.~S. Park, C.~Piech, E.~Portelance, C.~Potts, A.~Raghunathan, R.~Reich, H.~Ren, F.~Rong, Y.~Roohani, C.~Ruiz, J.~Ryan, C.~Ré,
  D.~Sadigh, S.~Sagawa, K.~Santhanam, A.~Shih, K.~Srinivasan, A.~Tamkin, R.~Taori, A.~W. Thomas, F.~Tramèr, R.~E. Wang, W.~Wang, B.~Wu, J.~Wu, Y.~Wu, S.~M. Xie, M.~Yasunaga, J.~You, M.~Zaharia, M.~Zhang, T.~Zhang, X.~Zhang, Y.~Zhang, L.~Zheng, K.~Zhou, and P.~Liang.
\newblock On the opportunities and risks of foundation models, 2022.

\bibitem[Bonsangue et~al.(1998)Bonsangue, {van Breugel}, and Rutten]{BONSANGUE19981}
M.~Bonsangue, F.~{van Breugel}, and J.~Rutten.
\newblock Generalized metric spaces: Completion, topology, and powerdomains via the yoneda embedding.
\newblock \emph{Theoretical Computer Science}, 193\penalty0 (1):\penalty0 1--51, 1998.
\newblock ISSN 0304-3975.
\newblock \doi{https://doi.org/10.1016/S0304-3975(97)00042-X}.
\newblock URL \url{https://www.sciencedirect.com/science/article/pii/S030439759700042X}.

\bibitem[Borceux(1994)]{borceux_1994}
F.~Borceux.
\newblock \emph{Handbook of Categorical Algebra}, volume~1 of \emph{Encyclopedia of Mathematics and its Applications}.
\newblock Cambridge University Press, 1994.
\newblock \doi{10.1017/CBO9780511525858}.

\bibitem[Borkar(2008)]{borkar}
V.~S. Borkar.
\newblock \emph{Stochastic Approximation: A Dynamical Systems Viewpoint}.
\newblock Cambridge University Press, 2008.

\bibitem[Bradley et~al.(2022)Bradley, Terilla, and Vlassopoulos]{bradley:enriched-yoneda-llms}
T.~Bradley, J.~Terilla, and Y.~Vlassopoulos.
\newblock An enriched category theory of language: From syntax to semantics.
\newblock \emph{La Matematica}, 1:\penalty0 551--580, 2022.

\bibitem[Carlsson and Memoli(2010)]{Carlsson2010}
G.~Carlsson and F.~Memoli.
\newblock Classifying clustering schemes, 2010.
\newblock URL \url{http://arxiv.org/abs/1011.5270}.
\newblock cite arxiv:1011.5270.

\bibitem[Chaitin(2002)]{chaitin}
G.~J. Chaitin.
\newblock \emph{Exploring {RANDOMNESS}}.
\newblock Discrete mathematics and theoretical computer science. Springer, 2002.
\newblock ISBN 978-1-85233-417-8.

\bibitem[Coecke et~al.(2016)Coecke, Fritz, and Spekkens]{Coecke_2016}
B.~Coecke, T.~Fritz, and R.~W. Spekkens.
\newblock A mathematical theory of resources.
\newblock \emph{Information and Computation}, 250:\penalty0 59--86, oct 2016.
\newblock \doi{10.1016/j.ic.2016.02.008}.
\newblock URL \url{https://doi.org/10.1016%2Fj.ic.2016.02.008}.

\bibitem[Cover and Thomas(2006)]{cover}
T.~M. Cover and J.~A. Thomas.
\newblock \emph{Elements of Information Theory 2nd Edition (Wiley Series in Telecommunications and Signal Processing)}.
\newblock Wiley-Interscience, July 2006.
\newblock ISBN 0471241954.

\bibitem[Fagin et~al.(1995)Fagin, Halpern, Moses, and Vardi]{fagin}
R.~Fagin, J.~Y. Halpern, Y.~Moses, and M.~Y. Vardi.
\newblock \emph{Reasoning About Knowledge}.
\newblock {MIT} Press, 1995.
\newblock ISBN 9780262562003.
\newblock \doi{10.7551/MITPRESS/5803.001.0001}.
\newblock URL \url{https://doi.org/10.7551/mitpress/5803.001.0001}.

\bibitem[Feys et~al.(2018)Feys, Hansen, and Moss]{feys:hal-02044650}
F.~Feys, H.~H. Hansen, and L.~S. Moss.
\newblock {Long-Term Values in Markov Decision Processes, (Co)Algebraically}.
\newblock In C.~C{\^i}rstea, editor, \emph{{14th International Workshop on Coalgebraic Methods in Computer Science (CMCS)}}, volume LNCS-11202 of \emph{Coalgebraic Methods in Computer Science}, pages 78--99, Thessaloniki, Greece, Apr. 2018. {Springer International Publishing}.
\newblock \doi{10.1007/978-3-030-00389-0\_6}.
\newblock URL \url{https://inria.hal.science/hal-02044650}.

\bibitem[Fong(2012)]{fong:ms}
B.~Fong.
\newblock Causal theories: A categorical perspective on bayesian networks, 2012.

\bibitem[Fong and Spivak(2018)]{fong2018seven}
B.~Fong and D.~I. Spivak.
\newblock \emph{Seven Sketches in Compositionality: An Invitation to Applied Category Theory}.
\newblock Cambridge University Press, 2018.

\bibitem[Fong et~al.(2019)Fong, Spivak, and Tuy{\'{e}}ras]{DBLP:conf/lics/FongST19}
B.~Fong, D.~I. Spivak, and R.~Tuy{\'{e}}ras.
\newblock Backprop as functor: {A} compositional perspective on supervised learning.
\newblock In \emph{34th Annual {ACM/IEEE} Symposium on Logic in Computer Science, {LICS} 2019, Vancouver, BC, Canada, June 24-27, 2019}, pages 1--13. {IEEE}, 2019.
\newblock \doi{10.1109/LICS.2019.8785665}.
\newblock URL \url{https://doi.org/10.1109/LICS.2019.8785665}.

\bibitem[Fritz and Klingler(2023)]{fritz:jmlr}
T.~Fritz and A.~Klingler.
\newblock The d-separation criterion in categorical probability.
\newblock \emph{Journal of Machine Learning Research}, 24\penalty0 (46):\penalty0 1--49, 2023.
\newblock URL \url{http://jmlr.org/papers/v24/22-0916.html}.

\bibitem[Gabriel et~al.(1967)Gabriel, Gabriel, and Zisman]{gabriel1967calculus}
P.~Gabriel, P.~Gabriel, and M.~Zisman.
\newblock \emph{Calculus of Fractions and Homotopy Theory}.
\newblock Calculus of Fractions and Homotopy Theory. Springer-Verlag, 1967.
\newblock ISBN 9780387037776.
\newblock URL \url{https://books.google.com/books?id=UEQZAQAAIAAJ}.

\bibitem[Gavrilovich(2017)]{lifting}
M.~Gavrilovich.
\newblock The unreasonable power of the lifting property in elementary mathematics, 2017.
\newblock URL \url{https://arxiv.org/abs/1707.06615}.

\bibitem[Gu et~al.(2022)Gu, Goel, and R{\'{e}}]{DBLP:conf/iclr/GuGR22}
A.~Gu, K.~Goel, and C.~R{\'{e}}.
\newblock Efficiently modeling long sequences with structured state spaces.
\newblock In \emph{The Tenth International Conference on Learning Representations, {ICLR} 2022, Virtual Event, April 25-29, 2022}. OpenReview.net, 2022.
\newblock URL \url{https://openreview.net/forum?id=uYLFoz1vlAC}.

\bibitem[Gu et~al.(2023)Gu, Johnson, Timalsina, Rudra, and R{\'{e}}]{DBLP:conf/iclr/GuJTRR23}
A.~Gu, I.~Johnson, A.~Timalsina, A.~Rudra, and C.~R{\'{e}}.
\newblock How to train your {HIPPO:} state space models with generalized orthogonal basis projections.
\newblock In \emph{The Eleventh International Conference on Learning Representations, {ICLR} 2023, Kigali, Rwanda, May 1-5, 2023}. OpenReview.net, 2023.
\newblock URL \url{https://openreview.net/pdf?id=klK17OQ3KB}.

\bibitem[Halpern(2016)]{halpern:ac}
J.~Y. Halpern.
\newblock \emph{Actual Causality}.
\newblock {MIT} Press, 2016.
\newblock ISBN 978-0-262-03502-6.

\bibitem[Jacobs(2016)]{jacobs:book}
B.~Jacobs.
\newblock \emph{Introduction to Coalgebra: Towards Mathematics of States and Observation}, volume~59 of \emph{Cambridge Tracts in Theoretical Computer Science}.
\newblock Cambridge University Press, 2016.
\newblock ISBN 9781316823187.
\newblock \doi{10.1017/CBO9781316823187}.
\newblock URL \url{https://doi.org/10.1017/CBO9781316823187}.

\bibitem[Jacobs et~al.(2019)Jacobs, Kissinger, and Zanasi]{string-diagram-surgery}
B.~Jacobs, A.~Kissinger, and F.~Zanasi.
\newblock Causal inference by string diagram surgery, 2019.

\bibitem[Joyal(2002)]{quasicats}
A.~Joyal.
\newblock Quasi-categories and kan complexes.
\newblock \emph{Journal of Pure and Applied Algebra}, 175\penalty0 (1):\penalty0 207--222, 2002.
\newblock ISSN 0022-4049.
\newblock \doi{https://doi.org/10.1016/S0022-4049(02)00135-4}.
\newblock URL \url{https://www.sciencedirect.com/science/article/pii/S0022404902001354}.
\newblock Special Volume celebrating the 70th birthday of Professor Max Kelly.

\bibitem[Kan(1958)]{kan}
D.~Kan.
\newblock Adjoint functors.
\newblock \emph{Transactions of the American Mathematical Society}, 87\penalty0 (2):\penalty0 294--329, 1958.
\newblock URL \url{https://doi.org/10.2307/1993102}.

\bibitem[Kozen and Ruozzi(2009)]{kozen}
D.~Kozen and N.~Ruozzi.
\newblock Applications of metric coinduction.
\newblock \emph{Log. Methods Comput. Sci.}, 5\penalty0 (3), 2009.
\newblock URL \url{http://arxiv.org/abs/0908.2793}.

\bibitem[Kushner and Yin(2003)]{kushner2003stochastic}
H.~Kushner and G.~Yin.
\newblock \emph{Stochastic Approximation and Recursive Algorithms and Applications}.
\newblock Stochastic Modelling and Applied Probability. Springer New York, 2003.
\newblock ISBN 9780387008943.
\newblock URL \url{https://books.google.com/books?id=_0bIieuUJGkC}.

\bibitem[Liu et~al.(2009)Liu, Chen, and Ye]{liu2009large}
J.~Liu, J.~Chen, and J.~Ye.
\newblock Large-scale sparse logistic regression.
\newblock In \emph{Proceedings of the 15th ACM SIGKDD International Conference on Knowledge Discovery and Data Mining}, pages 547--556. ACM, 2009.

\bibitem[Loregian(2021)]{loregian_2021}
F.~Loregian.
\newblock \emph{(Co)end Calculus}.
\newblock London Mathematical Society Lecture Note Series. Cambridge University Press, 2021.
\newblock \doi{10.1017/9781108778657}.

\bibitem[Lurie(2009)]{Lurie:higher-topos-theory}
J.~Lurie.
\newblock \emph{{Higher Topos Theory}}.
\newblock Annals of mathematics studies. Princeton University Press, Princeton, NJ, 2009.
\newblock URL \url{https://cds.cern.ch/record/1315170}.

\bibitem[Lurie(2022)]{kerodon}
J.~Lurie.
\newblock Kerodon.
\newblock \url{https://kerodon.net}, 2022.

\bibitem[MacLane(1971)]{maclane:71}
S.~MacLane.
\newblock \emph{Categories for the Working Mathematician}.
\newblock Springer-Verlag, New York, 1971.
\newblock Graduate Texts in Mathematics, Vol. 5.

\bibitem[MacLane and leke Moerdijk(1994)]{maclane:sheaves}
S.~MacLane and leke Moerdijk.
\newblock \emph{Sheaves in Geometry and Logic: A First Introduction to Topos Theory}.
\newblock Springer, 1994.

\bibitem[Mahadevan(2021{\natexlab{a}})]{sm:homotopy}
S.~Mahadevan.
\newblock Causal homotopy, 2021{\natexlab{a}}.
\newblock URL \url{https://arxiv.org/abs/2112.01847}.

\bibitem[Mahadevan(2021{\natexlab{b}})]{sm:udm}
S.~Mahadevan.
\newblock Universal decision models.
\newblock \emph{CoRR}, abs/2110.15431, 2021{\natexlab{b}}.
\newblock URL \url{https://arxiv.org/abs/2110.15431}.

\bibitem[Mahadevan(2022)]{categoroids}
S.~Mahadevan.
\newblock Categoroids: Universal conditional independence, 2022.
\newblock URL \url{https://arxiv.org/abs/2208.11077}.

\bibitem[Mahadevan(2023)]{DBLP:journals/entropy/Mahadevan23}
S.~Mahadevan.
\newblock Universal causality.
\newblock \emph{Entropy}, 25\penalty0 (4):\penalty0 574, 2023.
\newblock \doi{10.3390/E25040574}.
\newblock URL \url{https://doi.org/10.3390/e25040574}.

\bibitem[May(1992)]{may1992simplicial}
J.~May.
\newblock \emph{Simplicial Objects in Algebraic Topology}.
\newblock University of Chicago Press, 1992.

\bibitem[May and Ponto(2012)]{may2012more}
J.~May and K.~Ponto.
\newblock \emph{More Concise Algebraic Topology: Localization, Completion, and Model Categories}.
\newblock Chicago Lectures in Mathematics. University of Chicago Press, 2012.
\newblock ISBN 9780226511788.
\newblock URL \url{https://books.google.com/books?id=SHhmxUPskFwC}.

\bibitem[McInnes et~al.(2018)McInnes, Healy, and Melville]{umap}
L.~McInnes, J.~Healy, and J.~Melville.
\newblock Umap: Uniform manifold approximation and projection for dimension reduction, 2018.
\newblock URL \url{https://arxiv.org/abs/1802.03426}.

\bibitem[Papillon et~al.(2023)Papillon, Sanborn, Hajij, and Miolane]{papillon2023architectures}
M.~Papillon, S.~Sanborn, M.~Hajij, and N.~Miolane.
\newblock Architectures of topological deep learning: A survey on topological neural networks, 2023.

\bibitem[Pearl(2009)]{pearl-book}
J.~Pearl.
\newblock \emph{Causality: Models, Reasoning and Inference}.
\newblock Cambridge University Press, USA, 2nd edition, 2009.
\newblock ISBN 052189560X.

\bibitem[Quillen(1967)]{Quillen:1967}
D.~G. Quillen.
\newblock \emph{Homotopical algebra}.
\newblock Springer, 1967.

\bibitem[Richter(2020)]{richter2020categories}
B.~Richter.
\newblock \emph{From Categories to Homotopy Theory}.
\newblock Cambridge Studies in Advanced Mathematics. Cambridge University Press, 2020.
\newblock ISBN 9781108479622.
\newblock URL \url{https://books.google.com/books?id=pnzUDwAAQBAJ}.

\bibitem[Riehl(2017)]{riehl2017category}
E.~Riehl.
\newblock \emph{Category Theory in Context}.
\newblock Aurora: Dover Modern Math Originals. Dover Publications, 2017.
\newblock ISBN 9780486820804.
\newblock URL \url{https://books.google.com/books?id=6B9MDgAAQBAJ}.

\bibitem[Robbins and Monro(1951)]{rm}
H.~Robbins and S.~Monro.
\newblock {A Stochastic Approximation Method}.
\newblock \emph{The Annals of Mathematical Statistics}, 22\penalty0 (3):\penalty0 400 -- 407, 1951.
\newblock \doi{10.1214/aoms/1177729586}.
\newblock URL \url{https://doi.org/10.1214/aoms/1177729586}.

\bibitem[Rutten(2000)]{rutten2000universal}
J.~Rutten.
\newblock Universal coalgebra: a theory of systems.
\newblock \emph{Theoretical Computer Science}, 249\penalty0 (1):\penalty0 3 -- 80, 2000.
\newblock ISSN 0304-3975.
\newblock \doi{http://dx.doi.org/10.1016/S0304-3975(00)00056-6}.
\newblock URL \url{http://www.sciencedirect.com/science/article/pii/S0304397500000566}.
\newblock Modern Algebra.

\bibitem[Sch\"{o}lkopf and Smola(2002)]{kernelbook}
B.~Sch\"{o}lkopf and A.~J. Smola.
\newblock \emph{Learning with Kernels: Support Vector Machines, Regularization, Optimization, and Beyond}.
\newblock MIT Press, 2002.

\bibitem[Sokolova(2011)]{SOKOLOVA20115095}
A.~Sokolova.
\newblock Probabilistic systems coalgebraically: A survey.
\newblock \emph{Theoretical Computer Science}, 412\penalty0 (38):\penalty0 5095--5110, 2011.
\newblock ISSN 0304-3975.
\newblock \doi{https://doi.org/10.1016/j.tcs.2011.05.008}.
\newblock URL \url{https://www.sciencedirect.com/science/article/pii/S0304397511003902}.
\newblock CMCS Tenth Anniversary Meeting.

\bibitem[Song and Ermon(2019)]{DBLP:conf/nips/SongE19}
Y.~Song and S.~Ermon.
\newblock Generative modeling by estimating gradients of the data distribution.
\newblock In H.~M. Wallach, H.~Larochelle, A.~Beygelzimer, F.~d'Alch{\'{e}}{-}Buc, E.~B. Fox, and R.~Garnett, editors, \emph{Advances in Neural Information Processing Systems 32: Annual Conference on Neural Information Processing Systems 2019, NeurIPS 2019, December 8-14, 2019, Vancouver, BC, Canada}, pages 11895--11907, 2019.
\newblock URL \url{https://proceedings.neurips.cc/paper/2019/hash/3001ef257407d5a371a96dcd947c7d93-Abstract.html}.

\bibitem[Spivak(2013)]{SPIVAK_2013}
D.~I. Spivak.
\newblock Database queries and constraints via lifting problems.
\newblock \emph{Mathematical Structures in Computer Science}, 24\penalty0 (6), oct 2013.
\newblock \doi{10.1017/s0960129513000479}.
\newblock URL \url{https://doi.org/10.1017%2Fs0960129513000479}.

\bibitem[Spivak and Kent(2012)]{Spivak_2012}
D.~I. Spivak and R.~E. Kent.
\newblock Ologs: A categorical framework for knowledge representation.
\newblock \emph{{PLoS} {ONE}}, 7\penalty0 (1):\penalty0 e24274, jan 2012.
\newblock \doi{10.1371/journal.pone.0024274}.

\bibitem[Studeny(2010)]{studeny2010probabilistic}
M.~Studeny.
\newblock \emph{Probabilistic Conditional Independence Structures}.
\newblock Information Science and Statistics. Springer London, 2010.
\newblock ISBN 9781849969482.
\newblock URL \url{https://books.google.com.gi/books?id=bGFRcgAACAAJ}.

\bibitem[Sutton and Barto(1998)]{DBLP:books/lib/SuttonB98}
R.~S. Sutton and A.~G. Barto.
\newblock \emph{Reinforcement learning - an introduction}.
\newblock Adaptive computation and machine learning. {MIT} Press, 1998.
\newblock ISBN 978-0-262-19398-6.
\newblock URL \url{https://www.worldcat.org/oclc/37293240}.

\bibitem[Turing(1950)]{turing}
A.~Turing.
\newblock Computing machinery and intelligence.
\newblock \emph{Mind}, 49:\penalty0 433--460, 1950.

\bibitem[Vaswani et~al.(2017)Vaswani, Shazeer, Parmar, Uszkoreit, Jones, Gomez, Kaiser, and Polosukhin]{DBLP:conf/nips/VaswaniSPUJGKP17}
A.~Vaswani, N.~Shazeer, N.~Parmar, J.~Uszkoreit, L.~Jones, A.~N. Gomez, L.~Kaiser, and I.~Polosukhin.
\newblock Attention is all you need.
\newblock In I.~Guyon, U.~von Luxburg, S.~Bengio, H.~M. Wallach, R.~Fergus, S.~V.~N. Vishwanathan, and R.~Garnett, editors, \emph{Advances in Neural Information Processing Systems 30: Annual Conference on Neural Information Processing Systems 2017, December 4-9, 2017, Long Beach, CA, {USA}}, pages 5998--6008, 2017.
\newblock URL \url{https://proceedings.neurips.cc/paper/2017/hash/3f5ee243547dee91fbd053c1c4a845aa-Abstract.html}.

\bibitem[Vigna(2003)]{vigna2003guided}
S.~Vigna.
\newblock A guided tour in the topos of graphs, 2003.

\bibitem[Villani(2003)]{ot}
C.~Villani.
\newblock \emph{Topics in Optimal Transportation}.
\newblock American Mathematical Society, 2003.

\bibitem[Vollmar(2006)]{DBLP:journals/jca/Vollmar06}
R.~Vollmar.
\newblock John von neumann and self-reproducing cellular automata.
\newblock \emph{J. Cell. Autom.}, 1\penalty0 (4):\penalty0 353--376, 2006.
\newblock URL \url{http://www.oldcitypublishing.com/journals/jca-home/jca-issue-contents/jca-volume-1-number-4-2006/jca-1-4-p-353-376/}.

\bibitem[Wagstaff et~al.(2022)Wagstaff, Fuchs, Engelcke, Osborne, and Posner]{DBLP:journals/jmlr/WagstaffFEOP22}
E.~Wagstaff, F.~B. Fuchs, M.~Engelcke, M.~A. Osborne, and I.~Posner.
\newblock Universal approximation of functions on sets.
\newblock \emph{J. Mach. Learn. Res.}, 23:\penalty0 151:1--151:56, 2022.
\newblock URL \url{http://jmlr.org/papers/v23/21-0730.html}.

\bibitem[Wolfram(2002)]{wolfram:book}
S.~Wolfram.
\newblock \emph{A new kind of science}.
\newblock Wolfram-Media, 2002.
\newblock ISBN 978-1-57955-008-0.

\bibitem[Yarotsky(2018)]{yarotsky}
D.~Yarotsky.
\newblock Universal approximations of invariant maps by neural networks.
\newblock \emph{CoRR}, abs/1804.10306, 2018.
\newblock URL \url{http://arxiv.org/abs/1804.10306}.

\bibitem[Yin et~al.(2023)Yin, Gharbi, Zhang, Shechtman, Durand, Freeman, and Park]{yin2023onestep}
T.~Yin, M.~Gharbi, R.~Zhang, E.~Shechtman, F.~Durand, W.~T. Freeman, and T.~Park.
\newblock One-step diffusion with distribution matching distillation, 2023.

\bibitem[Yoneda(1960)]{yoneda-end}
N.~Yoneda.
\newblock On ext and exact sequences.
\newblock \emph{J. Fac. Sci. Univ. Tokyo}, Sect. I 8:\penalty0 507--576, 1960.

\bibitem[Yun et~al.(2020)Yun, Bhojanapalli, Rawat, Reddi, and Kumar]{DBLP:conf/iclr/YunBRRK20}
C.~Yun, S.~Bhojanapalli, A.~S. Rawat, S.~J. Reddi, and S.~Kumar.
\newblock Are transformers universal approximators of sequence-to-sequence functions?
\newblock In \emph{8th International Conference on Learning Representations, {ICLR} 2020, Addis Ababa, Ethiopia, April 26-30, 2020}. OpenReview.net, 2020.
\newblock URL \url{https://openreview.net/forum?id=ByxRM0Ntvr}.

\end{thebibliography}
\end{document}